\documentclass{article}

\usepackage{microtype}
\usepackage{graphicx}
\usepackage{subfigure}
\usepackage{booktabs} % for professional tables

\usepackage{amssymb}
\usepackage{amsmath}
\usepackage{xspace}
\usepackage{bm}

\usepackage{amsthm}
\usepackage{afterpage}
\usepackage{indentfirst}
\usepackage{enumitem,booktabs,cfr-lm}
\usepackage[referable]{threeparttablex}

\usepackage[numbers,sort&compress]{natbib}
\usepackage{color}
\usepackage{mathtools}

\usepackage{multicol,multirow}
\usepackage{hyperref}

\usepackage[accepted]{arxiv}

\icmltitlerunning{Stability and Generalization of Stochastic Gradient Methods for Minimax Problems}

\newtheorem{theorem}{Theorem}
\newtheorem{lemma}[theorem]{Lemma}

\newtheorem{corollary}[theorem]{Corollary}
\theoremstyle{definition}
\newtheorem{definition}{Definition}

\newtheorem{assumption}{Assumption}

\theoremstyle{definition}

\newtheorem{remark}{Remark}

\newcommand{\nbb}{\mathbb{N}}

\newcommand{\bw}{\mathbf{w}}

\newcommand{\ibb}{\mathbb{I}}

\newcommand{\wcal}{\mathcal{W}}
\newcommand{\vcal}{\mathcal{V}}

\newcommand{\bu}{\mathbf{u}}

\newcommand{\zcal}{\mathcal{Z}}

\newcommand{\ecal}{\mathcal{E}}

\newcommand{\ocal}{O}
\newcommand{\ebb}{\mathbb{E}}

\newcommand{\bv}{\mathbf{v}}

\newcommand{\rbb}{\mathbb{R}}
\newcommand{\pbb}{\mathbb{P}}

\numberwithin{equation}{section}

\newcommand\numberthis{\addtocounter{equation}{1}\tag{\theequation}}
\begin{document}

\twocolumn[
\icmltitle{Stability and Generalization of Stochastic Gradient Methods for Minimax Problems}



\icmlsetsymbol{equal}{*}

\begin{icmlauthorlist}
%\icmlauthor{Yunwen Lei}{birmingham,tuk}
\icmlauthor{Yunwen Lei}{equal,birmingham}
\icmlauthor{Zhenhuan Yang}{equal,albany}
\icmlauthor{Tianbao Yang}{iowa}
\icmlauthor{Yiming Ying}{albany}
\end{icmlauthorlist}

\icmlaffiliation{birmingham}{School of Computer Science, University of Birmingham, Birmingham B15 2TT, UK}
\icmlaffiliation{iowa}{Department of Computer Science, The University of Iowa, Iowa City, IA 52242, USA}
\icmlaffiliation{albany}{Department of Mathematics and Statistics, State University of New York at Albany, USA}

%\icmlcorrespondingauthor{Cieua Vvvvv}{c.vvvvv@googol.com}
\icmlcorrespondingauthor{Yiming Ying}{yying@albany.edu}

\icmlkeywords{SGD, Stability, Generalization, Minimax Optimization}

\vskip 0.3in
]



\printAffiliationsAndNotice{\icmlEqualContribution} % otherwise use the standard text.

\begin{abstract}
Many machine learning problems can be formulated as minimax problems such as Generative Adversarial Networks (GANs), AUC maximization and robust estimation, to mention but a few. A substantial amount of studies are devoted to studying the convergence behavior of their stochastic gradient-type algorithms. In contrast, there is relatively little work on understanding their generalization, i.e., how the learning models built from training examples would behave on test examples. In this paper, we provide a comprehensive generalization analysis of stochastic gradient methods for minimax problems under both convex-concave and nonconvex-nonconcave cases through the lens of algorithmic stability. We establish a quantitative connection between stability and several generalization measures both in expectation and with high probability. For the convex-concave setting, our stability analysis shows that stochastic gradient descent ascent attains optimal generalization bounds for both smooth and nonsmooth minimax problems. We also establish  generalization bounds for both weakly-convex-weakly-concave and gradient-dominated problems. We report preliminary experimental results  to  verify  our  theory.

%Many learning problems can be formulated as minimax problems such as Generative  Adversarial  Networks (GANs), AUC maximization and robust estimation, to mention but a few. A substantial amount of studies are devoted to studying the convergence behavior of their stochastic gradient-type algorithms. In contrast, there is relatively little work on their generalization, i.e.   how the learning algorithms built from training examples would behave on test examples. In this paper, we provide a comprehensive generalization analysis of stochastic gradient methods for minimax problems under both convex-concave and nonconvex-nonconcave cases  through the lens of algorithmic stability.  We establish a quantitative connection between stability and several generalization measures both in expectation and with high probability. In particular, for the convex-concave setting,  our stability analysis shows that stochastic gradient descent ascent algorithms attain optimal generalization bounds for both smooth and nonsmooth minimax problems. Furthermore, we establish stability and generalization bounds for the nonconvex-nonconcave setting without smoothness assumption for the first time. \yiming{for this last sentence, I am not sure???}
\end{abstract}
%under a suitable early-stopping strategy
%, which is arguably a mostly used algorithm in minimax optimization.

\vspace*{-0.30cm}
\section{Introduction}
\vspace*{-0.060cm}
%We study the minimax learning problem
In machine learning we often encounter minimax optimization problems, where the decision variables are partitioned into two groups: one for minimization and one for maximization. This framework covers many important problems as specific instantiations, including adversarial learning~\citep{goodfellow2014generative}, robust optimization~\citep{chen2017robust,namkoong2017variance}, reinforcement learning~\citep{du2017stochastic,dai2018sbeed} and AUC maximization~\citep{ying2016stochastic,liu2018fast,zhao2011online,gao2013one,lei2021stochastic}. To solve these problems, researchers have proposed various efficient optimization algorithms, for which a representative algorithm is the stochastic gradient descent ascent (SGDA) due to its simplicity and widespread use in real-world applications.

There is a large amount of work on the convergence analysis of  minimax optimization algorithms  in different settings such as  convex-concave~\citep{nemirovski2009robust}, strongly-convex-strongly-concave (SC-SC)~\citep{balamurugan2016stochastic}, nonconvex-concave~\citep{rafique2018non} and nonconvex-nonconcave~\citep{liu2020firstorder,yang2020global} cases.
However, there is relatively little work on  studying the generalization, i.e., how the model trained based on the training examples would generalize to test examples. Indeed, a model with good performance on training data may not generalize well if the models are too complex. It is imperative to study the generalization error of the trained models to foresee their prediction behavior. This often entails the investigation of  the tradeoff between optimization and estimation for an implicit regularization.

To our best knowledge, there is only two recent work on the generalization analysis for minimax optimization algorithms~\citep{zhang2020generalization,farnia2020train}.
The argument stability for the specific empirical saddle point (ESP) was studied~\citep{zhang2020generalization}, which implies weak generalization and strong generalization bounds. However, the discussion there ignored optimization errors and nonconvex-nonconcave cases, which can be restrictive in practice. For SC-SC, convex-concave, nonconvex-nonconcave objective functions, the uniform stability of several gradient-based minimax learners was developed in a smooth setting~\citep{farnia2020train}, including gradient descent ascent (GDA), proximal point method (PPM) and GDmax. While they developed optimal generalization bounds for PPM, their discussions did not yield vanishing risk bounds for GDA in the general convex-concave case since their generalization bounds grow exponentially in terms of the iteration number. Furthermore, the above mentioned papers only study generalization bounds in expectation, and there is a lack of high-probability analysis.

In this paper, we leverage the lens of algorithmic stability to  study the generalization behavior of minimax learners for both convex-concave and nonconvex-nonconcave problems. Our discussion shows how the optimization and generalization should be balanced for good prediction performance. Our main results are listed in Table \ref{tab:summary-convex}. In particular, our contributions can be summarized as follows.

\noindent 1. We establish a quantitative connection between stability and generalization for minimax learners in different forms including weak/strong primal-dual
generalization, primal generalization and generalization with high probability. For the technical contributions, we introduce novel decompositions to handle the correlation between the primal model and dual model for connecting stability and generalization. %This is the first systematic study on connecting stability and generalization in the minimax learning setting.

\noindent 2.  We establish stability bounds of SGDA for convex-concave problems, from which we derive its optimal population risk bounds under an appropriate early-stopping strategy. We consider several measures of generalization and show that the optimal population risk bounds can be derived even in the nonsmooth case. To the best of our knowledge, our results are the first-ever known population risk bounds for minimax problems in the nonsmooth setting and the high-probability format.

\noindent 3. We further extend our analysis to the nonconvex-nonconcave setting and give the first generalization bounds for nonsmooth objective functions. Our analysis relaxes the range of step size for a controllable stability and implies meaningful primal population risk bounds under some regularity assumptions of objective functions, e.g., a decay of weak-convexity-weak-concavity parameter along the optimization process or a two-sided PL condition.

%In particular, generalization bounds for SGDA under nonsmooth and weakly-convex-weakly-concave objectives are given for the first time. For problems satisfying two-sided PL conditions, we show that a small optimization error implies a good generalization.
%with various stochastic gradient-based methods. In particular, generalization bounds for SGDA under nonsmooth and weakly-convex-weakly-concave objectives are given for the first time
% Furthermore, if the weak-convexity-weak-concavity of the objective improves with respect to training itrations $t$ by the rate $\ocal(t^{-\alpha})$, where $\alpha \in (0,1)$, we show SGDA with a moderately decaying step size $\eta_t = c t^{\min\{\alpha-1,-1/2\}}$ guarantees generalization, $c$ is a constant. For Alternating Gradient Descent Ascent (AGDA) \cite{yang2020global}, our stability analysis indicates the step size $\eta_t \asymp t^{-1}$ is capable to address small generalization and optimization error simultaneously.
%\yiming{What does it mean??} {\color{blue}Yunwen: it means that we can use step size $\eta_t=t^{-\alpha},\alpha\in(0,1)$ under some assumptions.}
%, by deriving their stability bounds of SGDA/AGDA for nonconvex-nonconcave problems,  that such algorithms with

The paper is organized as follows. The related work is discussed in Section \ref{sec:work} and  the minimax problem formulation is given in Section \ref{sec:problem}. The connection between stability and generalization is studied in Section \ref{sec:stab}. We develop population risk bounds in the convex-concave case in Section \ref{sec:convex} and extend our discussions to the nonconvex-nonconcave case in Section \ref{sec:nonconvex}. We report preliminary experiments in Section \ref{sec:exp} and conclude the paper in Section \ref{sec:conclusion}.

\vspace*{-0.02cm}
\subsection{Related Work\label{sec:work}}
\vspace*{-0.02cm}

We first review related work of stochastic optimization for minimax problems.
Convergence rates of order $O(1/\sqrt{T})$ were established for SGDA with $T$ iterations in the convex-concave case~\citep{nemirovski2009robust,nedic2009subgradient}, which can be further improved for SC-SC problems~\citep{balamurugan2016stochastic,hsieh2019convergence}. These discussions were extended to nonconvex-strongly-concave~\citep{rafique2018non,lin2020gradient,luo2020stochastic,yan2020optimal}, nonconvex-concave~\citep{thekumparampil2019efficient,lin2020gradient} and nonconvex-nonconcave~\citep{yang2020global,loizou2020stochastic,liu2020firstorder} minimax optimization problems. All the above mentioned work consider the convergence rate of optimization errors, while the generalization analysis was much less studied~\citep{zhang2020generalization,farnia2020train}. %
%All the above mentioned work consider the convergence rate of optimization errors. To our best knowledge, there is only two recent work on the stability and generalization analysis of minimax optimization algorithms~\citep{zhang2020generalization,farnia2020train}. %The argument stability for the saddle point of the empirical risk was studied~\citep{zhang2020generalization}, which implies weak generalization and strong generalization bounds. The discussion there ignores optimization errors and requires the objective function to be SC-SC. For SC-SC, convex-concave, nonconvex-nonconcave objective functions, the uniform stability of several gradient-based minimax learners was developed, including gradient descent ascent (GDA) and proximal point method (PPM) and GDmax~\citep{farnia2020train}. However, their discussions fail to yield non-vacuous generalization bounds for GDA in the general convex-concave case.

We now survey related work on stability and generalization. The framework of stability analysis was established in a seminal paper~\citep{bousquet2002stability}, where the celebrated concept of uniform stability was introduced. This stability was extended to study randomized algorithms~\citep{elisseeff2005stability}. It was shown that stability is closely related to the fundamental problem of learnability~\citep{shalev2010learnability,rakhlin2005stability}. \citet{hardt2016train} pioneered the generalization analysis of SGD via stability, which inspired several upcoming work to understand stochastic optimization algorithms based on different algorithmic stability measures, e.g., uniform stability~\citep{lin2016generalization,chen2018stability,mou2018generalization,richards2020graph,madden2020high}, argument stability~\citep{liu2017algorithmic,bassily2020stability,lei2020fine}, on-average stability~\citep{kuzborskij2018data,lei2021sharper}, hypothesis stability~\citep{london2017pac,charles2018stability,foster2019hypothesis}, Bayes stability~\citep{li2019generalization} and locally elastic stability~\citep{deng2020toward}.

\begin{table*}[t]
    \centering
    %\tlstyle
    \begin{tabular}{|c|c|c|c|c|}
    \hline
     Algorithm  & Reference  &  Assumption & Measure & Rate \\\hline
     \multirow{2}{*}{ESP} &
     \multirow{3}{*}{\citet{zhang2020generalization}}&  $\rho$-SC-SC, Lip &  Weak PD Risk    &   $\ocal(1/(n\rho))$\\  \cline{3-5}
      & & $\rho$-SC-SC, Lip, S  &   Strong PD Risk  & $\ocal(1/(n\rho^2))$\\ \cline{1-1}\cline{3-5}
      R-ESP  & &  C-C, Lip   &  Weak PD Risk  &    $\ocal(1/\sqrt{n})$\\\hline
       SGDA, SGDmax & \multirow{2}{*}{\citet{farnia2020train}} &  $\rho$-SC-SC, Lip, S   &  Weak PD Generalization\protect\footnotemark &   $\ocal(\log(n)/(n\rho))$\\ \cline{1-1}\cline{3-5}
       PPM  & &  C-C, Lip, S   &  Weak PD Risk  &    $\ocal(1/\sqrt{n})$\\ \hline
       \multirow{4}{*}{SGDA}  & \multirow{4}{*}{This work} & C-C, Lip (S) & Weak PD Risk & $\ocal({1}/{\sqrt{n}})$\\\cline{3-5}
          &  & C-$\rho$-SC, Lip, S & (H.P.) Primal Risk  & $\ocal({1}/{(\sqrt{n}\rho)})$\\\cline{3-5}
         & & C-C, Lip & H.P. Plain Risk & $\ocal(\log(n)/{\sqrt{n}})$\\\cline{3-5}
         & & $\rho$-SC-SC, Lip & Weak PD Risk & ${\ocal}(\sqrt{\log n}/(n\rho))$\\ \hline
    SGDA & \citet{farnia2020train}\tnotex{tnote:robots-r1} &  Lip, S   &  Weak PD Generalization  &    $\ocal\big(T^{\frac{Lc}{Lc+1}}/n\big) $\\ \hline
    \multirow{2}{*}{SGDA}  & \multirow{3}{*}{This work} %& $\rho$-SC, L, S & Primal Generalization\tnotex{tnote:robots-r3} & $\ocal\big(T^{\frac{Sc}{Sc+1}}/n\big) $\\ \cline{3-5}
     & $\rho$-WC-WC, Lip & Weak PD Generalization & $\ocal\big(T^{\frac{2c\rho}{2c\rho+3}}/n^{\frac{2c\rho+1}{2c\rho+3}}\big)$ \\ \cline{3-5}
     %& & Lip, S & H.P. Plain Generalization & $\ocal\big(T^{cL}\log(n)/\sqrt{n}\big)$ \\ \cline{3-5}
     & & D, Lip, S & Weak PD Generalization & $\ocal\big(1/\sqrt{n}+\sqrt{T}/n\big)$ \\ \cline{1-1}\cline{3-5}
      AGDA  &  & $\rho$-SC, PL, Lip, S & Primal Risk & $\ocal\big(n^{-\frac{cL+1}{2cL+1}}\big) $\\ \hline
    \end{tabular}
    \vspace*{-3mm}
    \caption{Summary of Results. Bounds are stated in expectation or with high probability (H.P.). For risk bounds, the optimal $T$ (number of iterations) is chosen to trade-off generalization and optimization.  Here, C-C means convex-concave, C-$\rho$-SC means convex-$\rho$-strongly-concave, $\rho$-SC means nonconvex-$\rho$-strongly-concave, Lip means Lipschitz continuity, S means the smoothness, D means a decay of weak-convexity-weak-concavity parameter along the optimization process as Eq. \eqref{wcwc-diminishing-ass} and PL means the two-sided condition as Assumption \ref{ass:pl-two} . AGDA means Alternating Gradient Descent Ascent and (R)-ESP means the (regularized)-empirical risk saddle point. $c$ is a parameter in the step size and $L$ is given in Assumption~\ref{ass:smooth}. \label{tab:summary-convex}}
    \vspace*{-2mm}
\end{table*}

\vspace*{-0.060cm}
\section{Problem Formulation\label{sec:problem}}
\vspace*{-0.060cm}

Let $\wcal$ and $\vcal$ be two parameter spaces in $\rbb^d$. Let $\pbb$ be a probability measure defined on a sample space $\zcal$ and $f:\wcal\times\vcal\times\zcal\mapsto\rbb$.
We consider the following minimax optimization problem:
\begin{equation}\label{eq:minimax-problem}
\min_{\bw\in \wcal}\max_{\bv\in\vcal} F(\bw,\bv) := \ebb_{z\sim\pbb} [f(\bw,\bv;z)].
\end{equation}
%and want to find a saddle point
%\begin{equation}\label{wv-star}
%(\bw^*,\bv^*)=\arg\min_{\bw\in\wcal}\arg\max_{\bv\in\vcal}F(\bw,\bv).
%\end{equation}
%Then for any $\bw\in\wcal,\bv\in\vcal$ we have \[F(\bw^*,\bv)\leq F(\bw^*,\bv^*)\leq F(\bw,\bv^*).\]
In practice, we do not know $\pbb$ but instead have access to a dataset $S=\{z_1,\ldots,z_n\}$ independently drawn from $\pbb$.  Then, we approximate $F$ by an empirical risk \[
F_S(\bw,\bv)=\frac{1}{n}\sum_{i=1}^nf(\bw,\bv;z_i).\]
We apply a (randomized) algorithm $A$ to the dataset $S$ and get a model $A(S):=(A_{\bw}(S),A_{\bv}(S))\in\wcal\times \vcal$ as an approximate solution  of the problem \eqref{eq:minimax-problem}. Since the model $A(S)$ is trained based on the training dataset $S$, its empirical behavior as measured by $F_S$ may not generalize well to a test example~\citep{bousquet2002stability}. We are interested in studying the test error (population risk) of $A(S)$. Unlike the standard statistical learning theory (SLT) setting where there is only a minimization of $\bw$, we have different measures of population risk due to the minimax structure~\citep{zhang2020generalization}. We collect the notations of these performance measures in Table \ref{tab:notation}.  Let $\ebb[\cdot]$ denote the expectation w.r.t. the randomness of both the algorithm $A$ and the dataset $S$.

%\protect\footnotemark
%\addtocounter{footnote}{-1}
%\footnotetext{Primal generalization bounds were presented in \citep{farnia2020train}. However, the stability analysis there actually only implies weak PD generalization bounds.}
%\addtocounter{footnote}{1}
%\footnotetext{If we do not assume smoothness, then SGDA requires $O(n^2)$ iterations to get the rate $O(1/\sqrt{n})$. If we impose the smoothness assumption, then SGDA requires $O(n)$ iterations to achieve this rate.}

\begin{definition}[Weak Primal-Dual Risk]
The weak Primal-Dual (PD) population risk of a (randomized) model $({\bw},{\bv})$ is defined as~\citep{zhang2020generalization}
\[
\triangle^w({\bw},{\bv})=\sup_{\bv'\in\vcal}\ebb\big[F({\bw},\bv')\big]-\inf_{\bw'\in\wcal}\ebb\big[F(\bw',{\bv})\big].
\]
The weak PD empirical risk of $({\bw},{\bv})$ is defined as
\[
\triangle^w_S({\bw},{\bv})=\sup_{\bv'\in\vcal}\ebb\big[F_S({\bw},\bv')\big]-\inf_{\bw'\in\wcal}\ebb\big[F_S(\bw',{\bv})\big].
\]
We refer to $\triangle^w({\bw},{\bv})-\triangle^w_S({\bw},{\bv})$ as the weak PD generalization error of the model $({\bw},{\bv})$.
\end{definition}
\begin{definition}[Strong Primal-Dual Risk]
The strong PD population risk of a model $({\bw},{\bv})$ is defined as%~\citep{zhang2020generalization}
\[
\triangle^s({\bw},{\bv})=\sup_{\bv'\in\vcal}F({\bw},\bv')-\inf_{\bw'\in\wcal}F(\bw',{\bv}).
\]
The strong PD empirical risk of $({\bw},{\bv})$ is defined as
\[
\triangle^s_S({\bw},{\bv})=\sup_{\bv'\in\vcal}F_S({\bw},\bv')-\inf_{\bw'\in\wcal}F_S(\bw',{\bv}).
\]
We refer to $\triangle^s({\bw},{\bv})-\triangle^s_S({\bw},{\bv})$ as the strong PD generalization error of the model $({\bw},{\bv})$.
\end{definition}
\begin{definition}[Primal Risk]
The primal population risk of a model $\bw$ is defined as
$R(\bw)=\sup_{\bv\in\vcal}F(\bw,\bv).$
The primal empirical risk of $\bw$ is defined as
$R_S(\bw)=\sup_{\bv\in\vcal}F_S(\bw,\bv).$
We refer to $R(\bw)-R_S(\bw)$ as the primal generalization error of the model $\bw$, and $R(\bw)-\inf_{\bw'}R(\bw')$ as the excess primal population risk.
\end{definition}
According to the above definitions, we know
$\triangle^w({\bw},{\bv})\leq \ebb\big[\triangle^s({\bw},{\bv})\big]$ and
$R(\bw)-R_S(\bw)$ is closely related to $\triangle^s({\bw},{\bv})-\triangle^s_S({\bw},{\bv})$.
The key difference between the weak PD risk and the strong PD risk is that the expectation is inside of the supremum/infimum for weak PD risk, while outside of the supremum/infimum for strong PD risk. In this way, one does not need to consider the coupling between primal and dual models for studying weak PD risks, and has to consider this coupling for strong PD risks.
Furthermore, we refer to $F(\bw,\bv)-F_S(\bw,\bv)$ as the plain generalization error as it is standard in SLT.
An approach to handle a population risk is to decompose it into a generalization error (estimation error) and an empirical risk (optimization error)~\citep{bousquet2008tradeoffs}.
For example, the weak PD population risk can be decomposed as
\begin{equation}\label{gen-est-opt}
\!\!\!\!\!\!\triangle^w(\bw,\bv)=\big(\triangle^w(\bw,\bv)-\triangle_S^w(\bw,\bv)\big)+\triangle_S^w(\bw,\bv).
\end{equation}
The generalization error comes from the approximation of $\pbb$ with $S$, while the empirical risk comes since the algorithm may not find the saddle point of $F_S$. Our basic idea is to use algorithmic stability to study the generalization error and use optimization theory to study the empirical risk.

%We refer to $\triangle^w(\hat{\bw},\hat{\bv})-\triangle^w_S(\hat{\bw},\hat{\bv})$ as the generalization gap of the model $(\hat{\bw},\hat{\bv})$. We apply a randomized algorithm $A$ to the dataset $S$ and get a model as an approximate solution $A(S)=(A_{\bw}(S),A_{\bv}(S))\in\wcal\times \vcal$ of the minimax problem $F$.
We now introduce necessary definitions and assumptions. Denote $\|\cdot\|_2$ as the Euclidean norm and $\langle\cdot,\cdot\rangle$ as the inner product.
A function $g:\wcal\mapsto\rbb$ is said to be $\rho$-strongly-convex ($\rho\geq0$) iff for all $\bw,\bw'\in\wcal$ there holds
$
g(\bw)\geq g(\bw')+\langle\bw-\bw',\nabla g(\bw')\rangle+\frac{\rho}{2}\|\bw-\bw'\|_2^2,
$
where $\nabla$ is the gradient operator.
We say $g$ is convex if $g$ is 0-strongly-convex. We say $g$ is $\rho$-strongly concave if $-g$ is $\rho$-strongly convex and concave if $-g$ is convex.
\begin{definition}
Let $\rho\geq0$ and $g:\wcal\times \vcal\mapsto\rbb$. We say
\begin{enumerate}[label=(\alph*)]
    \item $g$ is $\rho$-strongly-convex-strongly-concave ($\rho$-SC-SC) if
    for any $\bv\in\vcal$, the function $\bw\mapsto g(\bw,\bv)$ is $\rho$-strongly-convex and for any $\bw\in\wcal$, the function $\bv\mapsto g(\bw,\bv)$ is $\rho$-strongly-concave.
    \item $g$ is convex-concave if $g$ is $0$-SC-SC.
    \item $g$ is $\rho$-weakly-convex-weakly-concave ($\rho$-WC-WC) if $g+\frac{\rho}{2}\big(\|\bw\|_2^2-\|\bv\|_2^2)$ is convex-concave.
\end{enumerate}
\end{definition}
%The case of $\rho$-SC-SC with $\rho=0$ corresponds to the standard convex-concave case.

\addtocounter{footnote}{0}
\footnotetext{Primal generalization bounds were presented in \citet{farnia2020train}. However, the stability analysis there actually only implies bounds on weak PD risk.}
%,leftmargin=0cm,itemindent=.5cm,labelwidth=\itemindent,labelsep=0cm,align=left
%\bigskip
%\yiming{Why the population and empirical risks are reasonable for convex and concave problems? Can we give examples?}

The following two assumptions are standard~\citep{zhang2020generalization,farnia2020train}. Assumption \ref{ass:lipschitz} amounts to saying $f$ is Lipschitz continuous with respect to (w.r.t.) both $\bw$ and $\bv$. Let $\nabla_{\bw} f$ denote the gradient w.r.t. $\bw$.
\begin{assumption}\label{ass:lipschitz}
  Let $G>0$. Assume for all $\bw\in\wcal, \bv\in\vcal$ and $z\in\zcal$, there holds
 $ \big\|\nabla_{\bw}f(\bw,\bv;z)\big\|_2\leq G\;\text{ and }\; \big\|\nabla_{\bv}f(\bw,\bv;z)\big\|_2\leq G.$
\end{assumption}
\begin{assumption}\label{ass:smooth}
  Let $L>0$. For any $z$, the function $(\bw,\bv)\mapsto f(\bw,\bv;z)$ is said to be $L$-smooth, if the following inequality holds for all $\bw\in\wcal,\bv\in\vcal$ and $z\in\zcal$
  \[
  \left\|\begin{pmatrix}
           \nabla_{\bw}f(\bw,\bv;z)-\nabla_{\bw}f(\bw',\bv';z) \\
           \nabla_{\bv}f(\bw,\bv;z)-\nabla_{\bv}f(\bw',\bv';z)
         \end{pmatrix}\right\|_2\!\leq\! L\left\|\begin{pmatrix}
                                             \bw\!-\!\bw' \\
                                             \bv\!-\!\bv'
                                           \end{pmatrix}\right\|_2.
  \]
\end{assumption}

\vspace*{-0.02cm}
\subsection{Motivating Examples}
\vspace*{-0.02cm}

The minimax formulation \eqref{eq:minimax-problem} has broad applications in machine learning. Here we give some examples.

\textbf{AUC Maximization.} Area Under ROC Curve (AUC) is a popular measure for binary classification. Let $h(\bw;x)$ denote a scoring function parameterized by $\bw$ at $x$. It was shown  that AUC maximization for learning $h$ under the square loss reduces to the problem  \citep{ying2016stochastic}
\begin{equation}\label{solam}
\min_{(\bw,a,b) \in \rbb^{d+2}}\max_{\alpha \in \rbb} \ebb[f(\bw,a,b,\alpha;z)],
\end{equation}
where $p = \pbb[y=1]$ and $f(\bw,a,b,\alpha;z) = (h(\bw;x) - a)^2\ibb_{[y=1]}/p + (h(\bw;x) - b)^2\ibb_{[y=-1]}/(1-p) + 2(1+\alpha)h(\bw;x)(\ibb_{[y=-1]}/(1-p) - \ibb_{[y=1]}/p) - \alpha^2$ ($\ibb_{[\cdot]}$ is the indicator function). It is clear that $\alpha\mapsto f(\bw,a,b,\alpha;z)$ is a (strongly) concave function. Depending on $h$, the function $f$ can be convex, nonconvex, smooth or nonsmooth. %And $f(\bw,a,b,\cdot;z)$ can be convex (e.g. $h$ is a linear function) or nonconvex (e.g. $h$ comes from deep neural nets). $f$ is Lipschitz continuous. $f$ can be nonsmooth or smooth based on the assumption of $h$.

%based on random noises  and to determine whether it is a true example $z \sim \pbb_z$ or fake $G_\bv(\xi)$.
\textbf{Generative Adversarial Networks.} GAN \cite{goodfellow2014generative} refers to a popular class of generative models that consider generative modeling as a game between a generator network $G_\bv$ and a discriminator network $D_\bw$. The generator network produces synthetic data from random noise $\xi \sim \pbb_\xi$, while the
discriminator network discriminates between the true data and the synthetic data.
In particular, a popular variant of GAN named as WGAN \cite{pmlr-v70-arjovsky17a} can be written as a minimax problem
\begin{align*}
\min_{\bw}\max_{\bv} \ebb[f(\bw,\bv;z,\xi)] = \ebb_{z} [D_\bw(z)] - \ebb_{\xi} [D_\bw(G_\bv(\xi))].
\end{align*}
While this problem is generally nonconvex-nonconcave, it is weakly-convex-weakly-concave under smoothness assumptions on $D$ and $G$~\citep{liu2020firstorder}.% or structure assumptions on the neural networks~\citep{richards2021learning}. %for various applications in reinforcement learning, learning a robust model under heavy and adversarial learning~\citep{liu2020firstorder}.
%Such problem is generally nonconvex-nonconcave. $D$ is a Lipschitz continuous function. It can be shown that $f$ is smooth whenever the activation functions are smooth.
%{\color{red} Neyo: ~\citep{liu2020firstorder} proved this problem is WC-WC under smoothness assumption. So in general, this problem is NC-NC and smooth. The purpose of this example is to give a NC-NC and smooth problem. The WC-WC example is given in the next example therefore it is no need to  emphasize WC-WC again}

\textbf{Robust Estimation with minimax estimator}.
\citet{audibert2011robust} formulated robust estimation as a minimax problem as follows
\begin{align*}
\min_{\bw}\max_{\bv}\ F_S(\bw,\bv) =  \frac{1}{n}\sum_{i=1}^n \psi(\ell_1(\bw;z_i) -  \ell_2(\bv;z_i)),
\end{align*}
where $\psi:\rbb\mapsto\rbb$ is a truncated loss, and $\ell_1, \ell_2$ are Lipschitz continuous and convex loss functions. A typical truncated loss is $\psi(x) = \log(1+|x|+x^2/2)\text{sign}(x)$ to compute a mean estimator under heavy-tailed distribution of data \cite{brownlees2015empirical,xu2020learning}, where $\text{sign}(x)$ is the sign of $x$. The composition function $F_S$ can be nonconvex and nonsmooth since $\psi$ is nonconvex and $\ell_1,\ell_2$ can be nonsmooth. Following \citet{xu2020learning}, it can be shown that $F_S(\bw,\bv)$ is weakly-convex-weakly-concave.  %{\color{blue}Yunwen: I guest $\frac{1}{n}\!\sum_{i=1}^n \!\psi(\ell(\bw;z_i) \!-\! \ell(\bv;z_i))$ is WC-WC?} {\color{red} I have modified it.}
%{\color{red}Yunwen: Please be careful. What is $g$ here? $\ell(\bw;z)$ is nonnegative. Then there is no need to use indicator function in $\psi$. Also this example is general and it use many notations which may not be easy to understand.} {\color{blue} Fixed. Compare to GAN example, this example is nonsmooth.}

\vspace*{-0.060cm}
\section{Connecting Stability and Generalization\label{sec:stab}}
\vspace*{-0.060cm}
A fundamental concept in our analysis is the algorithmic stability, which measures the sensitivity of an algorithm w.r.t. the perturbation of training datasets~\citep{bousquet2002stability}.
We say $S,S'\subset\zcal$ are neighboring datasets if they differ by at most a single example. We introduce three stability measures to the minimax learning setting. The weak-stability and uniform-stability quantify the sensitivity measured by function values, while the argument-stability quantifies the sensitivity measured by arguments.
We collect these notations of stabilities in Table \ref{tab:stability} in Appendix \ref{sec:notation}.
\begin{definition}[Algorithmic Stability]
Let $A$ be a randomized algorithm, $\epsilon>0$ and $\delta\in(0,1)$. Then  we say
\begin{enumerate}[label=(\alph*)]
    \item $A$ is $\epsilon$-weakly-stable if for all neighboring $S$ and $S'$, there holds
    \begin{multline*}
    \!\sup_{z}\Big(\sup_{\bv'\in\vcal}\ebb_A\big[f(A_{\bw}(S),\bv';z)-f(A_{\bw}(S'),\bv';z)\big]\\+\sup_{\bw'\in\wcal}\ebb_A\big[f(\bw',A_{\bv}(S);z)-f(\bw',A_{\bv}(S');z)\big]\Big)\leq\epsilon.
    \end{multline*}
    \item $A$ is $\epsilon$-argument-stable in expectation if for all neighboring $S$ and $S'$, there holds
    \[
    \ebb_A\left[\left\|\begin{pmatrix}
    A_{\bw}(S)-A_{\bw}(S') \\
    A_{\bv}(S)-A_{\bv}(S')
  \end{pmatrix}\right\|_2\right]\leq\epsilon.
    \]
   $A$ is $\epsilon$-argument-stable with probability at least $1-\delta$ if with probability at least $1-\delta$ % for any  neighboring $S$ and $S'$ the following inequality holds with probability at least $1-\delta$
    \[
    \left\|\begin{pmatrix}
    A_{\bw}(S)-A_{\bw}(S') \\
    A_{\bv}(S)-A_{\bv}(S')
  \end{pmatrix}\right\|_2\leq\epsilon.
    \]
    \item $A$ is $\epsilon$-uniformly-stable with probability at least $1-\delta$ if with probability at least $1-\delta$
    \[
    \!\!\!\!\sup_{z}\big[f(A_{\bw}(S),A_{\bv}(S);z)-f(A_{\bw}(S'),A_{\bv}(S');z)\big]\leq\epsilon.
    \]
\end{enumerate}
\end{definition}

Under Assumption \ref{ass:lipschitz}, argument stability implies weak and uniform stability. As we will see, argument stability plays an important role in getting primal population risk bounds.
%Since we will require Assumption \ref{ass:lipschitz} throughout this paper, it is straightforward to see the argument stability in expectation implies the weak stability and the almost sure argument stability implies the uniform stability.  Yet we will see argument stability plays an important role to derive stronger generalization bound.

As our first main result, we establish a quantitative connection between algorithmic stability and generalization in the following theorem to be proved in Appendix \ref{sec:proof-stab-gen}. Part (a) establishes the connection between weak-stability and weak PD generalization error.
Part (b) and Part (c) establish the connection between argument stability and strong/primal generalization error under a further assumption on the strong convexity/concavity.
%It should be noted that we actually only need the argument stability w.r.t. $\bw$ in Part (b).
Part (d) and Part (e) establish high-probability bounds based on the uniform stability, which are much more challenging to derive than bounds in expectation and are important to understand the variation of an algorithm in several independent runs~\citep{feldman2019high,bousquet2020sharper}. Regarding the technical contributions, we introduce novel decompositions in handling the correlation between $A_{\bw}(S)$ and $\bv_S^*=\arg\sup_{\bv}F(A_{\bw}(S),\bv)$, especially for high-probability analysis.
%The proofs are given . %Part (e) is due to \citet{feldman2019high,bousquet2020sharper} and establishes bounds on plain generalization errors.
%The proof is standard and is given in the appendix.
% Assume for all $\bw,\bw'\in\wcal,\bv\in\vcal$ there holds
  %\begin{equation}\label{smooth-F}
 % \[\|\nabla_{\bv}F(\bw,\bv)-\nabla_{\bv}F(\bw',\bv)\|_2\leq %L\|\bw-\bw'\|_2.
 % \]%\end{equation}
\begin{theorem}\label{thm:stab-gen}
Let $A$ be a randomized algorithm and $\epsilon>0$.
\begin{enumerate}[label=(\alph*)]
  \item If $A$ is $\epsilon$-weakly-stable, then the weak PD generalization error of $(A_{\bw}(S),A_{\bv}(S))$ satisfies
  \[%\begin{equation}\label{stab-gen-a}
        \triangle^w(A_{\bw}(S),A_{\bv}(S))-\triangle^w_S(A_{\bw}(S),A_{\bv}(S))\leq\epsilon.
  \]
  \item
  If $A$ is $\epsilon$-argument-stable in expectation, the function $\bv\mapsto F(\bw,\bv)$ is $\rho$-strongly-concave and Assumptions \ref{ass:lipschitz}, \ref{ass:smooth} hold, then the primal generalization error satisfies
  \[%\begin{equation}\label{stab-gen-b}
  \ebb_{S,A}\Big[R(A_{\bw}(S))-R_S(A_{\bw}(S))\Big]\leq \big(1+L/\rho\big)G\epsilon.
  \]%\end{equation}
  \item
  If $A$ is $\epsilon$-argument-stable in expectation, $\bv\mapsto F(\bw,\bv)$ is $\rho$-SC-SC and Assumptions \ref{ass:lipschitz}, \ref{ass:smooth}  hold, then the strong PD generalization error satisfies
  \begin{multline*}
  \ebb_{S,A}\Big[\triangle^s(A_{\bw}(S),A_{\bv}(S))-\triangle^s_S(A_{\bw}(S),A_{\bv}(S))\Big]\\
  \leq \big(1+L/\rho\big)G\sqrt{2}\epsilon.
  \end{multline*}
    \item Assume $|f(\bw,\bv;z)|\leq R$ for some $R>0$ and $\bw\in\wcal,\bv\in\vcal,z\in\zcal$. Assume for all $\bw$,
  the function $\bv\mapsto F(\bw,\bv)$ is $\rho$-strongly-concave and Assumptions \ref{ass:lipschitz}, \ref{ass:smooth} hold. Let $\delta\in(0,1)$. If $A$ is $\epsilon$-uniformly stable almost surely (a.s.), then with probability at least $1-\delta$
  \begin{multline*}
  R(A_{\bw}(S))-R_S(A_{\bw}(S))=\\O\Big(GL\rho^{-1}\epsilon\log n\log(1/\delta)+Rn^{-\frac{1}2}\sqrt{\log(1/\delta)}\Big).
  \end{multline*}
    \item Assume $|f(\bw,\bv;z)|\leq R$ for some $R>0$ and $\bw\in\wcal,\bv\in\vcal,z\in\zcal$. Let $\delta\in(0,1)$. If $A$ is $\epsilon$-uniformly-stable a.s., then with probability $1-\delta$ there holds
  \begin{multline*}
  \big|F(A_{\bw}(S),A_{\bv}(S))-F_S(A_{\bw}(S),A_{\bv}(S))\big|\\
  =O\Big(\epsilon\log n\log(1/\delta)+Rn^{-\frac{1}2}\sqrt{\log(1/\delta)}\Big).
  \end{multline*}
\end{enumerate}
\end{theorem}
\begin{remark}
We compare Theorem \ref{thm:stab-gen} with related work. Weak and strong PD generalization error bounds were established for (R)-ESP~\citep{zhang2020generalization}. However, the discussion there does not consider the connection between stability and generalization. Primal generalization bounds were studied for stable algorithms~\citep{farnia2020train}. However, the discussion there is not rigorous since they used an identity $nR_S(A_{\bw}(S))=\sum_{i=1}^n\max_{\bv}f(A_{\bw}(S),\bv;z_i)$, which does not hold. To our best knowledge, Theorem \ref{thm:stab-gen} gives the first systematic connection between stability and generalization for minimax problems. %It also gives the first high-probability generalization bounds for learning with minimax problems.
%As compared weak PD generalization error bounds, the discussions on strong PD generalization errors require further assumptions on strong convexity/concavity and smoothness. Furthermore, it has an additional multiplicative factor $L/\rho$. This multiplicative factor also appears in the strong PD generalization bounds for the specific ESP of SC-SC problems~\citep{zhang2020generalization}. It remains an interesting question to investigate whether this factor can be removed. One can also develop high-probability bound on strong PD generalization error under a SC-SC condition. We omit this discussion for brevity. %, which can be very large in practice
\end{remark}

\begin{remark}
  We provide some intuitive understanding of Theorem \ref{thm:stab-gen} here. Part (a) shows that weak-stability is sufficient for weak PD generalization. This is as expected since both the supremum over $\bw'$ and $\bv'$ are outside of the expectation operator in the definition of weak stability/generalization. We do not need to consider the correlation between $A_{\bw}(S)$ and $\bv'$. As a comparison, the primal generalization needs the much stronger argument-stability. The  reason is that the supremum over $\bw'$ is inside the expectation and $\bv^{(i)}:=\arg\sup_{\bv'}F(A_{\bw}(S^{(i)}),\bv')$ is different for different $i$ ($\bv^{(i)}$ correlates to $A_{\bw}(S^{(i)})$ and $S^{(i)}$ is derived by replacing the $i$-th example in $S$ with $z_i'$). We need to estimate how $\bv^{(i)}$ differs from each other due to the difference among $A_{\bw}(S^{(i)})$. This explains why we need argument-stability and a strong-concavity in Parts (b), (d) for primal generalization. Similarly, the strong PD generalization assumes SC-SC functions.
\end{remark}

\vspace*{-0.060cm}
\section{SGDA: Convex-Concave Case\label{sec:convex}}
\vspace*{-0.060cm}

In this section, we are interested in SGDA for solving minimax optimization problems in the convex-concave case.
Let $\bw_1=0\in\wcal$ and $\bv_1=0\in\vcal$ be the initial point. Let $\text{Proj}_{\wcal}(\cdot)$ and $\text{Proj}_{\vcal}(\cdot)$ denote the projections onto $\wcal$ and $\vcal$, respectively. Let $\{\eta_t\}_t$ be a sequence of positive stepsizes. At each iteration, we randomly draw $i_t$ from the uniform distribution over $[n]:=\{1,2,\ldots,n\}$ and do the update
\begin{equation}\label{gda}
  \begin{cases}
    \bw_{t+1}=\text{Proj}_{\wcal}\big(\bw_t-\eta_t\nabla_{\bw} f(\bw_t,\bv_t;z_{i_t})\big), & \\
    \bv_{t+1}=\text{Proj}_{\vcal}\big(\bv_t+\eta_t\nabla_{\bv} f(\bw_t,\bv_t;z_{i_t})\big). &
  \end{cases}
\end{equation}

%{\color{red}Yunwen: I would suggest to move the following subsection after the problem formulation where some definitions are defined, e.g., weakly convex weakly conave}
\vspace*{-0.02cm}
\subsection{Stability Bounds\label{sec:stab-convex}}
\vspace*{-0.02cm}

In this section, we present the stability bounds for SGDA in the convex-concave case. We consider both the nonsmooth setting and smooth setting. Part (a) and Part (b) establish stability bounds in expectation, while Part (c) and Part (d) give stability bounds with high probability. Part (e) consider the SC-SC case. The proofs are given in Appendix \ref{sec:proof-stab-gda}. %We first establish stability bounds in a nonsmooth case and then extend it to a smooth case. Let $e$ be the base of natural logarithm.
%\begin{comment}
%\begin{theorem}\label{thm:stab-gda}
%Let $\{\bw_t,\bv_t\}$ be the sequence produced by \eqref{gda} with $\eta_t=\eta$. Assume for all $z$, the function $(\bw,\bv)\mapsto f(\bw,\bv;z0$ is convex-concave. Let %the algorithm $A$ be defined by $A_{\bw}(S)=\bar{\bw}_T$ and $A_{\bv}(S)=\bar{\bv}_T$, where
%\[
%\bar{\bw}_T=\frac{1}{T}\sum_{t=1}^{T}\bw_t,\quad\bar{\bv}_T=\frac{1}{T}\sum_{t=1}^{T}\bv_t.
%\]
%\begin{enumerate}[(a)]
%  \item If Assumption \ref{ass:lipschitz} holds, then $A$ is $4\eta\sqrt{2}G^2\big(\sqrt{T}+\frac{T}{n}\big)$-argument-stable in expectation.
%  \item If Assumptions \ref{ass:lipschitz}, \ref{ass:smooth} hold, then $A$ is $\frac{\sqrt{8e(t+t^2/n)}G^2\eta}{\sqrt{n}}\exp\big(Lt\eta^2/2\big)$-argument-stable in expectation.
%\end{enumerate}
%\end{theorem}
%
%\end{comment}
%\yiming{Neyo,  for theorem \ref{lem:stab-gda} can we show something for weakly convex and concave??}

%We can actually quantify a stronger stability of SGDA in the sense of models. That is, how the output models trained by SGDA on two neighboring datasets would differ as measured by the Euclidean norm.

\begin{theorem}\label{lem:stab-gda}
  Assume for all $z$, the function $(\bw,\bv)\mapsto f(\bw,\bv;z)$ is convex-concave. Let the algorithm $A$ be SGDA \eqref{gda} with $t$ iterations. Let $\delta\in(0,1)$.
  \begin{enumerate}[label=(\alph*)]
    \item Assume $\eta_t=\eta$. If Assumption \ref{ass:lipschitz} holds, then $A$ is $4\eta G\big(\sqrt{t}+t/n\big)$-argument-stable in expectation.
    \item If Assumptions \ref{ass:lipschitz}, \ref{ass:smooth} hold, then $A$ is $\epsilon$-argument-stable in expectation, where
    \[
    \!\!\!\!\epsilon \leq \frac{\sqrt{8e(1+t/n)}G}{\sqrt{n}}\exp\Big(2^{-1}L^2\sum_{j=1}^{t}\eta_{j}^2\Big)\Big(\sum_{k=1}^{t}\eta_k^2\Big)^{\frac{1}{2}}.
    \]
    \item Let $\eta_t=\eta$. If Assumption \ref{ass:lipschitz} holds, then $A$ is $\epsilon$-argument-stable with probability at least $1-\delta$, where
        \[\!\!\!\epsilon\leq \sqrt{8e}G\eta\Big(\sqrt{t}+t/n+\log(1/\delta)+\sqrt{2tn^{-1}\log(1/\delta)}\Big).\]
    \item Let $\eta_t=\eta$. If Assumptions \ref{ass:lipschitz}, \ref{ass:smooth} hold, then $A$ is $\epsilon$-argument-stable with probability at least $1-\delta$, where
   \begin{multline*}\epsilon\leq \sqrt{8e}G\eta\exp\big(2^{-1}L^2t\eta^2\big)\times\\
         \Big(1+t/n+\log(1/\delta)+\sqrt{2tn^{-1}\log(1/\delta)}\Big).
   \end{multline*}
\item If $(\bw,\bv)\mapsto f(\bw,\bv;z)$ is $\rho$-SC-SC, Assumption \ref{ass:lipschitz} holds and $\eta_t=1/(\rho t)$, then $A$ is $\epsilon$-argument-stable in expectation, where
    $\epsilon\leq \frac{2\sqrt{2}G}{\rho}\big(\frac{\log(et)}{t}+\frac{1}{n(n-2)}\big)^{\frac{1}{2}}.$
  \end{enumerate}
\end{theorem}
\begin{remark}
If $t=O(n^2)$, then the stability bounds in a nonsmooth case (Part a) become $O(\eta \sqrt{t})$ and we can still get non-vacuous bounds by taking small step size $\eta=o(t^{-1/2})$. If we choose $\eta_j=1/\sqrt{t}$ for $j\in[t]$, then the stability bound in Part (b) under a further smoothness assumption becomes $O(\sqrt{t}/n+n^{-\frac{1}{2}})$, which matches the existing result for SGD in a convex setting~\citep{hardt2016train}. The high-probability bounds in Part (c) and Part (d) enjoy the same behavior.
\end{remark}
\begin{remark}
The stability bounds of SGDA and GDA were discussed in \citet{farnia2020train} for SC-SC, Lipschitz continuous and smooth problems. We remove the smoothness assumption in Part (e) in the SC-SC case. The stability of GDA was also studied there for convex-concave $f$, which, however, implies non-vanishing generalization bounds growing exponentially with the iteration count~\citep{farnia2020train}. We extend their discussions to SGDA in this convex-concave case, and, as we will show in Theorem \ref{thm:gen}, our stability bounds imply optimal bounds on PD population risks. Furthermore, the existing discussions~\citep{farnia2020train} require the function $f$ to be smooth, while we show that meaningful stability bounds can be achieved in a nonsmooth setting (Parts (a), (c), (e)).% Finally, we develop stability bounds with high probability in Parts (c), (d), while the existing stability bounds are presented in expectation. %It should be mentioned that \citet{farnia2020train} also considered stability bounds of other algorithms such as proximal point method (PPM) and a non-simultaneous update algorithm GDmax.
\end{remark}
%\begin{comment}
%\begin{lemma}\label{lem:stab-gda-hp}
%  Assume for all $z$, the function $(\bw,\bv)\mapsto f(\bw,\bv;z)$ is convex-concave.
%  Let $S=\{z_1,\ldots,z_n\}$ and $S'=\{z_1,\ldots,z_{n-1},z_n'\}$. Let $\{\bw_t,\bv_t\}$ and $\{\bw_t',\bv_t'\}$ be the sequence produced by \eqref{gda} w.r.t. $S$ and $S'$, respectively. Let $\delta\in(0,1)$.
%  \begin{enumerate}[(a)]
% \end{enumerate}
%\end{lemma}
%\end{comment}

\begin{remark}
We consider stability bounds under various assumptions on loss functions. We now sketch the technical difference in our analysis.
Let $\delta_t:=\|\bw_t-\bw_t'\|_2^2+\|\bv_t-\bv_t'\|_2^2$, where $(\bw_t,\bv_t),(\bw_t',\bv_t')$ are SGDA iterates for $S$ and $S'$ differing only by the last element.  For convex-concave and nonsmooth problems, we show $\delta_{t+1}=\delta_t+O(\eta_t^2)$ if $i_t\neq n$.  For convex-concave and smooth problems, we show $\delta_{t+1}=(1+O(\eta_t^2))\delta_t$ if $i_t\neq n$. For $\rho$-SC-SC and nonsmooth problems, we show $\delta_{t+1}=(1-2\rho\eta_t)\delta_t+O(\eta_t^2)$ if $i_t\neq n$. For the above cases, we first control $\delta_{t+1}$ and then take expectation w.r.t. $i_t$. A key point to tackle \emph{nonsmooth} problems is to consider the evolution $\delta_t$ instead of  $\|\bw_t-\bw_t'\|_2+\|\bv_t-\bv_t'\|_2$, which is able to yield nontrivial bounds by making $\sum_{t}\eta_t^2=o(1)$ with sufficiently small $\eta_t$.
\end{remark}

%\subsection{Generalization bounds}
\vspace*{-0.02cm}
\subsection{Population Risks}
\vspace*{-0.02cm}

We now use stability bounds in Theorem \ref{lem:stab-gda} to develop error bounds of SGDA which outputs an average of iterates
\begin{equation}\label{bar-w-v}
\bar{\bw}_T=\frac{\sum_{t=1}^{T}\eta_t\bw_t}{\sum_{t=1}^T\eta_t}\ \text{ and }\ \bar{\bv}_T=\frac{\sum_{t=1}^{T}\eta_t\bv_t}{\sum_{t=1}^T\eta_t}.
\end{equation}
The underlying reason to introduce the average operator is to simplify the optimization error analysis~\citep{nemirovski2009robust}.
Indeed, our stability and generalization analysis applies to any individual iterates. As a comparison, the optimization error analysis
for the last iterate is much more difficult than that for the averaged iterate.
We use the notation $B\asymp \widetilde{B}$ if there exist constants $c_1,c_2>0$ such that $c_1\widetilde{B}<B\leq c_2\widetilde{B}$. The following theorem to be proved in Appendix \ref{sec:proof-gen-sgda} gives weak PD population risk bounds.
%\sup_{\bv\in\vcal}\ebb[F(\bar{\bw}_T,\bv)]-\inf_{\bw\in\wcal}\ebb[F(\bw,\bar{\bv}_T)]

\begin{theorem}%[Weak generalization bounds]
[Weak PD risk]\label{thm:gen}
  Let $\{\bw_t,\bv_t\}$ be produced by \eqref{gda}. Assume for all $z$, the function $(\bw,\bv)\mapsto f(\bw,\bv;z)$ is convex-concave. Let $A$ be defined by $A_{\bw}(S)=\bar{\bw}_T$ and $A_{\bv}(S)=\bar{\bv}_T$ for $(\bar{\bw}_T,\bar{\bv}_T)$ in \eqref{bar-w-v}.
  Assume $\sup_{\bw\in\wcal}\|\bw\|_2\leq B_W$ and $\sup_{\bv\in\vcal}\|\bv\|_2\leq B_V$.
  \begin{enumerate}[label=(\alph*)]
    \item If  $\eta_t=\eta$ and Assumption \ref{ass:lipschitz} holds, then
    \begin{multline}\label{gen-a}
    \triangle^w(\bar{\bw}_T,\bar{\bv}_T)\leq 4\sqrt{2}\eta G^2\Big(\sqrt{T}+\frac{T}{n}\Big)+\eta G^2\\ + \frac{B_W^2+B_V^2}{2\eta T}+\frac{G(B_W+B_V)}{\sqrt{T}}.
    \end{multline}
    If we choose $T\asymp n^2$ and $\eta\asymp T^{-3/4}$, then we get $\triangle^w(\bar{\bw}_T,\bar{\bv}_T)=O(n^{-1/2})$.
    %\begin{equation}\label{gen-b}
    %    \ebb\big[\triangle^w(\bar{\bw}_T,\bar{\bv}_T)\big]=O(n^{-1/2}).
    %\end{equation}
    \item If $\eta_t=\eta$ and Assumptions \ref{ass:lipschitz}, \ref{ass:smooth} hold, then
    \begin{multline}\label{gen-c}
    \!\!\!\!\!\!\!\!\!\!\triangle^w(\bar{\bw}_T,\bar{\bv}_T)\leq \frac{4\sqrt{e(T+T^2/n)}G^2\eta\exp\big(LT\eta^2/2\big)}{\sqrt{n}}\\
    +\eta G^2 + \frac{B_W^2+B_V^2}{2\eta T}+\frac{G(B_W+B_V)}{\sqrt{T}}.
    \end{multline}
    We can choose $T\asymp n$ and $\eta\asymp T^{-1/2}$ to derive $\triangle^w(\bar{\bw}_T,\bar{\bv}_T)=O(n^{-1/2})$.%\eqref{gen-b}.
    \item If $(\bw,\bv)\mapsto f(\bw,\bv;z)$ is $\rho$-SC-SC ($\rho>0$), Assumption \ref{ass:lipschitz} holds, $\eta_t=1/(\rho t)$ and $T\asymp n^2$, then
    \[
    \triangle^w(\bar{\bw}_T,\bar{\bv}_T)=O(\sqrt{\log n}/(n\rho)).
    \]
    \item If $(\bw,\bv)\mapsto f(\bw,\bv;z)$ is $\rho$-SC-SC ($\rho>0$), Assumptions \ref{ass:lipschitz}, \ref{ass:smooth} hold, $\eta_t=1/(\rho (t+t_0))$ with $t_0\geq L^2/\rho^2$ and $T\asymp n$, then $\triangle^w(\bar{\bw}_T,\bar{\bv}_T)=O(\log(n)/(n\rho))$.
  \end{enumerate}
\end{theorem}
\begin{remark}
We first compare our bounds with the related work in a convex-concave setting. Weak PD population risk bounds were established for PPM under Assumptions \ref{ass:lipschitz}, \ref{ass:smooth}~\citep{farnia2020train}, which updates $(\bw_{t+1}^{\text{PPM}},\bv_{t+1}^{\text{PPM}})$ as the saddle point of the following minimax problem
\[
\min_{\bw\in\wcal}\max_{\bv\in\vcal}F_S(\bw,\bv)+\frac{1}{2\eta_t}\|\bw-\bw_{t}^{\text{PPM}}\|_2^2+\frac{1}{2\eta_t}\|\bv-\bv_{t}^{\text{PPM}}\|_2^2.
\]
In particular, they developed population risk bounds  $O(1/\sqrt{n})$ by taking $T\asymp \sqrt{n}$ for PPM. However, the implementation of PPM requires to find the exact saddle point at each iteration, which is often computationally expensive. As a comparison, Part (b) shows the minimax optimal population risk bounds $O(1/\sqrt{n})$ for SGDA with $O(n)$ iterations.  Weak PD population risk bounds $\ocal(1/\sqrt{n})$ were established for R-ESP~\citep{zhang2020generalization} without a smoothness assumption, which, however, ignore the interplay between generalization and optimization. In this setting, we show SGDA achieves the same population risk bounds $O(1/\sqrt{n})$ by taking $\eta\asymp T^{-3/4}$ and $T\asymp n^2$ in Part (a). We now consider the SC-SC setting. Weak PD risk bounds $O(1/(n\rho))$ were established for ESP ~\citep{zhang2020generalization}. Since \citet{farnia2020train} did not present an explicit risk bound, we use their stability analysis to give an explicit risk bound $O(\log(n)/(n\rho))$ in the smooth case (Part (d)). As a comparison, we establish the same population risk bounds for SGDA within a logarithmic factor by taking $\eta_t=1/(\rho t)$ and $T\asymp n^2$ without the smoothness assumption (Part (c)). %Similar generalization bounds $O(\sqrt{\log(n)}/(n\rho))$ for SGDA is given in Part (c) by a logarithmic factor.
\end{remark}
We further develop bounds on primal population risks under a strong concavity assumption on $\bv\mapsto F(\bw,\bv)$. Primal risk bounds measure the performance of  primal variables, which are of real interest in some learning problems, e.g., AUC maximization and robust optimization. We consider both bounds in expectation and bounds with high probability. Let $(\bw^*,\bv^*)$ be a saddle point of $F$, i.e., for any $\bw \in \wcal$ and $\bv \in \vcal$, there holds $F(\bw^*,\bv) \leq F(\bw,\bv) \leq F(\bw,\bv^*)$.  %The following theorem can be proved in a similar way excepting using Part (b) of Theorem \ref{thm:stab-gen}. We omit the proof for brevity.   \yiming{Proof here??? } % It should be mentioned that the smoothness assumption (Assumption \ref{ass:smooth}) here can be replaced by the assumption
\begin{theorem}%[Primal generalization bounds]
[{Excess primal risk}]\label{thm:gen-strong}
  Let $\{\bw_t,\bv_t\}$ be produced by \eqref{gda} with $\eta_t=\eta$. Assume for all $z$, the function $(\bw,\bv)\mapsto f(\bw,\bv;z)$ is convex-concave and the function $\bv\mapsto F(\bw,\bv)$ is $\rho$-strongly-concave.
  Assume $\sup_{\bw\in\wcal}\|\bw\|_2\leq B_W$ and $\sup_{\bv\in\vcal}\|\bv\|_2\leq B_V$. Let the algorithm $A$ be defined by $A_{\bw}(S)=\bar{\bw}_T$ and $A_{\bv}(S)=\bar{\bv}_T$ for $(\bar{\bw}_T,\bar{\bv}_T)$ in \eqref{bar-w-v}. If Assumptions \ref{ass:lipschitz}, \ref{ass:smooth} hold, then
    \begin{multline*}
    \!\!\!\!\ebb\big[R(\bar{\bw}_T)\big]-\inf_{\bw}R(\bw)\leq\eta G^2+\frac{B_W^2\!+\!B_V^2}{2\eta T}+\frac{G(B_W\!+\!B_V)}{\sqrt{T}}\\ +\frac{(1+L/\rho)\sqrt{32e(T+T^2/n)}G^2\eta\exp\big(L^2T\eta^2/2\big)}{\sqrt{n}}.
    \end{multline*}
    In particular, if we choose $T\asymp n$ and $\eta\asymp T^{-1/2}$ then
    \begin{equation}\label{gen-strong-b}
        \ebb\big[R(\bar{\bw}_T)\big]-\inf_{\bw\in\wcal}R(\bw)=O((L/\rho)n^{-1/2}).
    \end{equation}
    Furthermore, for any $\delta\in(0,1)$ we can choose $T\asymp n$ and $\eta\asymp T^{-1/2}$ to show with probability at least $1-\delta$
    \begin{equation}\label{primal-hp}
        \!\!\!\!R(\bar{\bw}_T)\!-\!R(\bw^*)=O\Big((L/\rho)n^{-\frac{1}{2}}\log n\log^2(1/\delta)\Big).
    \end{equation}
\end{theorem}

%It also gives the first high-probability bounds for minimax problems.  gives the first primal risk bounds for convex-strongly-concave problems.
Theorem \ref{thm:gen-strong} is proved in Appendix \ref{sec:proof-gen-sgda}.
In Theorem \ref{thm:gen-hp}, we also develop high-probability bounds of order ${O}(n^{-\frac{1}{2}}\log n)$ on plain generalization errors $|F(\bar{\bw}_T,\bar{\bv}_T)-F(\bw^*,\bv^*)|$. %We defer it to due to the space limit. %It can be derived by combining Part (e) of Theorem \ref{thm:stab-gen} and Parts (c), (d) of Theorem \ref{lem:stab-gda} together. The proof is given in Appendix \ref{sec:proof-gen-sgda}. %We fail to develop high-probability bounds for the strong risk since the established connection between strong risk and argument stability requires strong concavity in a high-probability analysis (cf. Theorem \ref{thm:stab-gen}). , which shows that $(\bar{\bw}_T,\bar{\bv}_T)$ returned by SGDA is an approximate saddle point of $F$
% Note that these high-probability bounds match those in expectation within a logarithmic factor.

%\begin{remark}
%We can also develop high-probability bounds for the primal population risks %under a further strong concavity assumption on $\bv\mapsto F_S(\bw,\bv)$. We %omit the discussions for brevity.
%\end{remark}

%Stability and Generalization for
\vspace*{-0.060cm}
\section{Nonconvex-Nonconcave Objectives\label{sec:nonconvex}}
\vspace*{-0.060cm}

In this section, we extend our analysis to nonconvex-nonconcave  minimax learning problems.

\vspace*{-0.02cm}
\subsection{Stability and Generalization of SGDA}
\vspace*{-0.02cm}

%%%%%%%%%%%%%%%%%%%%%%%%%%%%%%%%%%%%%%%%%%%%%%%%%%%%%%%% SGDA %%%%%%%%%%%%%%%%%%%%%%%%%%%%%%%%%%%%%%%%%%%%%
%In this section, we consider SGDA under nonconvex-nonconcave objectives and compare our bound with previous cases. We first provide the weak PD generalization error of SGDA.
We first study the generalization bounds of SGDA for WC-WC problems. The proof is given in Appendix \ref{sec:sgda-wcwc}. In Appendix \ref{sec:sgda-hp-nonconvex}, we further give high-probability bounds.

\begin{theorem}[Weak generalization bound]\label{thm:stability-wcwc}
Let $\{\bw_t,\bv_t\}$ be produced by \eqref{gda} with $T$ iterations.
%\begin{enumerate}[label=(\alph*)]
Assume for all $z$, the function $(\bw,\bv)\mapsto f(\bw,\bv;z)$ is $\rho$-WC-WC and $|f(\cdot, \cdot; z)| \leq 1$. If Assumption \ref{ass:lipschitz} holds and  $\eta_t=c/t$, then the weak PD generalization error of SGDA is bounded by
  \begin{align*}
  %& \ebb_{S,A}\Big[\triangle^w (\bar{\bw}_T,\bar{\bv}_T) - \triangle^w_S (\bar{\bw}_T,\bar{\bv}_T)\Big]\\
  %&\ebb[\triangle^w(\bw_T,\bv_T) - \triangle^w_S(\bw_T,\bv_T)] \\
  %&\leq 8\bigg(\frac{4\sqrt{e}cG^2}{\sqrt{2c\rho + 1}}\Big(1+\frac{\sqrt{T}}{n}\Big)T^{c\rho}\bigg)^{\frac{2}{2c\rho+3}}\bigg(\frac{1}{n}\bigg)^{\frac{2c\rho+1}{2c\rho+3}}.
  \ocal\Big(\Big(1+\frac{\sqrt{T}}{n}\Big)T^{c\rho}\Big)^{\frac{2}{2c\rho+3}}\Big(\frac{1}{n}\Big)^{\frac{2c\rho+1}{2c\rho+3}}.
  \end{align*}
%  \item If Assumptions \ref{ass:lipschitz}, \ref{ass:smooth} hold and $\max\{\eta_{\bw,t}\,\eta_{\bv,t}\} = \frac{c}{t}$, then the $\epsilon$-weak PD generalization-error of SGDA satisfies
%  \begin{align*}
  %\ebb_{S,A}\Big[\triangle^w (\bar{\bw}_T,\bar{\bv}_T) \!-\! \triangle^w_S (\bar{\bw}_T,\bar{\bv}_T)\Big]
%  \epsilon \leq \frac{4(cL+1)G^2T^{cL}}{Ln}
%  \end{align*}
%\end{enumerate}
\end{theorem}

\begin{remark}
%Generalization bounds $O(n^{-1}T^{\frac{Lc}{Lc+1}})$ were established for SGDA under Assumptions \ref{ass:lipschitz}, \ref{ass:smooth}~\citep{farnia2020train}. We extend their result to a nonsmooth objective function. If we assume $T = \ocal(n^2)$, our weak PD generalization error bound is $\ocal\big(n^{-\frac{2c\rho+1}{2c\rho+3}}T^{\frac{2c\rho}{2c\rho+3}}\big)$.  Theorem \ref{thm:stability-wcwc} gives the first generalization bound of SGDA for nonsmooth and nonconvex-nonconcave objectives. In Appendix \ref{sec:sgda-hp-nonconvex}, we further give high-probability generalization bounds.

If $T = \ocal(n^2)$, our weak PD generalization error bound is of the order $\ocal\big(n^{-\frac{2c\rho+1}{2c\rho+3}}T^{\frac{2c\rho}{2c\rho+3}}\big)$. This is the first generalization bound of SGDA for nonsmooth and nonconvex-nonconcave objectives. \citet{farnia2020train} also studied generalization under nonconvex-nonconcave setting but required the objectives to be smooth, which is relaxed to a milder WC-WC assumption here. Our analysis readily applies to stochastic gradient descent (SGD) with nonsmooth weakly-convex functions, which has not been studied in the literature. %Note that it is direct to extend Theorem \ref{thm:stability-wcwc} to SGD for learning with non-smooth weakly-convex functions since SGD can be considered as a specific instantiation of SGDA with no dual variables. Therefore, our stability analysis even gives novel results in the standard nonconvex learning setting.

\end{remark}

We further consider a variant of weak-convexity-weak-concavity.  The proof is given in Appendix \ref{sec:wcwc-diminishing}. %{\color{green} Neyo: I suggest we give it a name for \eqref{wcwc-diminishing-ass} as well as in the table.}
\begin{theorem}[Weak generalization bound]\label{thm:wcwc-diminishing}
  Let $\{\bw_t,\bv_t\}$ be produced by \eqref{gda} with $T$ iterations. Let Assumptions \ref{ass:lipschitz}, \ref{ass:smooth} hold. Assume there are non-negative numbers $\{\rho_t\}_{t\in\nbb}$ such that the following inequality holds a.s.
  \begin{multline}\label{wcwc-diminishing-ass}
    \Big\langle\begin{pmatrix}
           \bw_t-\bw\\
           \bv_t-\bv
         \end{pmatrix},\begin{pmatrix}
           \nabla_{\bw} F_S(\bw_t,\bv_t)-\nabla_{\bw} F_S(\bw,\bv) \\
                  \nabla_{\bv}F_{S}(\bw,\bv)-\nabla_{\bv}F_S(\bw_t,\bv_t)
         \end{pmatrix}\Big\rangle\\ \geq-\rho_t\left\|\begin{pmatrix}
           \bw_t-\bw \\
           \bv_t-\bv
         \end{pmatrix}\right\|_2^2,\;\forall\bw\in\wcal,\bv\in\vcal.
  \end{multline}
  Then the weak PD generalization error of SGDA with $T$ iterations can be bounded by
  \[
  O\Big(n^{-1}\sum_{t=1}^{T}\Big(\eta_t^2+\frac{1}{n}\Big)\exp\Big(\sum_{k=t+1}^{T}\big(2\rho_k\eta_k+(L^2+1)\eta_k^2\big)\Big)\Big)^{\frac{1}{2}}.
  \]
\end{theorem}
Eq. \eqref{wcwc-diminishing-ass} allows the empirical objective $F_S$ to have varying weak-convexity-weak-concavity at different iterates encountered by the algorithm. This is motivated by the observation that the nonconvex-nonconcave function can have approximate convexity-concavity around a saddle point. For these problems, we can expect the weak-convexity-weak-concavity parameter $\rho_t$ to decrease along the optimization process~\citep{sagun2017empirical,yuan2019stagewise}. % as empirically observed

\begin{remark}
If $F_S$ is convex-concave, then $\rho_t=0$ and we can take $\eta_t\asymp 1/\sqrt{T}$ to show that SGDA with $T$ iterations enjoys the generalization bound $O\big(1/\sqrt{n}+\sqrt{T}/n\big)$. This extends Theorem \ref{thm:gen} since we only require the convexity-concavity of $F_S$ here instead of $f(\cdot,\cdot;z)$ for all $z$ in Theorem \ref{thm:gen}. If $\rho_t=O(t^{-\alpha})$ ($\alpha\in(0,1)$), then we can take $\eta_t\asymp t^{\min\{\alpha-1,-\frac{1}{2}\}}/\log T$ (note $\sum_{t=1}^{T}\eta_t^2=O(1),\sum_{t=1}^{T}\eta_t\rho_t=O(1)$) to show that SGDA with $T$ iterations enjoys the weak PD generalization bound $O\big(1/\sqrt{n}+\sqrt{T}/n\big)$. As compared to Theorem \ref{thm:stability-wcwc}, the assumption \eqref{wcwc-diminishing-ass} allows us to use much larger step sizes ($O(t^{-\beta}),\beta\in(0,1)$ vs $O(t^{-1})$). This larger step size allows for a better trade-off between generalization and optimization. We note that a recent work~\citep{richards2021learning} considered gradient descent under an assumption similar to \eqref{wcwc-diminishing-ass}, and developed interesting generalization bounds for $\eta_t=O(t^{-\beta})$ ($\beta\in(0,1)$). However, their discussions do not apply to the important SGD and require an additional assumption on the Lipschitz continuity of Hessian matrix which may be restrictive. It is direct to extend Theorem \ref{thm:wcwc-diminishing} to SGD for learning with weakly-convex functions for relaxing the step size under Eq. \eqref{wcwc-diminishing-ass}. Therefore, our stability analysis even gives novel results in the standard nonconvex learning setting.
We introduce a novel technique in achieving this improvement. Specifically, let $\delta_t:=\|\bw_t-\bw_t'\|_2^2+\|\bv_t-\bv_t'\|_2^2$, where $(\bw_t,\bv_t),(\bw_t',\bv_t')$ are SGDA iterates for neighboring datasets $S$ and $S'$. For the stability bounds in Section \ref{sec:stab-convex}, we first handle $\delta_{t+1}$ according to different realizations of $i_t$ and then consider the expectation w.r.t. $i_t$. While for $\rho$-WC-WC problems, we first take expectation w.r.t. $i_t$ and then show how $\ebb_{i_t}[\delta_{t+1}]$ would change along the iterations.
%It was posed a challenging problem as whether the relaxation of step size c
\end{remark}

\vspace*{-0.02cm}
\subsection{Stability and Generalization of AGDA and Beyond}
\vspace*{-0.02cm}

%%%%%%%%%%%%%%%%%%%%%%%%%%%%%%%%%%%%%%%%%%%%%%%%% AGDA %%%%%%%%%%%%%%%%%%%%%%%%%%%%%%%%%%%%%%%%%%%

We now study the Alternating Gradient Descent Ascent (AGDA) proposed recently to optimize nonconvex-nonconcave problems \cite{yang2020global}. Let $\{\eta_{\bw,t}, \eta_{\bv,t}\}_t$ be a sequence of positive stepsizes for updating $\{\bw_t,\bv_t\}_t$. At each iteration, we randomly draw $i_t$ and $j_t$ from the uniform distribution over $[n]$ and do the update
\begin{equation}\label{agda}
  \hspace*{-0.3cm}\begin{cases}
    \bw_{t+1}=\text{Proj}_{\wcal}\big(\bw_t-\eta_{\bw,t}\nabla_{\bw} f(\bw_t,\bv_t;z_{i_t})\big), & \hspace*{-0.3cm}\\
    \bv_{t+1}=\text{Proj}_{\vcal}\big(\bv_t+\eta_{\bv,t}\nabla_{\bv} f(\bw_{t+1},\bv_t;z_{j_t})\big). & \hspace*{-0.3cm}
  \end{cases}
\end{equation}
This algorithm differs from SGDA in two aspects. First, it randomly selects two examples to update $\bw$ and $\bv$ per iteration. Second, it uses the updated $\bw_{t+1}$ when updating $\bv_{t+1}$. Theorem \ref{thm:generalization-agda} to be proved in Appendix \ref{sec:agda-nonconvex} provides generalization bounds for AGDA.

%This algorithm is closely related to SGDA with the only difference: while SGDA utilizes the same random example to simultaneously update $\bw$ and $\bv$, AGDA updates $\bw$ and $\bv$ individually with two random examples. This simple modification detangles the dependence between $\bw$ and $\bv$ and therefore the author was able to provide provable global convergence under nonconvex-nonconcave setting with further two-sided PL condition assumption.

%We will provide stability and generalization analysis of AGDA and justify our choice of stepsizes by comparing to \cite{yang2020global}.

%\begin{lemma}\label{lem:stability-agda}
%Let $S = \{z_1, \cdots, z_n\}$ and $S' = \{z_1, \cdot, z_{n-1},z'_n\}$. Let $\{\bw_t,\bv_t\}$ and $\{\bw'_t,\bv'_t\}$ be the sequence produced by \eqref{agda}. If Assumptions \ref{ass:lipschitz} and \ref{ass:smooth} hold, and $\eta_{\bw,t} + \eta_{\bv,t} = \frac{c}{t}$, then AGDA is $\epsilon$-weakly stable with
%\begin{align*}
%\epsilon \leq & \Big(\frac{G}{L}\Big)^{\frac{1}{cL+1}}n^{-1}T^{\frac{cL}{cL+1}}.
%\end{align*}
%\end{lemma}

%In \cite{yang2020global}, the convergence is established based on a potential function $P_t = R_S(\bw_t) - R_S(\bw_*(\bv_*)) + \lambda(R_S(\bw_t) - F_S(\bw_t, \bv_t))$ in expectation, where $\bv_* = \max_{\bv \in \vcal} F_S(\bw,\bv)$ and $\bw_*(\bv_*) = \min_{\bw \in \wcal} F(\bw,\bv_*)$. We connect the potential function and weak PD empirical risk by the following lemma.

%\begin{lemma}\label{lem:agda-opt}
%The following statement is true
%\begin{align*}
%\triangle_S^w(\bw_t,\bv_t) \leq \frac{1}{\lambda} P_t
%\end{align*}
%\end{lemma}

\begin{theorem}[Weak generalization bounds]\label{thm:generalization-agda}
Let $\{\bw_t,\bv_t\}$ be the sequence produced by \eqref{agda}. If Assumptions \ref{ass:lipschitz}, \ref{ass:smooth} hold and $\eta_{\bw,t} + \eta_{\bv,t} \leq \frac{c}{t}$ for some $c>0$, then the weak PD generalization error can be upper bounded by
$\ocal\big(n^{-1}T^{\frac{cL}{cL+1}}\big).$
\end{theorem}
%footnote{The convergence is measured by the potential function $\ebb[R_S(\bw_t) - \min_{\bw}R_S(\bw)+ \lambda(R_S(\bw_t) - F_S(\bw_t,\bv_t))]$ and it equals 0 if and only if $(\bw_t,\bv_t)$ is a minimax point.}
Global convergence of AGDA was studied based on the two-sided PL condition defined below \citep{yang2020global}, which means the suboptimality of function values can be bounded by gradients and were shown for several rich classes of functions~\citep{karimi2016linear}. We also refer to the two-sided PL condition as the gradient dominance condition.
\begin{assumption}\label{ass:pl-two}
  Assume $F_S$ satisfies the two-sided PL condition, i.e., there exist constants $\beta_1(S),\beta_2(S)>0$ such that the following inequalities hold for all $\bw\in\wcal,\bv\in\vcal$
  \[
  2\beta_1(S)\big(F_S(\bw,\bv)-\inf_{\bw'\in\wcal}F_S(\bw',\bv)\big)\leq \|\nabla_{\bw}F_S(\bw,\bv)\|_2^2,
  \]
  \[
  2\beta_2(S)\big(\sup_{\bv'\in\vcal}F_S(\bw,\bv')-F_S(\bw,\bv)\big)\leq \|\nabla_{\bv}F_S(\bw,\bv)\|_2^2.
  \]
\end{assumption}
As a combination of our generalization bounds and optimization error bounds in \citet{yang2020global}, we can derive the following informal corollary on primal population risks by early stopping the algorithm to balance the optimization and generalization. It gives the first primal risk bounds for learning with nonconvex-strongly-concave functions. The precise statement can be found in Corollary \ref{cor:pl-adga}.
\begin{corollary}[Informal]
Let $\beta_1,\rho>0$. Assume for all $\bw$, the function $\bv\mapsto F(\bw,\bv)$ is $\rho$-strongly concave.  Let Assumptions \ref{ass:lipschitz}, \ref{ass:smooth}, \ref{ass:pl-two}  with $\beta_1(S)\geq \beta_1,\beta_2(S)\geq\rho$ hold. Then AGDA with some appropriate step size and $T\asymp\big(n\beta_1^{-2}\rho^{-3}\big)^{\frac{cL+1}{2cL+1}}$ satisfy ($c\asymp 1/(\beta_1\rho^2)$)
\[
  \ebb\big[R(\bw_T)\big]-R(\bw^*)=O\Big(n^{-\frac{cL+1}{2cL+1}}\beta_1^{-\frac{2cL}{2cL+1}}\rho^{-\frac{5cL+1}{2cL+1}}\Big).
\]
\end{corollary}

For gradient dominated problems, we further have the following error bounds to be proved in Appendix~\ref{sec:proof-pl}. Note we do not need the smoothness assumption here.
\begin{theorem}\label{thm:gen-pl}
  Let $A$ be an algorithm.
  Let Assumptions \ref{ass:lipschitz}, \ref{ass:pl-two} hold. Let $\bu_S=(A_{\bw}(S),A_{\bv}(S))$ and $\bu_S^{(S)}$ be the projection of $\bu_S$ onto the set of stationary points of $F_S$.
  Then,
  \begin{multline*}
  \big|\ebb\big[F(\bu_S)-F_S(\bu_S)\big]\big|\leq \frac{2G^2}{n}\max\Big\{\ebb\big[1/{\beta_1(S)}\big],\\
  \ebb\big[1/{\beta_2(S)}\big]\Big\}+2G\ebb\big[\|\bu_S-\bu_S^{(S)}\|_2\big].
  \end{multline*}
\end{theorem}
\begin{remark}
Note $\|\bu_S-\bu_S^{(S)}\|_2$ measures how far the point found by $A$ is from the set of stationary points of $F_S$, and can be interpreted as an optimization error.
Therefore, Theorem \ref{thm:gen-pl} gives a connection between generalization error and optimization error.
For a variant of AGDA with noiseless stochastic gradients, it was shown that $\|(\bw_T,\bv_T)-(\bw_T,\bv_T)^{(S)}\|_2$ decays linearly w.r.t. $T$~\citep{yang2020global}. We can plug this linear convergent optimization bound into Theorem \ref{thm:gen-pl} to directly get generalization bounds.
If $A$ returns a saddle point of $F_S$,
then $\|\bu_S-\bu_S^{(S)}\|_2=0$ and therefore $\big|\ebb\big[F(\bu_S)-F_S(\bu_S)\big]\big|=O\big(n^{-1}\max\{\ebb[1/\beta_1(S)],\ebb[1/\beta_2(S)]\}\big)$. Generalization errors of this particular ESP were studied in \citet{zhang2020generalization} for SC-SC minimax problems, which were extended to more general gradient-dominated problems in Theorem \ref{thm:gen-pl}. Furthermore, Theorem \ref{thm:gen-pl} applies to any optimization algorithm instead of the specific ESP. It should be mentioned that \citet{zhang2020generalization} addressed PD population risks, while we consider plain generalization errors.
\end{remark}

%{\color{red} Neyo: Our generalization bound also aligns with \citep{yang2020global} who showed the potential function\footnote{The potential function is given by $\ebb[\max_{\bv}F_S(\bw_t,\bv) - \min_{\bw}\max_{\bv}F_S(\bw,\bv)+ \lambda(\max_{\bv}F_S(\bw_t,\bv) - F_S(\bw_t,\bv_t))]$ and it equals 0 if and only if $(\bw_t,\bv_t)$ is a minimax point.} converges at the rate $\ocal(T^{-1})$ by running AGDA under two-sided PL condition.}
%\vspace*{-0.1cm}

\vspace*{-0.060cm}
\section{Experiments\label{sec:exp}}
\vspace*{-0.060cm}

In this subsection, we report preliminary experimental results to validate our theoretical results\footnote{The source codes are available at \url{https://github.com/zhenhuan-yang/minimax-stability}.}.
We consider two datasets available at the LIBSVM website: \texttt{svmguide3} and \texttt{w5a}~\citep{chang2011libsvm}.  We follow the experimental setup in \citet{hardt2016train} to study how the stability of SGDA would behave along the learning process. To this end, we build a neighboring dataset $S'$ by changing the last example of the training set $S$. We apply the same randomized algorithm to $S$ and $S'$ and get two model sequences $\{(\bw_t,\bv_t)\}$ and $\{(\bw_t',\bv_t')\}$. Then we  evaluate the Euclidean distance $\triangle_t=\big(\|\bw_t-\bw_t'\|_2^2+\|\bv_t-\bv_t'\|_2^2\big)^{\frac{1}{2}}$. We consider the \texttt{SOLAM} \citep{ying2016stochastic} algorithm, which is the SGDA for the solving the problem \eqref{solam} (a minimax reformulation of the AUC maximization problem). We consider step sizes $\eta_t=\eta/\sqrt{T}$ with $\eta\in\{0.1,0.3,1,3\}$.
We repeat the experiments 25 times and report the average of the experimental results as well as the standard deviation. In Figure \ref{fig:solam}, we report $\triangle_t$ as a function of the number of passes (the iteration number divided by $n$). %For each dataset, we use 80 percents of the dataset as the training dataset $S$.
It is clear that the Euclidean distance continues to increase during the learning process. Furthermore, the Euclidean distances increase if we consider larger step sizes. This phenomenon is consistent with our stability bounds in Theorem \ref{lem:stab-gda}. More experiments on the stability of SGDA in GAN training can be found in Appendix \ref{sec:exp-more}.

\begin{figure}[ht]
  \vspace*{-0.1cm}
  \setlength{\abovecaptionskip}{0pt}
  \setlength{\belowcaptionskip}{3pt}
  \centering
  \hspace*{-0.6cm}
\includegraphics[width=0.48\textwidth]{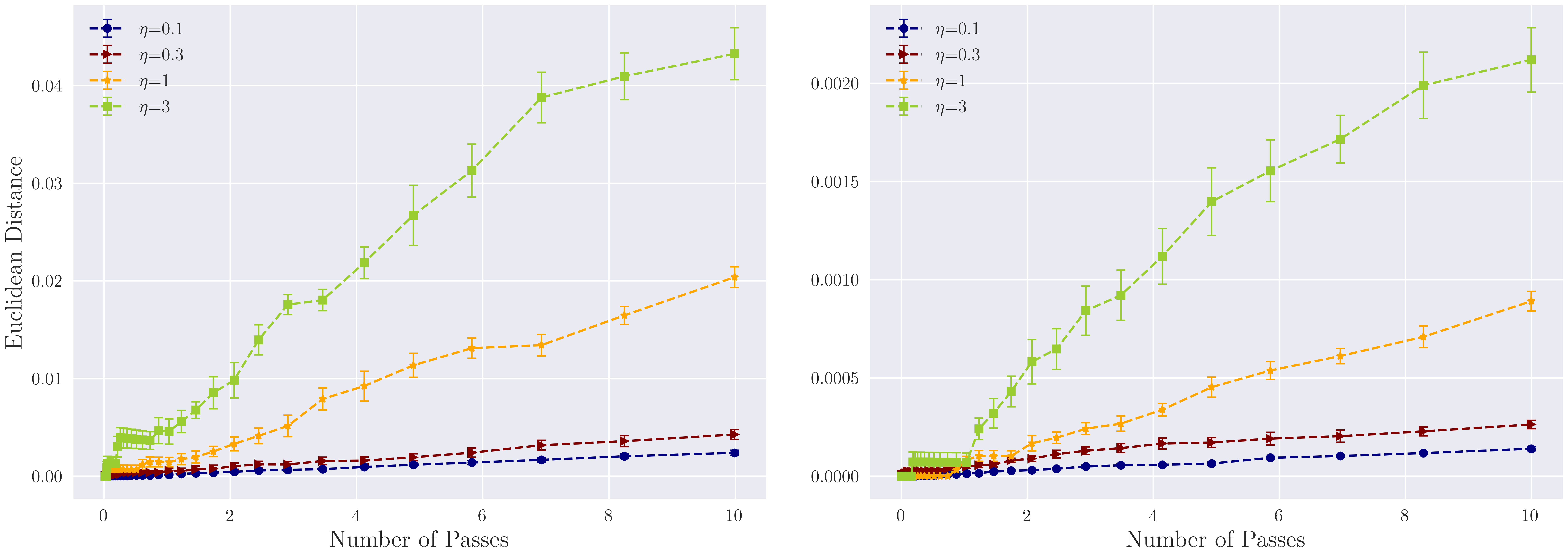}
  \vspace*{-0.1cm}
\caption{\label{fig:solam}$\triangle_t$ versus the number of passes on  \texttt{svmguide3} (left) and  \texttt{w5a} (right).} % In the left, we show the result for svmguide3 and in the right, we show the result for w5a.}
\end{figure}

\vspace*{-0.30cm}
\section{Conclusion\label{sec:conclusion}}
\vspace*{-0.060cm}

%\vspace*{-0.03cm}
%We establish the first systematic connection between various stability measures and generalization for minimax setting (Thm. 1). We develop the first minimax optimal risk bounds for SGDA in a general convex-concave case (Thm. 3 \& 4), covering both smooth and nonsmooth setting. We also give the first nontrivial risk bounds for nonconvex-nonconcave problems (Cor. 8).

We presented a comprehensive stability and generalization analysis of stochastic algorithms for minimax objective functions. We introduced various generalization measures and stability measures, and present a systematic study on their quantitative relationship.
In particular, we obtained the first minimax optimal risk bounds for SGDA in a general convex-concave case, covering both smooth and nonsmooth setting. We also derived the first non-trivial risk bounds for nonconvex-nonconcave problems. Our bounds show how to early-stop the algorithm in practice to train a model with better generalization. Our theoretical results have potential applications in
developing differentially private algorithms to handle sensitive data.
%We establish optimal generalization bounds of SGDA for convex-concave problems in both smooth and nonsmooth cases. Our analysis gives the first high-probability generalization bounds for minimax problems. Furthermore, we extend our discussions to several algorithms in the nonconvex-nonconcave case, and derive meaningful stability and excess risk bounds under some regularity conditions.

There are some interesting problems for further investigation. Our primal generalization bounds require a strong concavity assumption. It is interesting to remove this assumption.  On the other front,  it remains an open question to us on understanding how the concavity of dual variables can help generalization in a nonconvex setting. %In the nonconvex setting, it would be interesting and challenging to derive excess risk bound.

\section*{Acknowledgments}
We are grateful to the anonymous reviewers for their constructive comments and suggestions.
The work of Yiming Ying is partially supported by NSF grants IIS-1816227 and IIS-2008532.
The work of Tianbao Yang is partially supported by NSF \#1933212 and NSF CAREER Award \#1844403.

\setlength{\bibsep}{0.2cm}
%\bibliography{learning}
%\bibliographystyle{icml2021}


\onecolumn
\appendix
\numberwithin{equation}{section}
\numberwithin{theorem}{section}
\numberwithin{remark}{section}
\numberwithin{figure}{section}
\numberwithin{table}{section}
\renewcommand{\thesection}{{\Alph{section}}}
\renewcommand{\thesubsection}{\Alph{section}.\arabic{subsection}}
\renewcommand{\thesubsubsection}{\Roman{section}.\arabic{subsection}.\arabic{subsubsection}}
\setcounter{secnumdepth}{-1}
\setcounter{secnumdepth}{3}

\vspace*{0cm}
\begin{center}
  \Large \textbf{Appendix for ``Stability and Generalization of Stochastic Gradient Methods for Minimax Problems''}
\bigskip
\end{center}
\section{Notations\label{sec:notation}}

We collect in Table \ref{tab:notation} the notations of performance measures used in this paper.

\begin{table}[h]
  \centering\renewcommand{\arraystretch}{1.6}
  \begin{tabular}{|c|c|c|c|}
    \hline
     & Notation & Meaning & Definition \\ \hline
    \multirow{4}{*}{\shortstack{Weak\\ Measure}} & $\triangle^w({\bw},{\bv})$  & weak PD population risk &  $\displaystyle\sup_{\bv'\in\vcal}\ebb\big[F({\bw},\bv')\big]-\inf_{\bw'\in\wcal}\ebb\big[F(\bw',{\bv})\big]$ \\ \cline{2-4}
     &  $\displaystyle\triangle^w_S({\bw},{\bv})$ & weak PD empirical risk & $\displaystyle\sup_{\bv'\in\vcal}\ebb\big[F_S({\bw},\bv')\big]-\inf_{\bw'\in\wcal}\ebb\big[F_S(\bw',{\bv})\big]$ \\ \cline{2-4}
     & \multirow{2}{*}{$\displaystyle\triangle^w({\bw},{\bv})-\triangle^w_S({\bw},{\bv})$} & \multirow{2}{*}{weak PD generalization error}  & $\displaystyle\big(\sup_{\bv'\in\vcal}\ebb\big[F({\bw},\bv')\big]-\sup_{\bv'\in\vcal}\ebb\big[F_S(\bw,\bv')\big]\big)$ \\
     & & & $\displaystyle+ \big(\inf_{\bw'\in\vcal}\ebb\big[F_S(\bw',\bv)\big]-\inf_{\bw'\in\wcal}\ebb\big[F(\bw',\bv)\big]\big)$ \\ \hline
    \multirow{4}{*}{\shortstack{Strong\\ Measure}} & $\triangle^s({\bw},{\bv})$ & strong PD population Risk & $\displaystyle\sup_{\bv'\in\vcal}F({\bw},\bv')-\inf_{\bw'\in\wcal}F(\bw',{\bv})$ \\ \cline{2-4}
     & $\displaystyle\triangle^s_S({\bw},{\bv})$ & strong  PD empirical Risk & $\displaystyle\sup_{\bv'\in\vcal}F_S({\bw},\bv')-\inf_{\bw'\in\wcal}F_S(\bw',{\bv})$ \\ \cline{2-4}
     & \multirow{2}{*}{$\displaystyle\triangle^s({\bw},{\bv})-\triangle^s_S({\bw},{\bv})$} & \multirow{2}{*}{strong PD generalization error} & $\displaystyle\big(\sup_{\bv'\in\vcal}F({\bw},\bv')-\sup_{\bv'\in\wcal}F_S(\bw,\bv')\big)$ \\
     & & & $\displaystyle+\big(\inf_{\bw'\in\vcal}F_S(\bw',\bv)-\inf_{\bw'\in\wcal}F(\bw',\bv)\big)$\\ \hline
     \multirow{3}{*}{\shortstack{Primal\\ Measure}}& $\displaystyle R(\bw) - \inf_{\bw' \in \wcal}R(\bw')$  & excess primal population risk &  $\displaystyle\sup_{\bv'\in\vcal}F(\bw,\bv') - \inf_{\bw' \in \wcal}\sup_{\bv'\in\vcal}F(\bw',\bv')$\\ \cline{2-4}
     & $\displaystyle R_S(\bw) - \inf_{\bw' \in \wcal}R_S(\bw')$ & excess primal empirical risk  & $\displaystyle\sup_{\bv'\in\vcal}F_S(\bw,\bv') - \inf_{\bw' \in \wcal}\sup_{\bv'\in\vcal}F_S(\bw',\bv')$ \\ \cline{2-4}
     & $\displaystyle R(\bw) - R_S(\bw)$ & primal generalization error & $\displaystyle \sup_{\bv'\in\vcal}F(\bw,\bv')-\sup_{\bv'\in\vcal}F_S(\bw,\bv')$  \\
    \hline
    & $\displaystyle F(\bw,\bv)-F_S(\bw,\bv)$ & plain generalization error &  \\
    \hline
  \end{tabular}
  \caption{Notations on Measures of Performance.}\label{tab:notation}
\end{table}

%\bigskip

We collect in Table \ref{tab:stability} the stability measures for a (randomized) algorithm $A$.
\begin{table}[h]
  \centering\renewcommand{\arraystretch}{1.6}
  \scalebox{0.9}{\begin{tabular}{|c|c|}
    \hline
    Stability Measure & Definition \\ \hline
    Weak Stability & $\displaystyle \sup_{z}\Big(\sup_{\bv'\in\vcal}\ebb_A\big[f(A_{\bw}(S),\bv';z)-f(A_{\bw}(S'),\bv';z)\big]+\sup_{\bw'\in\wcal}\ebb_A\big[f(\bw',A_{\bv}(S);z)-f(\bw',A_{\bv}(S');z)\big]\Big)$ \\ \hline
    Argument Stability & $\displaystyle \ebb_A\left[\left\|\begin{pmatrix}
    A_{\bw}(S)-A_{\bw}(S') \\
    A_{\bv}(S)-A_{\bv}(S')
  \end{pmatrix}\right\|_2\right]\text{ or } \left\|\begin{pmatrix}
    A_{\bw}(S)-A_{\bw}(S') \\
    A_{\bv}(S)-A_{\bv}(S')
  \end{pmatrix}\right\|_2$ \\ \hline
    Uniform Stability & $\displaystyle \sup_{z}\big[f(A_{\bw}(S),A_{\bv}(S);z)-f(A_{\bw}(S'),A_{\bv}(S');z)\big]$ \\
    \hline
  \end{tabular}}
  \caption{Stability Measures. Here $S$ and $S'$ are neighboring datasets.}\label{tab:stability}
\end{table}

\newpage
\section{Proof of Theorem \ref{thm:stab-gen}\label{sec:proof-stab-gen}}
In this section, we prove Theorem \ref{thm:stab-gen} on the connection between stability measure and generalization.

Let $S=\{z_1,\ldots,z_n\}$ and $S'=\{z_1',\ldots,z_n'\}$ be two datasets drawn from the same distribution. For any $i\in[n]$, define $S^{(i)}=\{z_1,\ldots,z_{i-1},z_i',z_{i+1},\ldots,z_n\}$.
For any function $g,\tilde{g}$, we have the basic inequalities
\begin{equation}\label{sup-inf}
  \begin{split}
      & \sup_{\bw}g(\bw)-\sup_{\bw}\tilde{g}(\bw)\leq \sup_{\bw}\big(g(\bw)-\tilde{g}(\bw)\big)\\
      & \inf_{\bw}g(\bw)-\inf_{\bw}\tilde{g}(\bw)\leq \sup_{\bw}\big(g(\bw)-\tilde{g}(\bw)\big).
  \end{split}
\end{equation}

\subsection{Proof of Part (a)}
We first prove Part (a).
It follows from \eqref{sup-inf} that
\begin{multline*}
  \triangle^w(A_{\bw}(S),A_{\bv}(S))-\triangle^w_S(A_{\bw}(S),A_{\bv}(S)) \leq \sup_{\bv'\in\vcal}\ebb[F(A_{\bw}(S),\bv')-F_S(A_{\bw}(S),\bv')]\\+
  \sup_{\bw'\in\wcal}\ebb[F_S(\bw',A_{\bv}(S))-F(\bw',A_{\bv}(S))].
\end{multline*}
According to the symmetry between $z_i$ and $z_i'$ we know
\begin{align*}
  \ebb[F(A_{\bw}(S),\bv')-F_S(A_{\bw}(S),\bv')] & =  \frac{1}{n}\sum_{i=1}^{n} \ebb[F(A_{\bw}(S^{(i)}),\bv')]-\ebb[F_S(A_{\bw}(S),\bv')]\\
  &=  \frac{1}{n}\sum_{i=1}^{n} \ebb\big[f(A_{\bw}(S^{(i)}),\bv';z_i)-f(A_{\bw}(S),\bv';z_i)\big],
\end{align*}
where the second identity holds since $z_i$ is not used to train $A_{\bw}(S^{(i)})$.
In a similar way, we can prove
\[
\ebb[F_S(\bw',A_{\bv}(S))-F(\bw',A_{\bv}(S))]=
\frac{1}{n}\sum_{i=1}^{n}\big[f(\bw',A_{\bv}(S^{(i)});z_i)-f(\bw',A_{\bv}(S);z_i)\big].
\]
As a combination of the above three inequalities we get
\begin{multline*}
    \triangle^w(A_{\bw}(S),A_{\bv}(S))-\triangle^w_S(A_{\bw}(S),A_{\bv}(S)) \leq  \sup_{\bv'\in\vcal}\Big[\frac{1}{n}\sum_{i=1}^{n} \ebb\big[f(A_{\bw}(S^{(i)}),\bv';z_i)-f(A_{\bw}(S),\bv';z_i)\big]\Big]+\\
    \sup_{\bw'\in\wcal}\Big[\frac{1}{n}\sum_{i=1}^{n}\big[f(\bw',A_{\bv}(S^{(i)});z_i)-f(\bw',A_{\bv}(S);z_i)\big]\Big].
\end{multline*}
The stated bound in Part (a) then follows directly from the definition of stability.
% \big\|\nabla_{\bw}\phi(\bw,\bv)-\nabla_{\bw}\phi(\bw,\bv')\big\|_2\leq %L\|\bv-\bv'\|_2,\quad
\subsection{Proof of Part (b)}
The following lemma quantifies the sensitivity of the optimal $\bv$ w.r.t. the perturbation of $\bw$.
\begin{lemma}[\citealt{lin2020gradient}]\label{lem:mengdi}
 Let $\phi:\wcal\times\vcal\mapsto\rbb$. Assume that for any $\bw$, the function $\bv\mapsto\phi(\bw,\bv)$ is $\rho$-strongly-concave. Suppose for any $(\bw,\bv),(\bw',\bv')$ we have
 \[
 \big\|\nabla_{\bv}\phi(\bw,\bv)-\nabla_{\bv}\phi(\bw',\bv)\big\|_2\leq L\|\bw-\bw'\|_2.
 \]
 For any $\bw$, denote $\bv^*(\bw)=\arg\max_{\bv\in\vcal}\phi(\bw,\bv)$.
 Then for any $\bw,\bw'\in\wcal$, we have
 \[
 \big\|\bv^*(\bw)-\bv^*(\bw')\|_2\leq \frac{L}{\rho}\|\bw-\bw'\|_2.
 \]
\end{lemma}
We now prove Part (b). For any $S$, let $\bv_S^*=\arg\max_{\bv\in\vcal} F(A_{\bw}(S),\bv)$.
According to the symmetry between $z_i$ and $z_i'$ we know
\begin{align*}
    \ebb\big[\sup_{\bv'\in\vcal}F(A_{\bw}(S),\bv')\big]
    & =\frac{1}{n}\sum_{i=1}^n\ebb\big[\sup_{\bv'\in\vcal}F(A_{\bw}(S^{(i)}),\bv')\big]\\
    & = \frac{1}{n}\sum_{i=1}^n\ebb\big[F(A_{\bw}(S^{(i)}),\bv^*_{S^{(i)}})\big]
      = \frac{1}{n}\sum_{i=1}^n\ebb\big[f(A_{\bw}(S^{(i)}),\bv^*_{S^{(i)}};z_i)\big],
\end{align*}
where the last identity holds since $z_i$ is independent of $A_{\bw}(S^{(i)})$ and $\bv^*_{S^{(i)}}$.

%\begin{align*}
%    \sup_{\bv' \in \vcal}\ebb_{z_i}\big[f(A_{\bw}(S^{(i)}),\bv';z_i)\big]&
%    = \sup_{\bv' \in \vcal}\ebb_{z_i}[f(A_{\bw}(S^{(i)}),\bv';z_i)-f(A_{\bw(S)},\bv_S^*;z_i)]+\ebb_{\bz_i}[f(A_{\bw}(S),\bv_S^*;z_i)]
%\end{align*}
%$\sup_{\bv'\in\vcal}f(A_{\bw}(S^{(i)}),\bv';z)=f(A_{\bw}(S^{(i)}),\bv^*_{S^{(i)}};z)$

According to Assumption \ref{ass:lipschitz}, we know
\begin{align}
  & f(A_{\bw}(S^{(i)}),\bv^*_{S^{(i)}};z_i) - f(A_{\bw}(S),\bv_S^*;z_i) \notag\\
  & = f(A_{\bw}(S^{(i)}),\bv^*_{S^{(i)}};z_i) - f(A_{\bw}(S^{(i)}),\bv_S^*;z_i) + f(A_{\bw}(S^{(i)}),\bv^*_{S};z_i) - f(A_{\bw}(S),\bv_S^*;z_i)\notag\\
  & \leq G \big\|A_{\bw}(S^{(i)})-A_{\bw}(S)\big\|_2+ G\big\|\bv^*_{S^{(i)}}-\bv_S^*\big\|_2
  \leq \big(1+L/\rho\big)G\big\|A_{\bw}(S^{(i)})-A_{\bw}(S)\big\|_2,\label{stab-gen-0}
\end{align}
where in the last inequality we have used Lemma \ref{lem:mengdi} due to the strong concavity of $\bv\mapsto F(\bw,\bv)$ for any $\bw$.
As a combination of the above two inequalities, we get
\begin{align}
\ebb\big[\sup_{\bv'\in\vcal}F(A_{\bw}(S),\bv')\big]&\leq \frac{1}{n}\sum_{i=1}^n\ebb\big[f(A_{\bw}(S),\bv_S^*;z_i)\big]+ \frac{\big(1+L/\rho\big)G}{n}\sum_{i=1}^n\ebb\big[\big\|A_{\bw}(S^{(i)})-A_{\bw}(S)\big\|_2\big]\notag\\
&=
\ebb\big[F_S(A_{\bw}(S),\bv_S^*)\big]+ \frac{\big(1+L/\rho\big)G}{n}\sum_{i=1}^n\ebb\big[\big\|A_{\bw}(S^{(i)})-A_{\bw}(S)\big\|_2\big]\notag\\
& \leq \ebb\big[\sup_{\bv'\in\vcal}F_S(A_{\bw}(S),\bv')\big]+ \frac{\big(1+L/\rho\big)G}{n}\sum_{i=1}^n\ebb\big[\big\|A_{\bw}(S^{(i)})-A_{\bw}(S)\big\|_2\big].\label{stab-gen-1}
\end{align}
The stated bound in Part (b) then follows.

\subsection{Proof of Part (c)}
In a similar way, one can show that
\[
\ebb\big[\inf_{\bw'\in\wcal}F_S(\bw',A_{\bv}(S))\big]-\ebb\big[\inf_{\bw'\in\wcal}F(\bw',A_{\bv}(S))\big]
\leq \frac{\big(1+L/\rho\big)G}{n}\sum_{i=1}^n\ebb\big[\|A_{\bv}(S^{(i)})-A_{\bv}(S)\|_2\big].
\]
The above inequality together with \eqref{stab-gen-1} then implies
\begin{align*}
& \ebb\big[\triangle^s(A_{\bw}(S),A_{\bv}(S))\big]
-\ebb\big[\triangle_S^s(A_{\bw}(S),A_{\bv}(S))\big]\\
& = \ebb\big[\sup_{\bv'\in\vcal}F(A_{\bw}(S),\bv')\big]-\ebb\big[\sup_{\bv'\in\vcal}F_S(A_{\bw}(S),\bv')\big]+\ebb\big[\inf_{\bw'\in\wcal}F_S(\bw',A_{\bv}(S))\big]-\ebb\big[\inf_{\bw'\in\wcal}F(\bw',A_{\bv}(S))\big]\\
&\leq
{\big(1+L/\rho\big)G}\ebb\big[\|A_{\bw}(S^{(i)})-A_{\bw}(S)\|_2\big]+
{\big(1+L/\rho\big)G}\ebb\big[\|A_{\bv}(S^{(i)})-A_{\bv}(S)\|_2\big]\\
& \leq {\big(1+L/\rho\big)G\sqrt{2}}\ebb\Big[\Big\|\begin{pmatrix}
A_{\bw}(S^{(i)})-A_{\bw}(S) \\
A_{\bv}(S^{(i)})-A_{\bv}(S)
\end{pmatrix}\Big\|_2\Big],
\end{align*}
where we have used the elementary inequality $a+b\leq \sqrt{2(a^2+b^2)}$.
This proves the stated bound in Part (c).

\subsection{Proof of Part (d)}
To prove Part (d) on high-probability bounds, we need to introduce some lemmas.

The following lemma establishes a concentration inequality for a summation of weakly-dependent random variables.
We denote by $S\backslash\{z_i\}$ the set $\{z_1,\ldots,z_{i-1},z_{i+1},\ldots,z_n\}$.
The $L_p$-norm of a random variable $Z$ is denoted by $\|Z\|_p:=\big(\ebb[|Z|^p]\big)^{\frac{1}{p}},p\geq1$.
\begin{lemma}[\citealt{bousquet2020sharper}\label{lem:feldman}]
  Let $S=\{z_1,\ldots,z_n\}$ be a set of independent random variables each taking values in $\zcal$ and $M>0$. Let $g_1,\ldots,g_n$ be some functions $g_i:\zcal^n\mapsto\rbb$ such that the following holds for any $i\in[n]$
  \begin{itemize}
    \item $\big|\ebb_{S\backslash\{z_i\}}[g_i(S)]\big|\leq M$ almost surely (a.s.),
    \item $\ebb_{z_i}\big[g_i(S)\big]=0$ a.s.,
    \item for any $j\in[n]$ with $j\neq i$, and $z_j''\in\zcal$
    \begin{equation}\label{bounded-increment}
      \big|g_i(S)-g_i(z_1,\ldots,z_{j-1},z_j'',z_{j+1},\ldots,z_n)\big|\leq\beta.
    \end{equation}
  \end{itemize}
  Then, for any $p\geq2$
  \[
  \Big\|\sum_{i=1}^{n}g_i(S)\Big\|_p\leq12\sqrt{6}pn\beta\lceil\log_2n\rceil+3\sqrt{2}M\sqrt{pn}.
  \]
\end{lemma}
The following lemma shows how to relate moment bounds of random variables to tail behavior.
\begin{lemma}[\citealt{bousquet2020sharper,vershynin2018high}]\label{lem:tail-moment}
  Let $a,b\in\rbb_+$ and $\delta\in(0,1/e)$. Let $Z$ be a random variable with $\|Z\|_p\leq\sqrt{p}a+pb$ for any $p\geq2$. Then with probability at least $1-\delta$
  \[
  |Z|\leq e\Big(a\sqrt{\log(e/\delta)}+b\log(e/\delta)\Big).
  \]
\end{lemma}
With the above lemmas we are now ready to prove Part (d). For any $S$, denote
\begin{equation}\label{bv-star}
    \bv_S^*=\arg\max_{\bv\in\vcal} F(A_{\bw}(S),\bv).
\end{equation}
We have the following error decomposition
\begin{multline*}
  nF(A_{\bw}(S),\bv_S^*)-n\sup_{\bv'\in\vcal}F_S(A_{\bw}(S),\bv) = \sum_{i=1}^{n}\ebb_Z\big[f(A_{\bw}(S),\bv_S^*;Z)-\ebb_{z_i'}[f(A_{\bw}(S^{(i)}),\bv_{S^{(i)}}^*;Z)]\big]+ \\
  \sum_{i=1}^{n}\ebb_{z_i'}\Big[\ebb_Z[f(A_{\bw}(S^{(i)}),\bv_{S^{(i)}}^*;Z)]-f(A_{\bw}(S^{(i)}),\bv_{S^{(i)}}^*;z_i)\Big]+
  \sum_{i=1}^{n}\ebb_{z_i'}\Big[f(A_{\bw}(S^{(i)}),\bv_{S^{(i)}}^*;z_i)\Big]-n\sup_{\bv'\in\vcal}F_S(A_{\bw}(S),\bv).
\end{multline*}
By the definition of $\bv_{S^{(i)}}^*$ we know $\ebb_Z[f(A_{\bw}(S^{(i)}),\bv_{S^{(i)}}^*;Z)]\geq \ebb_Z[f(A_{\bw}(S^{(i)}),\bv_{S}^*;Z)]$.
It then follows that
\begin{multline*}
  nF(A_{\bw}(S),\bv_S^*)-n\sup_{\bv'\in\vcal}F_S(A_{\bw}(S),\bv') \leq \sum_{i=1}^{n}\ebb_Z\big[f(A_{\bw}(S),\bv_S^*;Z)-\ebb_{z_i'}[f(A_{\bw}(S^{(i)}),\bv_{S}^*;Z)]\big] + \\
  \sum_{i=1}^{n}\ebb_{z_i'}\Big[\ebb_Z[f(A_{\bw}(S^{(i)}),\bv_{S^{(i)}}^*;Z)]-f(A_{\bw}(S^{(i)}),\bv_{S^{(i)}}^*;z_i)\Big]+
  \sum_{i=1}^{n}\ebb_{z_i'}\Big[f(A_{\bw}(S^{(i)}),\bv_{S^{(i)}}^*;z_i)\Big]-n\sup_{\bv'\in\vcal}F_S(A_{\bw}(S),\bv').
\end{multline*}
According to \eqref{stab-gen-0}, we know
\begin{align*}
  \sum_{i=1}^{n}\ebb_{z_i'}\Big[f(A_{\bw}(S^{(i)}),\bv_{S^{(i)}}^*;z_i)\Big] &
  \leq \big(1+L/\rho\big)G\sum_{i=1}^n\ebb_{z_i'}\big[\big\|A_{\bw}(S^{(i)})-A_{\bw}(S)\big\|_2\big] + \sum_{i=1}^nf(A_{\bw}(S),\bv_S^*;z_i)\\
  & = \big(1+L/\rho\big)G\sum_{i=1}^n\ebb_{z_i'}\big[\big\|A_{\bw}(S^{(i)})-A_{\bw}(S)\big\|_2\big] + nF_S(A_{\bw}(S),\bv_S^*)\\
  & \leq \big(1+L/\rho\big)G\sum_{i=1}^n\ebb_{z_i'}\big[\big\|A_{\bw}(S^{(i)})-A_{\bw}(S)\big\|_2\big] + n\sup_{\bv'\in\vcal}F_S(A_{\bw}(S),\bv').
\end{align*}
As a combination of the above two inequalities, we derive
\begin{equation}\label{gen-stab-6}
nF(A_{\bw}(S),\bv_S^*)-n\sup_{\bv'\in\vcal}F_S(A_{\bw}(S),\bv')\leq \big(2+L/\rho\big)nG\epsilon+\sum_{i=1}^{n}g_i(S),
\end{equation}
where we introduce
\[
g_i(S)=\ebb_{z_i'}\Big[\ebb_Z[f(A_{\bw}(S^{(i)}),\bv_{S^{(i)}}^*;Z)]-f(A_{\bw}(S^{(i)}),\bv_{S^{(i)}}^*;z_i)\Big]
\]
and use the inequality \[f(A_{\bw}(S),\bv_S^*;Z)-f(A_{\bw}(S^{(i)}),\bv_{S}^*;Z)\leq G\|A_{\bw}(S)-A_{\bw}(S^{(i)})\|_2\leq G\epsilon.\]
Due to the symmetry between $z_i$ and $Z$, we know $\ebb_{z_i}[g_i(S)]=0$. The inequality $|\ebb_{S\backslash\{z_i\}}[g_i(S)]|\leq 2R$ is also clear.

For any $j\neq i$ and any $z_j''\in\zcal$, we know
\begin{align*}
  & \Big|\ebb_{z_i'}\Big[\ebb_Z[f(A_{\bw}(S^{(i)}),\bv_{S^{(i)}}^*;Z)]-f(A_{\bw}(S^{(i)}),\bv_{S^{(i)}}^*;z_i)\Big]
  -\ebb_{z_i'}\Big[\ebb_Z[f(A_{\bw}(S_j^{(i)}),\bv_{S_j^{(i)}}^*;Z)]-f(A_{\bw}(S_j^{(i)}),\bv_{S_j^{(i)}}^*;z_i)\Big]\Big|  \\
   & \leq \Big|\ebb_{z_i'}\Big[\ebb_Z[f(A_{\bw}(S^{(i)}),\bv_{S^{(i)}}^*;Z)]-\ebb_Z[f(A_{\bw}(S_j^{(i)}),\bv_{S_j^{(i)}}^*;Z)]\Big]\Big|+
   \Big|\ebb_{z_i'}\Big[f(A_{\bw}(S^{(i)}),\bv_{S^{(i)}}^*;z_i)-f(A_{\bw}(S_j^{(i)}),\bv_{S_j^{(i)}}^*;z_i)\Big]\Big|,
\end{align*}
where $S_j^{(i)}$ is the set derived by replacing the $j$-th element of $S^{(i)}$ with $z_j''$. For any $z$, there holds
\begin{align*}
  & \big|f(A_{\bw}(S^{(i)}),\bv_{S^{(i)}}^*;z)-f(A_{\bw}(S_j^{(i)}),\bv_{S_j^{(i)}}^*;z)\big| \\
  & \leq \big|f(A_{\bw}(S^{(i)}),\bv_{S^{(i)}}^*;z)-f(A_{\bw}(S^{(i)}),\bv_{S_j^{(i)}}^*;z)\big|
  + \big|f(A_{\bw}(S^{(i)}),\bv_{S_j^{(i)}}^*;z)-f(A_{\bw}(S_j^{(i)}),\bv_{S_j^{(i)}}^*;z)\big|\\
  & \leq G\|\bv_{S^{(i)}}^*-\bv_{S_j^{(i)}}^*\|_2+G\|A_{\bw}(S^{(i)})-A_{\bw}(S_j^{(i)})\|_2\leq \big(L/\rho+1\big)G\|A_{\bw}(S^{(i)})-A_{\bw}(S_j^{(i)})\|_2,
\end{align*}
where in the last inequality we have used the definition of $\bv_{S^{(i)}}^*$ and Lemma \ref{lem:mengdi} with $\phi=F$. Therefore $g_i(S)$ satisfies the condition \eqref{bounded-increment} with $\beta=\big(L/\rho+1\big)G\epsilon$. Therefore, all the conditions of Lemma \ref{lem:feldman} hold and we can apply Lemma \ref{lem:feldman} to derive the following inequality for any $p\geq2$
\[
\Big\|\sum_{i=1}^{n}g_i(S)\Big\|_p\leq12\sqrt{6}pn\big(L/\rho+1\big)G\epsilon\lceil\log_2n\rceil+6\sqrt{2}R\sqrt{pn}.
\]
This together with Lemma \ref{lem:tail-moment} implies the following inequality with probability $1-\delta$
\[
\Big|\sum_{i=1}^{n}g_i(S)\Big|\leq e\Big(6R\sqrt{2n\log(e/\delta)}+12\sqrt{6}n\big(L/\rho+1\big)G\epsilon\log(e/\delta)\lceil\log_2n\rceil\Big).
\]
We can plug the above inequality back into \eqref{gen-stab-6} and derive the following inequality with probability at least $1-\delta$
\[
F(A_{\bw}(S),\bv_S^*)-\sup_{\bv'\in\vcal}F_S(A_{\bw}(S),\bv')\leq \big(2+L/\rho\big)G\epsilon+
e\Big(6R\sqrt{2n^{-1}\log(e/\delta)}+12\sqrt{6}\big(L/\rho+1\big)G\epsilon\log(e/\delta)\lceil\log_2n\rceil\Big).
\]
This proves the stated bound in Part (d).
\subsection{Proof of Part (e)}
Part (e) is standard in the literature~\citep{bousquet2020sharper}.
%The proof is complete.

\section{Proof of Theorem \ref{lem:stab-gda}\label{sec:proof-stab-gda}}
In this section, we present the proof of Theorem \ref{lem:stab-gda} on the argument stability of SGDA.
\subsection{Approximate Nonexpansiveness of Gradient Map}
To prove stability bounds, we need to study the expansiveness of the gradient map
\[
G_{f,\eta}:\begin{pmatrix}
  \bw \\
  \bv
\end{pmatrix}\mapsto \begin{pmatrix}
           \bw-\eta\nabla_{\bw}f(\bw,\bv) \\
           \bv+\eta\nabla_{\bv}f(\bw,\bv)
         \end{pmatrix}
\]
associated with a (strongly) convex-concave $f$. The following lemma shows that $G_{f,\eta}$ is approximately nonexpansive in both the Lipschitz continuous case and the smooth case. It also shows that $G_{f,\eta}$ is nonexpansive if $f$ is SC-SC and the step size is small. Part (b) can be found in \citet{farnia2020train}.
\begin{lemma}\label{lem:nonexpansive}
  Let $f$ be $\rho$-SC-SC with $\rho\geq0$.
  \begin{enumerate}[label=(\alph*)]
    \item If Assumption \ref{ass:lipschitz} holds, then
    \[
    \bigg\|\begin{pmatrix}
           \bw-\eta\nabla_{\bw}f(\bw,\bv) \\
           \bv+\eta\nabla_{\bv}f(\bw,\bv)
         \end{pmatrix}-\begin{pmatrix}
                         \bw'-\eta\nabla_{\bw}f(\bw',\bv') \\
                         \bv'+\eta\nabla_{\bv}f(\bw',\bv')
                       \end{pmatrix}\bigg\|_2^2\leq (1-2\rho\eta)\bigg\|\begin{pmatrix}
                                                           \bw-\bw' \\
                                                           \bv-\bv'
                                                         \end{pmatrix}\bigg\|_2^2+8G^2\eta^2.
  \]

  \item  If Assumption \ref{ass:smooth} holds, then
     \[
    \bigg\|\begin{pmatrix}
           \bw-\eta\nabla_{\bw}f(\bw,\bv) \\
           \bv+\eta\nabla_{\bv}f(\bw,\bv)
         \end{pmatrix}-\begin{pmatrix}
                         \bw'-\eta\nabla_{\bw}f(\bw',\bv') \\
                         \bv'+\eta\nabla_{\bv}f(\bw',\bv')
                       \end{pmatrix}\bigg\|_2^2\leq (1-2\rho\eta+L^2\eta^2)\bigg\|\begin{pmatrix}
                                                           \bw-\bw' \\
                                                           \bv-\bv'
                                                         \end{pmatrix}\bigg\|_2^2.
  \]
  \end{enumerate}
\end{lemma}

%    \item If Assumption \ref{ass:smooth} holds, then
%    \[
%    \bigg\|\begin{pmatrix}
%           \bw-\eta\nabla_{\bw}f(\bw,\bv) \\
%           \bv+\eta\nabla_{\bv}f(\bw,\bv)
%         \end{pmatrix}-\begin{pmatrix}
%                         \bw'-\eta\nabla_{\bw}f(\bw',\bv') \\
%                         \bv'+\eta\nabla_{\bv}f(\bw',\bv')
%                       \end{pmatrix}\bigg\|_2^2\leq (1+L\eta^2)\bigg\|\begin{pmatrix}
%                                                           \bw-\bw' \\
%                                                           \bv-\bv'
%                                                         \end{pmatrix}\bigg\|_2^2.
%  \]

To prove Lemma \ref{lem:nonexpansive} we require the following standard lemma~\citep{rockafellar1976monotone}. % We give the proof for completeness.
\begin{lemma}\label{lem:non-expansive-key}
  Let $f$ be a $\rho$-SC-SC function, $\rho\geq0$. Then
  \begin{equation}\label{non-expansive-key}
  \bigg\langle\begin{pmatrix}
    \bw-\bw' \\
    \bv-\bv'
  \end{pmatrix},\begin{pmatrix}
                  \nabla_{\bw}f(\bw,\bv)-\nabla_{\bw}f(\bw',\bv') \\
                  \nabla_{\bv}f(\bw',\bv')-\nabla_{\bv}f(\bw,\bv)
                \end{pmatrix}\bigg\rangle \geq \rho\bigg\|\begin{pmatrix}
                                                           \bw-\bw' \\
                                                           \bv-\bv'
                                                         \end{pmatrix}\bigg\|_2^2.
  \end{equation}
\end{lemma}
%\begin{comment}
%\begin{proof}%[Proof of Lemma \ref{lem:nonexpansive}]
%According to the convexity of $f(\cdot,\bv)$ for any $\bv$, we know
%\[
%\langle\bw-\bw',\nabla_{\bw}f(\bw;\bv)\rangle \geq f(\bw,\bv)-f(\bw',\bv)
%\]
%\and
%\[
%\langle\bw'-\bw,\nabla_{\bw}f(\bw',\bv')\rangle\geq f(\bw',\bv')-f(\bw,\bv').
%\]
%We can combine the above two inequalities together and derive
%\begin{equation}\label{1}
%  \big\langle\bw-\bw',\nabla_{\bw}f(\bw,\bv)-\nabla_{\bw}f(\bw',\bv')\big\rangle\geq f(\bw,\bv)-f(\bw',\bv)+f(\bw',\bv')-f(\bw,\bv')
%\end{equation}
%According to the concavity of $f(\bw,\cdot)$ for any $\bw$, we know
%\[
%\langle\bv'-\bv,\nabla_{\bv}f(\bw,\bv)\rangle\geq f(\bw,\bv')-f(\bw,\bv)
%\]
%and
%\[
%\langle\bv-\bv',\nabla_{\bv}f(\bw',\bv')\rangle\geq f(\bw',\bv)-f(\bw',\bv').
%\]
%We can combine the above two inequalities together and derive
%\begin{equation}\label{2}
%  \langle\bv'-\bv,\nabla_{\bv}f(\bw,\bv)-\nabla_{\bv}f(\bw',\bv')\rangle \geq f(\bw,\bv')-f(\bw,\bv) + f(\bw',\bv)-f(\bw',\bv').
%\end{equation}
%Summing \eqref{1} and \eqref{2} gives the stated inequality. The proof is complete.
%\end{proof}
%\end{comment}
\begin{proof}[Proof of Lemma \ref{lem:nonexpansive}]
It is clear that
\begin{multline*}
  A:=\bigg\|\begin{pmatrix}
           \bw-\eta\nabla_{\bw}f(\bw,\bv) \\
           \bv+\eta\nabla_{\bv}f(\bw,\bv)
         \end{pmatrix}-\begin{pmatrix}
                         \bw'-\eta\nabla_{\bw}f(\bw',\bv') \\
                         \bv'+\eta\nabla_{\bv}f(\bw',\bv')
                       \end{pmatrix}\bigg\|_2^2 = \bigg\|\begin{pmatrix}
                                                           \bw-\bw' \\
                                                           \bv-\bv'
                                                         \end{pmatrix}\bigg\|_2^2 \\
  + \eta^2\bigg\|\begin{pmatrix}
  \nabla_{\bw}f(\bw',\bv') - \nabla_{\bw}f(\bw,\bv)\\
  \nabla_{\bv}f(\bw,\bv) - \nabla_{\bv}f(\bw',\bv')
  \end{pmatrix}\bigg\|_2^2-2\eta\bigg\langle\begin{pmatrix}
                      \bw-\bw' \\
                      \bv-\bv'
                    \end{pmatrix},\begin{pmatrix}
                                    \nabla_{\bw}f(\bw,\bv)-\nabla_{\bw}f(\bw',\bv') \\
                                    \nabla_{\bv}f(\bw',\bv')-\nabla_{\bv}f(\bw,\bv)
                                  \end{pmatrix}\bigg\rangle.
\end{multline*}
Plugging \eqref{non-expansive-key} to the above inequality, we derive
\[
  A\leq (1-2\rho\eta)\bigg\|\begin{pmatrix}
  \bw-\bw' \\
  \bv-\bv'
  \end{pmatrix}\bigg\|_2^2
  + \eta^2\bigg\|\begin{pmatrix}
  \nabla_{\bw}f(\bw',\bv') - \nabla_{\bw}f(\bw,\bv)\\
  \nabla_{\bv}f(\bw,\bv) - \nabla_{\bv}f(\bw',\bv')
  \end{pmatrix}\bigg\|_2^2.
\]
We can combine the above inequality with the Lipschitz continuity to derive Part (a). We refer the interested readers to \citet{farnia2020train} for the proof of Part (b).
\end{proof}

%\begin{align*}
% \ebb_{A}\left[\left\|\begin{pmatrix}
%           \bw_{t+1}-\bw_{t+1}' \\
%           \bv_{t+1}-\bv_{t+1}'
%         \end{pmatrix}\right\|_2^2\right]&\leq (1-\rho\eta_t)\ebb_A\left[\left\|\begin{pmatrix}
%           \bw_{t}-\bw_{t}' \\
%           \bv_{t}-\bv_{t}'
%         \end{pmatrix}\right\|_2^2\right]+(\rho\eta_t)8\eta_t\rho^{-1}G^2\Big(1+\frac{1}{n(n-2)\rho\eta_t}\Big)\\
%         & \leq \max\left\{\ebb_A\left[\left\|\begin{pmatrix}
%           \bw_{t}-\bw_{t}' \\
%           \bv_{t}-\bv_{t}'
%         \end{pmatrix}\right\|_2^2\right],8\eta_t\rho^{-1}G^2\Big(1+\frac{1}{n(n-2)\rho\eta_t}\Big)\right\}.
%\end{align*}
%We can apply this inequality recursively and get
%\[
% \ebb_{A}\left[\left\|\begin{pmatrix}
%           \bw_{t+1}-\bw_{t+1}' \\
%           \bv_{t+1}-\bv_{t+1}'
%         \end{pmatrix}\right\|_2^2\right]\leq 8\rho^{-1}G^2\max_{j=1,\ldots,t}\Big(\eta_j+\frac{1}{n(n-2)\rho}\Big).
%\]
We now prove Theorem \ref{lem:stab-gda}.
Let $S=\{z_1,\ldots,z_n\}$ and $S'=\{z_1,\ldots,z_{n-1},z_n'\}$. Let $\{\bw_t,\bv_t\}$ and $\{\bw_t',\bv_t'\}$ be the sequence produced by \eqref{gda} w.r.t. $S$ and $S'$, respectively.
\subsection{Proof of Part (a)}
We first prove Part (a).
Note that the projection step is nonexpansive.
We consider two cases at the $t$-th iteration. If $i_t\neq n$, then it follows from Part (a) of Lemma \ref{lem:nonexpansive} that
\begin{align}
  \left\|\begin{pmatrix}
           \bw_{t+1}-\bw_{t+1}' \\
           \bv_{t+1}-\bv_{t+1}'
         \end{pmatrix}\right\|_2^2
         &\leq\left\|\begin{pmatrix}
                   \bw_t-\eta_t\nabla_{\bw}f(\bw_t,\bv_t;z_{i_t})-\bw_t'+\eta_t\nabla_{\bw}f(\bw_t',\bv'_t;z_{i_t}) \\
                   \bv_t+\eta_t\nabla_{\bv}f(\bw_t,\bv_t;z_{i_t})-\bv_t'-\eta_t\nabla_{\bv}f(\bw_t',\bv'_t;z_{i_t})
                 \end{pmatrix}\right\|_2^2\notag \\
         & \leq \left\|\begin{pmatrix}
           \bw_{t}-\bw_{t}' \\
           \bv_{t}-\bv_{t}'
         \end{pmatrix}\right\|_2^2+8G^2\eta_t^2.\label{stab-gda-1}
\end{align}
If $i_t=n$, then it follows from the elementary inequality $(a+b)^2\leq(1+p)a^2+(1+1/p)b^2$ ($p>0$) that
\begin{align}
  \left\|\begin{pmatrix}
           \bw_{t+1}-\bw_{t+1}' \\
           \bv_{t+1}-\bv_{t+1}'
         \end{pmatrix}\right\|_2^2
         &\leq\left\|\begin{pmatrix}
                   \bw_t-\eta_t\nabla_{\bw}f(\bw_t,\bv_t;z_{n})-\bw_t'+\eta_t\nabla_{\bw}f(\bw_t',\bv'_t;z'_{n}) \\
                   \bv_t+\eta_t\nabla_{\bv}f(\bw_t,\bv_t;z_{n})-\bv_t'-\eta_t\nabla_{\bv}f(\bw_t',\bv'_t;z'_{n})
                 \end{pmatrix}\right\|_2^2\notag\\
         & \leq (1+p)\left\|\begin{pmatrix}
           \bw_{t}-\bw_{t}' \\
           \bv_{t}-\bv_{t}'
         \end{pmatrix}\right\|_2^2
        +(1+1/p)\eta_t^2\left\|\begin{pmatrix}
                                 \nabla_{\bw}f(\bw_t,\bv_t;z_n)-\nabla_{\bw}f(\bw_t',\bv_t';z'_n) \\
                                 \nabla_{\bv}f(\bw_t,\bv_t;z_n)-\nabla_{\bv}f(\bw_t',\bv_t';z'_n)
                               \end{pmatrix}\right\|_2^2.\label{stab-gda-2}
\end{align}
Note that the event $i_t\neq n$ happens with probability $1-1/n$ and the event $i_t=n$ happens with probability $1/n$. Therefore,
we know
\begin{align*}
 \ebb_{i_t}\left[\left\|\begin{pmatrix}
           \bw_{t+1}-\bw_{t+1}' \\
           \bv_{t+1}-\bv_{t+1}'
         \end{pmatrix}\right\|_2^2\right] & \leq \frac{n-1}{n}\left(\left\|\begin{pmatrix}
           \bw_{t}-\bw_{t}' \\
           \bv_{t}-\bv_{t}'
         \end{pmatrix}\right\|_2^2+8G^2\eta_t^2\right) + \frac{1+p}{n}\left\|\begin{pmatrix}
           \bw_{t}-\bw_{t}' \\
           \bv_{t}-\bv_{t}'
         \end{pmatrix}\right\|_2^2+\frac{8(1+1/p)}{n}\eta_t^2G^2 \\
         & = (1+p/n)\left\|\begin{pmatrix}
           \bw_{t}-\bw_{t}' \\
           \bv_{t}-\bv_{t}'
         \end{pmatrix}\right\|_2^2+8\eta_t^2G^2(1+1/(np)).
\end{align*}
Applying this inequality recursively implies that
\begin{align*}
  \ebb_A\left[\left\|\begin{pmatrix}
           \bw_{t+1}-\bw_{t+1}' \\
           \bv_{t+1}-\bv_{t+1}'
         \end{pmatrix}\right\|_2^2\right] & \leq 8\eta^2G^2\big(1+1/(np)\big)\sum_{k=1}^{t}\Big(1+\frac{p}{n}\Big)^{t-k}
   = 8\eta^2G^2\Big(1+\frac{1}{np}\Big)\frac{n}{p}\Big(\Big(1+\frac{p}{n}\Big)^t-1\Big)\\
  & = 8\eta^2G^2\Big(\frac{n}{p}+\frac{1}{p^2}\Big)\Big(\Big(1+\frac{p}{n}\Big)^t-1\Big).
\end{align*}
By taking $p=n/t$ in the above inequality and using $(1+1/t)^t\leq e$, we get
\[
\ebb_A\left[\left\|\begin{pmatrix}
           \bw_{t+1}-\bw_{t+1}' \\
           \bv_{t+1}-\bv_{t+1}'
         \end{pmatrix}\right\|_2^2\right]\leq 16\eta^2G^2\Big(t+\frac{t^2}{n^2}\Big).
\]
The stated bound then follows by Jensen's inequality.

\subsection{Proof of Part (b)}
We now prove Part (b). Analogous to \eqref{stab-gda-1}, we can use Part (b) of  Lemma \ref{lem:nonexpansive} to derive
\[
  \left\|\begin{pmatrix}
           \bw_{t+1}-\bw_{t+1}' \\
           \bv_{t+1}-\bv_{t+1}'
         \end{pmatrix}\right\|_2^2\leq
         (1+L^2\eta_t^2) \left\|\begin{pmatrix}
           \bw_{t}-\bw_{t}' \\
           \bv_{t}-\bv_{t}'
         \end{pmatrix}\right\|_2^2
\]
in the case $i_{t}\neq n$.
We can combine the above inequality and \eqref{stab-gda-2} to derive
\begin{align*}
  \ebb_{i_t}\left[\left\|\begin{pmatrix}
           \bw_{t+1}-\bw_{t+1}' \\
           \bv_{t+1}-\bv_{t+1}'
         \end{pmatrix}\right\|_2^2\right] & \leq \frac{(n-1)(1+L^2\eta_t^2)}{n} \left\|\begin{pmatrix}
           \bw_{t}-\bw_{t}' \\
           \bv_{t}-\bv_{t}'
         \end{pmatrix}\right\|_2^2 + \frac{1+p}{n}\left\|\begin{pmatrix}
           \bw_{t}-\bw_{t}' \\
           \bv_{t}-\bv_{t}'
         \end{pmatrix}\right\|_2^2+\frac{8(1+1/p)}{n}\eta_t^2G^2 \\
         & \leq \Big(1+L^2\eta_t^2+p/n\Big)\left\|\begin{pmatrix}
           \bw_{t}-\bw_{t}' \\
           \bv_{t}-\bv_{t}'
         \end{pmatrix}\right\|_2^2+\frac{8(1+1/p)}{n}\eta_t^2G^2.
\end{align*}
Applying this inequality recursively, we derive
\[
\ebb_A\left[\left\|\begin{pmatrix}
           \bw_{t+1}-\bw_{t+1}' \\
           \bv_{t+1}-\bv_{t+1}'
         \end{pmatrix}\right\|_2^2\right]\leq
         \frac{8G^2(1+1/p)}{n}\sum_{k=1}^{t}\eta_k^2\prod_{j=k+1}^{t}\Big(1+L^2\eta_{j}^2+p/n\Big).
\]
By the elementary inequality $1+a\leq\exp(a)$, we further derive
\begin{align*}
\ebb_A\left[\left\|\begin{pmatrix}
           \bw_{t+1}-\bw_{t+1}' \\
           \bv_{t+1}-\bv_{t+1}'
         \end{pmatrix}\right\|_2^2\right] & \leq \frac{8G^2(1+1/p)}{n}\sum_{k=1}^{t}\eta_k^2\prod_{j=k+1}^{t}\exp\Big(L^2\eta_{j}^2+p/n\Big)\\
         & = \frac{8G^2(1+1/p)}{n}\sum_{k=1}^{t}\eta_k^2\exp\Big(L^2\sum_{j=k+1}^{t}\eta_{j}^2+p(t-k)/n\Big)\\
         & \leq \frac{8G^2(1+1/p)}{n}\exp\Big(L^2\sum_{j=1}^{t}\eta_{j}^2+pt/n\Big)\sum_{k=1}^{t}\eta_k^2.
\end{align*}
By taking $p=n/t$ we get
\[
\ebb_A\left[\left\|\begin{pmatrix}
           \bw_{t+1}-\bw_{t+1}' \\
           \bv_{t+1}-\bv_{t+1}'
         \end{pmatrix}\right\|_2^2\right]
         \leq
         \frac{8eG^2(1+t/n)}{n}\exp\Big(L^2\sum_{j=1}^{t}\eta_{j}^2\Big)\sum_{k=1}^{t}\eta_k^2.
\]
The stated result then follows from the Jensen's inequality.

\subsection{Proof of Part (c)}
To prove stability bounds with high probability, we first introduce a concentration inequality~\citep{chernoff1952measure}.

\begin{lemma}[Chernoff's Bound\label{lem:chernoff}]
  Let $X_1,\ldots,X_t$ be independent random variables taking values in $\{0,1\}$. Let $X=\sum_{j=1}^{t}X_j$ and $\mu=\ebb[X]$. Then for any $\tilde{\delta}>0$ with probability at least $1-\exp\big(-\mu\tilde{\delta}^2/(2+\tilde{\delta})\big)$ we have $X\leq (1+\tilde{\delta})\mu$. Furthermore, for any $\delta\in(0,1)$ with probability at least $1-\delta$ we have
  \[
  X\leq \mu+\log(1/\delta)+\sqrt{2\mu\log(1/\delta)}.
  %X\leq \Big(\frac{\log(1/\delta)}{\mu}+\frac{\sqrt{2\log(1/\delta)}}{\sqrt{\mu}}\Big)\mu.
  \]
\end{lemma}

We now prove Part (c).
According to the analysis in Part (a), we know
\[
  \left\|\begin{pmatrix}
           \bw_{t+1}-\bw_{t+1}' \\
           \bv_{t+1}-\bv_{t+1}'
         \end{pmatrix}\right\|_2^2
         \leq  \left(\left\|\begin{pmatrix}
           \bw_{t}-\bw_{t}' \\
           \bv_{t}-\bv_{t}'
         \end{pmatrix}\right\|_2^2+8G^2\eta_t^2\right)\ibb_{[i_t\neq n]}+
         \left((1+p)\left\|\begin{pmatrix}
           \bw_{t}-\bw_{t}' \\
           \bv_{t}-\bv_{t}'
         \end{pmatrix}\right\|_2^2
        +8(1+1/p)\eta_t^2G^2\right)\ibb_{[i_t= n]}.
\]
It then follows that
\begin{equation}
\left\|\begin{pmatrix}
           \bw_{t+1}-\bw_{t+1}' \\
           \bv_{t+1}-\bv_{t+1}'
         \end{pmatrix}\right\|_2^2
         \leq\big(1+p\ibb_{[i_t=n]}\big)\left\|\begin{pmatrix}
           \bw_{t}-\bw_{t}' \\
           \bv_{t}-\bv_{t}'
         \end{pmatrix}\right\|_2^2+8G^2\eta_t^2\big(1+\ibb_{[i_t= n]}/p\big).
\end{equation}
Applying this inequality recursively gives
\begin{align*}
\left\|\begin{pmatrix}
           \bw_{t+1}-\bw_{t+1}' \\
           \bv_{t+1}-\bv_{t+1}'
         \end{pmatrix}\right\|_2^2 & \leq 8G^2\eta^2\sum_{k=1}^{t}\big(1+\ibb_{[i_k= n]}/p\big)\prod_{j=k+1}^{t}\big(1+p\ibb_{[i_j=n]}\big)\\
         & = 8G^2\eta^2\sum_{k=1}^{t}\big(1+\ibb_{[i_k= n]}/p\big)\prod_{j=k+1}^{t}(1+p)^{\ibb_{[i_j=n]}}\\
         & \leq 8G^2\eta^2(1+p)^{\sum_{j=1}^{t}\ibb_{[i_j=n]}}\big(t+\sum_{k=1}^{t}\ibb_{[i_k= n]}/p\big).
\end{align*}
Applying Lemma \ref{lem:chernoff} with $X_j=\ibb_{[i_j=n]}$ and $\mu=t/n$ (note $\ebb_A[X_j]=1/n$), with probability $1-\delta$ there holds
\begin{equation}\label{chernoff-1}
\sum_{j=1}^{t}\ibb_{[i_j=n]}\leq t/n+\log(1/\delta)+\sqrt{2tn^{-1}\log(1/\delta)}. %\frac{t}{n}\big(1+\sqrt{3nt^{-1}\log(1/\delta)}\big).
\end{equation}
The following inequality then holds with probability at least $1-\delta$
\[
\left\|\begin{pmatrix}
           \bw_{t+1}-\bw_{t+1}' \\
           \bv_{t+1}-\bv_{t+1}'
         \end{pmatrix}\right\|_2^2\leq 8G^2\eta^2(1+p)^{t/n+\log(1/\delta)+\sqrt{2tn^{-1}\log(1/\delta)}}\Big(t+t/(pn)+p^{-1}\log(1/\delta)+p^{-1}\sqrt{2tn^{-1}\log(1/\delta)}\Big).
\]
We can choose $p=\frac{1}{t/n+\log(1/\delta)+\sqrt{2tn^{-1}\log(1/\delta)}}$ (note $(1+x)^{1/x}\leq e$) and derive the following inequality with probability at least $1-\delta$
\[
\left\|\begin{pmatrix}
           \bw_{t+1}-\bw_{t+1}' \\
           \bv_{t+1}-\bv_{t+1}'
         \end{pmatrix}\right\|_2^2\leq 8eG^2\eta^2\Big(t+\big(t/n+\log(1/\delta)
         +\sqrt{2tn^{-1}\log(1/\delta)}\big)^2\Big).
\]
This finishes the proof of Part (c).
\subsection{Proof of Part (d)}
We now turn to Part (d). Under the smoothness assumption, the analysis in Part (b) implies
\[
  \left\|\begin{pmatrix}
           \bw_{t+1}-\bw_{t+1}' \\
           \bv_{t+1}-\bv_{t+1}'
         \end{pmatrix}\right\|_2^2
         \leq  (1+L^2\eta_t^2)\left\|\begin{pmatrix}
           \bw_{t}-\bw_{t}' \\
           \bv_{t}-\bv_{t}'
         \end{pmatrix}\right\|_2^2\ibb_{[i_t\neq n]}+
         \left((1+p)\left\|\begin{pmatrix}
           \bw_{t}-\bw_{t}' \\
           \bv_{t}-\bv_{t}'
         \end{pmatrix}\right\|_2^2
        +8(1+1/p)\eta_t^2G^2\right)\ibb_{[i_t= n]}.
\]
It then follows that
\[
  \left\|\begin{pmatrix}
           \bw_{t+1}-\bw_{t+1}' \\
           \bv_{t+1}-\bv_{t+1}'
         \end{pmatrix}\right\|_2^2
         \leq  \big(1+L^2\eta_t^2+p\ibb_{[i_t= n]}\big)\left\|\begin{pmatrix}
           \bw_{t}-\bw_{t}' \\
           \bv_{t}-\bv_{t}'
         \end{pmatrix}\right\|_2^2
        +8(1+1/p)\eta_t^2G^2\ibb_{[i_t= n]}.
\]
We can apply the above inequality recursively and derive
\begin{align*}
\left\|\begin{pmatrix}
           \bw_{t+1}-\bw_{t+1}' \\
           \bv_{t+1}-\bv_{t+1}'
         \end{pmatrix}\right\|_2^2 &\leq 8(1+1/p)G^2\sum_{k=1}^{t}\eta_k^2\ibb_{[i_k= n]}\prod_{j=k+1}^{t}\big(1+L^2\eta_j^2+p\ibb_{[i_j= n]}\big)\\
         & \leq 8(1+1/p)G^2\eta^2\sum_{k=1}^{t}\ibb_{[i_k= n]}\prod_{j=k+1}^{t}\big(1+L^2\eta_j^2\big)\prod_{j=k+1}^{t}\big(1+p\ibb_{[i_j= n]}\big)\\
         & = 8(1+1/p)G^2\eta^2\sum_{k=1}^{t}\ibb_{[i_k= n]}\prod_{j=k+1}^{t}\big(1+L^2\eta_j^2\big)\prod_{j=k+1}^{t}\big(1+p\big)^{\ibb_{[i_j= n]}}\\
         & \leq 8(1+1/p)G^2\eta^2\prod_{j=1}^{t}\big(1+L^2\eta_j^2\big)\prod_{j=1}^{t}\big(1+p\big)^{\ibb_{[i_j= n]}}\sum_{k=1}^{t}\ibb_{[i_k= n]}.
\end{align*}
It then follows from the elementary inequality $1+x\leq e^x$ that
\[
\left\|\begin{pmatrix}
           \bw_{t+1}-\bw_{t+1}' \\
           \bv_{t+1}-\bv_{t+1}'
         \end{pmatrix}\right\|_2^2\leq 8(1+1/p)G^2\eta^2\exp\Big(L^2\sum_{j=1}^{t}\eta_j^2\Big)\big(1+p\big)^{\sum_{j=1}^{t}\ibb_{[i_j= n]}}\sum_{k=1}^{t}\ibb_{[i_k= n]}
\]
According to \eqref{chernoff-1}, we get the following inequality with probability at least $1-\delta$
\[
\left\|\begin{pmatrix}
           \bw_{t+1}-\bw_{t+1}' \\
           \bv_{t+1}-\bv_{t+1}'
         \end{pmatrix}\right\|_2^2\leq 8(1+1/p)G^2\eta^2\exp\big(L^2t\eta^2\big)\big(1+p\big)^{t/n+\log(1/\delta)+\sqrt{2tn^{-1}\log(1/\delta)}}\big(t/n+\log(1/\delta)+\sqrt{2tn^{-1}\log(1/\delta)}\big).
\]
We can choose $p=\frac{1}{t/n+\log(1/\delta)+\sqrt{2tn^{-1}\log(1/\delta)}}$ and derive the following inequality with probability at least $1-\delta$
\[
\left\|\begin{pmatrix}
           \bw_{t+1}-\bw_{t+1}' \\
           \bv_{t+1}-\bv_{t+1}'
         \end{pmatrix}\right\|_2^2\leq 8eG^2\eta^2\exp\big(L^2t\eta^2\big)
         \Big(1+t/n+\log(1/\delta)+\sqrt{2tn^{-1}\log(1/\delta)}\Big)^2.
\]
The stated bound then follows.
\subsection{Proof of Part (e)}

If $i_t\neq n$, we can analyze analogously to \eqref{stab-gda-1} excepting using the strong convexity, and show
\[
\left\|\begin{pmatrix}
           \bw_{t+1}-\bw_{t+1}' \\
           \bv_{t+1}-\bv_{t+1}'
         \end{pmatrix}\right\|_2^2
         \leq (1-2\rho\eta_t)\left\|\begin{pmatrix}
           \bw_{t}-\bw_{t}' \\
           \bv_{t}-\bv_{t}'
         \end{pmatrix}\right\|_2^2+8G^2\eta_t^2.
\]
If $i_t=n$, then \eqref{stab-gda-2} holds. We can combine the above two cases and derive
\begin{align*}
  & \ebb_{i_t}\left[\left\|\begin{pmatrix}
           \bw_{t+1}-\bw_{t+1}' \\
           \bv_{t+1}-\bv_{t+1}'
         \end{pmatrix}\right\|_2^2\right] \\
         & \leq \frac{n-1}{n}\left((1-2\rho\eta_t)\left\|\begin{pmatrix}
           \bw_{t}-\bw_{t}' \\
           \bv_{t}-\bv_{t}'
         \end{pmatrix}\right\|_2^2+8G^2\eta_t^2\right) + \frac{1+p}{n}\left\|\begin{pmatrix}
           \bw_{t}-\bw_{t}' \\
           \bv_{t}-\bv_{t}'
         \end{pmatrix}\right\|_2^2+\frac{8(1+1/p)}{n}\eta_t^2G^2 \\
         & = (1-2\rho\eta_t+(2\rho\eta_t+p)/n)\left\|\begin{pmatrix}
           \bw_{t}-\bw_{t}' \\
           \bv_{t}-\bv_{t}'
         \end{pmatrix}\right\|_2^2+8\eta_t^2G^2(1+1/(np)).
\end{align*}
We can choose $p=\rho\eta_t(n-2)$ to derive
\[
 \ebb_{i_t}\left[\left\|\begin{pmatrix}
           \bw_{t+1}-\bw_{t+1}' \\
           \bv_{t+1}-\bv_{t+1}'
         \end{pmatrix}\right\|_2^2\right]\leq (1-\rho\eta_t)\left\|\begin{pmatrix}
           \bw_{t}-\bw_{t}' \\
           \bv_{t}-\bv_{t}'
         \end{pmatrix}\right\|_2^2+8\eta_t^2G^2\Big(1+\frac{1}{n(n-2)\rho\eta_t}\Big).
\]
It then follows that
\[
 \ebb_{A}\left[\left\|\begin{pmatrix}
           \bw_{t+1}-\bw_{t+1}' \\
           \bv_{t+1}-\bv_{t+1}'
         \end{pmatrix}\right\|_2^2\right]\leq 8G^2\sum_{j=1}^{t}\eta_j\Big(\eta_j+\frac{1}{n(n-2)\rho}\Big)\prod_{k=j+1}^{t}(1-\rho\eta_k).
\]
For $\eta_t=1/(\rho t)$, it follows from the identity $\prod_{k=j+1}^{t}(1-1/k)=j/t$ that
\[
 \ebb_{A}\left[\left\|\begin{pmatrix}
           \bw_{t+1}-\bw_{t+1}' \\
           \bv_{t+1}-\bv_{t+1}'
         \end{pmatrix}\right\|_2^2\right]\leq \frac{8G^2}{t\rho}\sum_{j=1}^{t}\Big((\rho j)^{-1}+\frac{1}{n(n-2)\rho}\Big)\leq \frac{8G^2}{\rho^2}\Big(\frac{\log(et)}{t}+\frac{1}{n(n-2)}\Big).
\]
%It then follows that
%\[\ebb_{A}\left[\left\|\begin{pmatrix}
%           \bw_{t+1}-\bw_{t+1}' \\
%           \bv_{t+1}-\bv_{t+1}'
%         \end{pmatrix}\right\|_2^2\right]\leq \frac{8G^2}{\rho^2}\Big(\frac{\log(et)}{t}+\frac{1}{n(n-2)}\Big).
%\]
The stated result then follows from the Jensen's inequality.

\section{Optimization Error Bounds: Convex-Concave Case}

In this section, we present optimization error bounds for SGDA, which are standard in the literature~\citep{nemirovski2009robust,nedic2009subgradient}. We give both bounds in expectation and bounds with high probability. The high-probability analysis requires to use concentration inequalities for martingales. Lemma \ref{lem:martingale} is an Azuma-Hoeffding inequality for real-valued martingale difference sequence~\citep{hoeffding1963probability}, while Lemma \ref{lem:pinelis} is a Bernstein-type inequality for martingale difference sequences in a Hilbert space~\citep{tarres2014online}.
\begin{lemma}\label{lem:martingale}
%Let $z_1,\ldots,z_n$ be a sequence of random variables such that $z_k$ may depend on the previous random variables $z_1,\ldots,z_{k-1}$ for all $k=1,\ldots,n$. Consider a sequence of functionals $\xi_k(z_1,\ldots,z_k),k=1,\ldots,n$.
Let $\{\xi_k:k\in\nbb\}$ be a martingale difference sequence taking values in $\rbb$, i.e., $\ebb[\xi_k|\xi_1,\ldots,\xi_{k-1}]=0$.
 Assume that $|\xi_k-\ebb_{z_k}[\xi_k]|\leq b_k$ for each $k$. For $\delta\in(0,1)$, with probability at least $1-\delta$ we have
    \begin{equation}\label{hoeffding}
      \sum_{k=1}^{n}\xi_k\leq \Big(2\sum_{k=1}^{n}b_k^2\log\frac{1}{\delta}\Big)^{\frac{1}{2}}.
    \end{equation}
\end{lemma}
\begin{lemma}\label{lem:pinelis}
  Let $\{\xi_k:k\in\nbb\}$ be a martingale difference sequence in a Hilbert space with the norm $\|\cdot\|_2$. Suppose that almost surely $\|\xi_k\|\leq B$ and $\sum_{k=1}^{t}\ebb[\|\xi_k\|^2|\xi_1,\ldots,\xi_{k-1}]\leq\sigma_t^2$ for $\sigma_t\geq0$. Then, for any $0<\delta<1$, the following inequality holds with probability at least $1-\delta$
  $$
    \max_{1\leq j\leq t}\Big\|\sum_{k=1}^{j}\xi_k\Big\|\leq 2\Big(\frac{B}{3}+\sigma_t\Big)\log\frac{2}{\delta}.
  $$
\end{lemma}

\begin{lemma}\label{lem:opt}
Let $\{\bw_t,\bv_t\}$ be the sequence produced by \eqref{gda} with $\eta_t=\eta$.
Let Assumption \ref{ass:lipschitz} hold and $F_S$ be convex-concave. Assume $\sup_{\bw\in\wcal}\|\bw\|_2\leq B_W$ and $\sup_{\bv\in\vcal}\|\bv\|_2\leq B_V$.
Then the following inequality holds
\begin{equation}\label{opt-a}
\ebb_A\big[\sup_{\bv\in\vcal}F_S(\bar{\bw}_T,\bv)-\inf_{\bw\in\wcal}F_S(\bw,\bar{\bv}_T)\big]\leq \eta G^2 + \frac{B_W^2+B_V^2}{2\eta T}+ \frac{G(B_W+B_V)}{\sqrt{T}},
\end{equation}
where $\big(\bar{\bw}_T,\bar{\bv}_T\big)$ is defined in \eqref{bar-w-v}.
Let $\delta\in(0,1)$. Then with probability at least $1-\delta$ we have
\begin{equation}\label{opt-b}
\sup_{\bv\in\vcal}F_S(\bar{\bw}_T,\bv)-\inf_{\bw\in\wcal}F_S(\bw,\bar{\bv}_T)\leq
\eta G^2+\frac{B_W^2+B_V^2}{2T\eta}+\frac{G\big(B_W+B_V\big)\big(9\log(6/\delta)+2)}{\sqrt{T}}.
%\eta G^2+\frac{B_W^2+B_V^2}{2T\eta}+\frac{12G(B_W+B_V)\log(6/\delta)}{\sqrt{T}}.
\end{equation}
\end{lemma}
\begin{proof}%[Proof of Lemma \ref{lem:opt}]
According to the non-expansiveness of projection and \eqref{gda}, we know
\begin{align*}
  & \|\bw_{t+1}-\bw\|_2^2 \leq \|\bw_t-\eta_t\nabla_{\bw} f(\bw_t,\bv_t;z_{i_t})-\bw\|_2^2\\
  & = \|\bw_t-\bw\|_2^2+\eta_t^2\|\nabla_{\bw} f(\bw_t,\bv_t;z_{i_t})\|_2^2 + 2\eta_t\langle\bw-\bw_t,\nabla_{\bw} f(\bw_t,\bv_t;z_{i_t})\rangle\\
  %& \leq  \|\bw_t-\bw\|_2^2+\eta_t^2G^2 + 2\eta_t\big(f(\bw,\bv_t;z_{i_t})- f(\bw_t,\bv_t;z_{i_t})\big).
  & \leq  \|\bw_t-\bw\|_2^2+\eta_t^2G^2+2\eta_t\langle\bw-\bw_t,\nabla_{\bw} F_S(\bw_t;\bv_t)\rangle
  + 2\eta_t\langle\bw-\bw_t,\nabla_{\bw} f(\bw_t,\bv_t;z_{i_t})-\nabla_{\bw} F_S(\bw_t,\bv_t)\rangle,
\end{align*}
where we have used Assumption \ref{ass:lipschitz}.
According to the convexity of $F_S(\cdot,\bv_t)$, we know
\begin{multline}\label{opt-001}
  2\eta_t\big(F_S(\bw_t,\bv_t)-F_S(\bw,\bv_t)\big) \leq \|\bw_t-\bw\|_2^2 - \|\bw_{t+1}-\bw\|_2^2
  +\\
  \eta_t^2G^2+2\eta_t\langle\bw-\bw_t,\nabla_{\bw} f(\bw_t,\bv_t;z_{i_t})-\nabla_{\bw} F_S(\bw_t,\bv_t)\rangle.
\end{multline}
Taking a summation of the above inequality from $t=1$ to $t=T$ ($\bw_1=0$), we derive
\begin{multline*}
2\eta\sum_{t=1}^{T}\big(F_S(\bw_t,\bv_t)-F_S(\bw,\bv_t)\big) \leq \|\bw\|_2^2 + T\eta^2G^2 \\
+ 2\eta\sum_{t=1}^{T}\langle\bw_t,\nabla_{\bw} F_S(\bw_t,\bv_t)-\nabla_{\bw} f(\bw_t,\bv_t;z_{i_t})\rangle+2\eta\sum_{t=1}^{T}\langle\bw,\nabla_{\bw} f(\bw_t,\bv_t;z_{i_t})-\nabla_{\bw} F_S(\bw_t,\bv_t)\rangle.%\label{opt-hp-1}
\end{multline*}
It then follows from the concavity of $F_S(\bw,\cdot)$ and Schwartz's inequality that
\begin{multline*}
2\eta\sum_{t=1}^{T}\big(F_S(\bw_t,\bv_t)-F_S(\bw,\bar{\bv}_T)\big) \leq B_W^2 + T\eta^2G^2 \\
+ 2\eta\sum_{t=1}^{T}\langle\bw_t,\nabla_{\bw} F_S(\bw_t,\bv_t)-\nabla_{\bw} f(\bw_t,\bv_t;z_{i_t})\rangle+2\eta B_W\Big\|\sum_{t=1}^{T}\big(\nabla_{\bw} f(\bw_t,\bv_t;z_{i_t})-\nabla_{\bw} F_S(\bw_t,\bv_t)\big)\Big\|_2.
\end{multline*}
Since the above inequality holds for all $\bw$, we further get
\begin{multline}
2\eta\sum_{t=1}^{T}\big(F_S(\bw_t,\bv_t)-\inf_{\bw}F_S(\bw,\bar{\bv}_T)\big) \leq B_W^2 + T\eta^2G^2 \\
+ 2\eta\sum_{t=1}^{T}\langle\bw_t,\nabla_{\bw} F_S(\bw_t,\bv_t)-\nabla_{\bw} f(\bw_t,\bv_t;z_{i_t})\rangle+2\eta B_W\Big\|\sum_{t=1}^{T}\big(\nabla_{\bw} f(\bw_t,\bv_t;z_{i_t})-\nabla_{\bw} F_S(\bw_t,\bv_t)\big)\Big\|_2.\label{opt-hp-1}
\end{multline}
Note
\begin{equation}\label{opt-1}
\ebb_{i_t}\big[\langle\bw_t,\nabla_{\bw} F_S(\bw_t,\bv_t)-\nabla_{\bw} f(\bw_t,\bv_t;z_{i_t})\rangle\big]=0.
\end{equation}
We can take an expectation over both sides of \eqref{opt-hp-1} and get
\[
\frac{1}{T}\sum_{t=1}^{T}\ebb_A\big[F_S(\bw_t,\bv_t)\big]-\ebb_A\big[\inf_{\bw}F_S(\bw,\bar{\bv}_T)\big]\leq \frac{B_W^2}{2\eta T} + \frac{\eta G^2}{2} +  \frac{B_W}{T}\ebb_A\Big[\Big\|\sum_{t=1}^{T}\big(\nabla_{\bw} f(\bw_t,\bv_t;z_{i_t})-\nabla_{\bw} F_S(\bw_t,\bv_t)\big)\Big\|_2\Big].
\]
According to Jensen's inequality and \eqref{opt-1}, we know
\begin{align*}
  \Big(\ebb_A\Big[\Big\|\sum_{t=1}^{T}\big(\nabla_{\bw} f(\bw_t,\bv_t;z_{i_t})-\nabla_{\bw} F_S(\bw_t,\bv_t)\big)\Big\|_2\Big]\Big)^2
  & \leq \ebb_A\Big[\Big\|\sum_{t=1}^{T}\big(\nabla_{\bw} f(\bw_t,\bv_t;z_{i_t})-\nabla_{\bw} F_S(\bw_t,\bv_t)\big)\Big\|_2^2\Big]\\
  & = \sum_{t=1}^{T}\ebb_A\Big[\Big\|\nabla_{\bw} f(\bw_t,\bv_t;z_{i_t})-\nabla_{\bw} F_S(\bw_t,\bv_t)\Big\|_2^2\Big] \leq TG^2.
\end{align*}
It then follows that
\begin{equation}\label{opt-2}
\frac{1}{T}\sum_{t=1}^{T}\ebb_A\big[F_S(\bw_t,\bv_t)\big]-\ebb_A\big[\inf_{\bw}F_S(\bw,\bar{\bv}_T)\big]\leq \frac{B_W^2}{2\eta T} + \frac{\eta G^2}{2} + \frac{B_WG}{\sqrt{T}}.
\end{equation}
%Taking a summation of the above inequality gives
%\[
%2\eta\sum_{t=1}^{T}\ebb_A\big[F_S(\bw_t,\bv_t)-F_S(\bw,\bv_t)\big]\leq \|\bw\|_2^2 + T\eta^2G^2.
%\]
%This together with the concavity of $f$ w.r.t. the second argument gives
%\begin{equation}\label{opt-1}
%\frac{1}{T}\sum_{t=1}^{T}\ebb_A\big[F_S(\bw_t,\bv_t)\big]-\ebb_A\big[F_S(\bw,\bar{\bv}_T)\big]\leq \frac{\|\bw\|_2^2}{2\eta T}+\frac{\eta G^2}{2}.
%\end{equation}
In a similar way, we can show  that
\begin{equation}\label{opt-3}
\ebb_A\big[\sup_{\bv}F_S(\bar{\bw}_T,\bv)\big]-\frac{1}{T}\sum_{t=1}^{T}\ebb_A\big[F_S(\bw_t,\bv_t)\big]\leq \frac{B_V^2}{2\eta T} + \frac{\eta G^2}{2} + \frac{B_VG}{\sqrt{T}}.
\end{equation}
The stated bound \eqref{opt-a} then follows from \eqref{opt-2} and \eqref{opt-3}.

We now turn to \eqref{opt-b}.
It is clear that $\big|\langle\bw_t,\nabla F_S(\bw_t,\bv_t)-\nabla f(\bw_t,\bv_t;z_{i_t})\rangle\big|\leq 2GB_W$, and therefore we can apply Lemma \ref{lem:martingale} to derive the following inequality with probability at least $1-\delta/6$ that
\begin{equation}\label{opt-hp-2}
\sum_{t=1}^{T}\langle\bw_t,\nabla_{\bw} F_S(\bw_t,\bv_t)-\nabla_{\bw} f(\bw_t,\bv_t;z_{i_t})\rangle
\leq 2GB_W\Big(2T\log(6/\delta)\Big)^{\frac{1}{2}}.
\end{equation}
For any $t\in\nbb$, define $\xi_t=\nabla_{\bw} f(\bw_t,\bv_t;z_{i_t})-\nabla_{\bw} F_S(\bw_t,\bv_t)$. Then it is clear that $\|\xi_t\|_2\leq 2G$ and \[\sum_{t=1}^{T}\ebb[\|\xi_t\|_2^2|\xi_1,\ldots,\xi_{t-1}]\leq 4TG^2.\] Therefore, we can apply Lemma \ref{lem:pinelis} to derive the following inequality with probability at least $1-\delta/3$
\[
\Big\|\sum_{t=1}^{T}\xi_t\Big\|_2\leq 2\big(\frac{2G}{3}+2G\sqrt{T}\big)\log(6/\delta).
\]
Then, the following inequality holds with probability at least $1-\delta/3$
\[\Big\|\sum_{t=1}^{T}\big(\nabla_{\bw} f(\bw_t,\bv_t;z_{i_t})-\nabla_{\bw} F_S(\bw_t,\bv_t)\big)\Big\|_2\leq 4G\big(1+\sqrt{T}\Big)\log(6/\delta).
\]
We can plug the above inequality and \eqref{opt-hp-2} back into \eqref{opt-hp-1}, and derive the following inequality with probability at least $1-\delta/2$
\[
  \frac{1}{T}\sum_{t=1}^{T}F_S(\bw_t,\bv_t)-\inf_{\bw}F_S(\bw,\bar{\bv}_T)\leq \frac{B_W^2}{2T\eta}+\frac{\eta G^2}{2}+
  \frac{2GB_W\sqrt{2\log(6/\delta)}}{\sqrt{T}}+\frac{8B_WG\log(6/\delta)}{\sqrt{T}}.
\]
In a similar way, we can get the following inequality with probability at least $1-\delta/2$ %uniformly for all $\bv\in\vcal$
\[
\sup_{\bv\in\vcal}F_S(\bar{\bw}_T,\bv)-\frac{1}{T}\sum_{t=1}^{T}F_S(\bw_t,\bv_t)\leq \frac{B_V^2}{2T\eta}+\frac{\eta G^2}{2}+\frac{9B_VG\log(6/\delta)+2B_VG}{\sqrt{T}}.
\]
Combining the above two inequalities together we get the stated inequality with probability at least $1-\delta$. % uniformly for all $\bw\in\wcal,\bv\in\vcal$
The proof is complete.
\end{proof}

The following lemma gives optimization error bounds for SC-SC problems.
\begin{lemma}\label{lem:opt-sc}
Let Assumption \ref{ass:lipschitz} hold, $t_0\geq0$ and $F_S(\cdot,\cdot)$ be $\rho$-SC-SC with $\rho>0$.
Let $\{\bw_t,\bv_t\}$ be the sequence produced by \eqref{gda} with $\eta_t=1/(\rho (t+t_0))$.
If $t_0=0$, then for $\big(\bar{\bw}_T,\bar{\bv}_T\big)$ defined in \eqref{bar-w-v} we have
\begin{equation}\label{opt-sc-a}
\ebb_A\big[\sup_{\bv\in\vcal}F_S(\bar{\bw}_T,\bv)-\inf_{\bw\in\wcal}F_S(\bw,\bar{\bv}_T)\big]\leq \frac{G^2\log(eT)}{\rho T}+\frac{(B_W+B_V)G}{\sqrt{T}}.
\end{equation}
If $\sup_{\bw\in\wcal}\|\bw\|_2\leq B_W$ and $\sup_{\bv\in\vcal}\|\bv\|_2\leq B_V$, then
\begin{equation}\label{opt-sc-b}
\triangle^w_S(\bar{\bw}_T,\bar{\bv}_T)\leq \frac{2\rho t_0(B_W^2+B_V^2)}{T}+\frac{G^2\log(eT)}{\rho T}.
\end{equation}
\end{lemma}
\begin{proof}
Analyzing analogously to \eqref{opt-001} but using the strong convexity of $\bw\mapsto F_S(\bw,\bv)$, we derive
\[
2\eta_t\big[F_S(\bw_t,\bv_t)-F_S(\bw,\bv_t)\big] \leq (1-\eta_t\rho)\|\bw_t-\bw\|_2^2 - \|\bw_{t+1}-\bw\|_2^2
  +\eta_t^2G^2+\xi_t(\bw),
\]
where $\xi_t(\bw)=2\eta_t\langle\bw-\bw_t,\nabla_{\bw} f(\bw_t,\bv_t;z_{i_t})-\nabla_{\bw} F_S(\bw_t,\bv_t)\rangle$. Since $\eta_t=1/(\rho (t+t_0))$, we further get
\[
\frac{2}{\rho (t+t_0)}\big[F_S(\bw_t,\bv_t)-F_S(\bw,\bv_t)\big] \leq (1-1/(t+t_0))\|\bw_t-\bw\|_2^2 - \|\bw_{t+1}-\bw\|_2^2
  +\frac{G^2}{\rho^2(t+t_0)^2}+\xi_t(\bw).
\]
Multiplying both sides by $t+t_0$ gives
\[
\frac{2}{\rho}\big[F_S(\bw_t,\bv_t)-F_S(\bw,\bv_t)\big] \leq (t+t_0-1)\|\bw_t-\bw\|_2^2- (t+t_0)\|\bw_{t+1}-\bw\|_2^2+(t+t_0)\xi_t(\bw)
  +\frac{G^2}{\rho^2(t+t_0)}.
\]
Taking a summation of the above inequality further gives
\[
\sum_{t=1}^{T}\big[F_S(\bw_t,\bv_t)-F_S(\bw,\bv_t)\big]\leq 2\rho t_0B_W^2+\frac{G^2\log(eT)}{2\rho}+\frac{\rho}{2}\sum_{t=1}^{T}(t+t_0)\xi_t(\bw),
\]
where we have used $\sum_{t=1}^{T}t^{-1}\leq\log(eT)$.
This together with the concavity of $\bv\mapsto F_S(\bw,\bv)$ gives
\begin{equation}\label{opt-sc-1}
\sum_{t=1}^{T}\big[F_S(\bw_t,\bv_t)-F_S(\bw,\bar{\bv}_T)\big]\leq 2\rho t_0B_W^2+\frac{G^2\log(eT)}{2\rho}+\frac{\rho}{2}\sum_{t=1}^{T}(t+t_0)\xi_t(\bw).
\end{equation}
Since the above inequality holds for any $\bw$, we know
\[
\sum_{t=1}^{T}\big[F_S(\bw_t,\bv_t)-\inf_{\bw\in\wcal}F_S(\bw,\bar{\bv}_T)\big]\leq 2\rho t_0B_W^2+\frac{G^2\log(eT)}{2\rho}+\frac{\rho}{2}\sup_{\bw\in\wcal}\sum_{t=1}^{T}(t+t_0)\xi_t(\bw).
\]
Since $\ebb_A[\langle\bw_t,\nabla_{\bw} f(\bw_t,\bv_t;z_{i_t})-\nabla_{\bw} F_S(\bw_t,\bv_t)\rangle]=0$ we know
\begin{align*}
\ebb_A\big[\sup_{\bw\in\wcal}\sum_{t=1}^{T}(t+t_0)\xi_t(\bw)\big] &=2\ebb_A\Big[\sup_{\bw\in\wcal}\sum_{t=1}^{T}(t+t_0)\eta_t\langle\bw,\nabla_{\bw} f(\bw_t,\bv_t;z_{i_t})-\nabla_{\bw} F_S(\bw_t,\bv_t)\rangle\Big]\\
& \leq 2\sup_{\bw\in\wcal}\|\bw\|_2\ebb_A\Big\|\sum_{t=1}^{T}(t+t_0)\eta_t\big(\nabla_{\bw} f(\bw_t,\bv_t;z_{i_t})-\nabla_{\bw} F_S(\bw_t,\bv_t)\big)\Big\|_2\\
& \leq 2B_W\bigg(\ebb_A\bigg\|\sum_{t=1}^{T}(t+t_0)\eta_t\big(\nabla_{\bw} f(\bw_t,\bv_t;z_{i_t})-\nabla_{\bw} F_S(\bw_t,\bv_t)\big)\bigg\|_2^2\bigg)^{1/2}\\
& \leq 2B_W\left(\sum_{t=1}^{T}(t+t_0)^2\eta_t^2\ebb_A\|\nabla_{\bw} f(\bw_t,\bv_t;z_{i_t})\|_2^2\right)^{1/2}\leq 2B_WG\rho^{-1}\sqrt{T}.
\end{align*}
We can combine the above two inequalities together and derive
\[
\sum_{t=1}^{T}\ebb_A\big[F_S(\bw_t,\bv_t)-\inf_{\bw\in\wcal}F_S(\bw,\bar{\bv}_T)\big]\leq 2\rho t_0B_W^2+\frac{G^2\log(eT)}{2\rho}+B_WG\sqrt{T}.
\]
In a similar way one can show
\[
\sum_{t=1}^{T}\ebb_A\big[\sup_{\bv\in\vcal}F_S(\bar{\bw}_T,\bv)-F_S(\bw_t,\bv_t)\big]\leq 2\rho t_0B_V^2+\frac{G^2\log(eT)}{2\rho}+B_VG\sqrt{T}.
\]
We can combine the above two inequalities together, and get the following optimization error bounds
\[
T\ebb_A\big[\sup_{\bv\in\vcal}F_S(\bar{\bw}_T,\bv)-\inf_{\bw\in\wcal}F_S(\bw,\bar{\bv}_T)\big]\leq 2\rho t_0(B_W^2+B_V^2)+\frac{G^2\log(eT)}{\rho}+(B_W+B_V)G\sqrt{T}.
\]
This proves \eqref{opt-sc-a} with $t_0=0$.

We now turn to \eqref{opt-sc-b}. Since $\ebb_A[\xi_t(\bw)]=0$, it follows from \eqref{opt-sc-1} that
\[
\sum_{t=1}^{T}\ebb_A\big[F_S(\bw_t,\bv_t)-F_S(\bw,\bar{\bv}_T)\big]\leq 2\rho t_0B_W^2+\frac{G^2\log(eT)}{2\rho}.
\]
In a similar way, one can show
\[
\sum_{t=1}^{T}\ebb_A\big[F_S(\bar{\bw}_T,\bv)-F_S(\bw_t,\bv_t)\big]\leq 2\rho t_0B_V^2+\frac{G^2\log(eT)}{2\rho}.
\]
We can combine the above two inequalities together and derive
\[
\ebb\big[F_S(\bar{\bw}_T,\bv)-F_S(\bw,\bar{\bv}_T\big]\leq \frac{2\rho t_0(B_W^2+B_V^2)}{T}+\frac{G^2\log(eT)}{\rho T}.
\]
The stated bound \eqref{opt-sc-b} then follows by taking the supremum over $\bw$ and $\bv$. The proof is complete.
\end{proof}

\section{Proofs on Generalization Bounds: Convex-Concave Case\label{sec:proof-gen-sgda}}
In this section, we prove the generalization bounds of SGDA in a convex-concave case. We first prove Theorem \ref{thm:gen} on bounds of weak PD population risks in expectation.
\begin{proof}[Proof of Theorem \ref{thm:gen}]
We first prove Part (a). We have the decomposition
\begin{equation}\label{decomposition-weak}
\triangle^w(\bar{\bw}_T,\bar{\bv}_T)=\triangle^w(\bar{\bw}_T,\bar{\bv}_T)-\triangle^w_S(\bar{\bw}_T,\bar{\bv}_T)+\triangle^w_S(\bar{\bw}_T,\bar{\bv}_T).
\end{equation}
%\begin{proof}[Proof of Theorem \ref{thm:stab-gda}]
%Without loss of generality, we assume that $S$ and $S'$ differ by the last example, i.e., $S'=\{z_1,\ldots,z_{n-1},z_n'\}$.
 According to Part (a) of Theorem \ref{lem:stab-gda} we know the following inequality for all $t$
  \[
    \ebb_A\left[\left\|\begin{pmatrix}
           \bw_{t+1}-\bw_{t+1}' \\
           \bv_{t+1}-\bv_{t+1}'
         \end{pmatrix}\right\|_2\right]\leq 4\eta G\Big(\sqrt{t}+\frac{t}{n}\Big).
  \]
  It then follows from the convexity of a norm that
    \[
    \ebb_A\left[\left\|\begin{pmatrix}
           \bar{\bw}_T-\bar{\bw}_T' \\
           \bar{\bv}_T-\bar{\bv}_T'
         \end{pmatrix}\right\|_2\right]\leq 4\eta G\Big(\sqrt{T}+\frac{T}{n}\Big)
  \]
  and therefore
  \begin{multline*}
      \sup_{z}\Big(\sup_{\bv'\in\vcal}\ebb_A\big[f(\bar{\bw}_T,\bv';z)-f(\bar{\bw}_T',\bv';z)\big]+\sup_{\bw'\in\wcal}\ebb_A\big[f(\bw',\bar{\bv}_T;z)-f(\bw',\bar{\bv}_T';z)\big]\Big)\\
      \leq G\Big(\ebb_A\big[\big\|\bar{\bw}_T-\bar{\bw}_T'\|_2\big]+\ebb_A\big[\big\|\bar{\bv}_T-\bar{\bv}_T'\big\|_2\big]\Big)\leq 4\sqrt{2}\eta G^2\Big(\sqrt{T}+\frac{T}{n}\Big).
  \end{multline*}
%  It then follows from Assumption \ref{ass:lipschitz} that
%  \begin{align}
%    & \ebb_A\Big(\sup_{\bv'\in\vcal}\big[f(A_{\bw}(S),\bv';z)-f(A_{\bw}(S'),\bv';z)\big]+\sup_{\bw'\in\wcal}\big[f(\bw',A_{\bv}(S);z)-f(\bw',A_{\bv}(S');z)\big]\Big)\\
%    & \leq G\ebb_A\Big[\|A_{\bw}(S)-A_{\bw}(S')\|_2+\|A_{\bv}(S)-A_{\bv}(S')\|_2\Big]\leq 4\eta\sqrt{2}G^2\Big(\sqrt{T}+\frac{T}{n}\Big).\label{stab-gda-3}
%  \end{align}
%\end{proof}
According to Part (a) of Theorem \ref{thm:stab-gen}, we know
\[
\triangle^w(\bar{\bw}_T,\bar{\bv}_T)-\triangle^w_S(\bar{\bw}_T,\bar{\bv}_T)\leq 4\sqrt{2}\eta G^2\Big(\sqrt{T}+\frac{T}{n}\Big).
\]
According to Eq. \eqref{opt-a}, we know
\[
\triangle^w_S(\bar{\bw}_T,\bar{\bv}_T)\leq \eta G^2 + \frac{B_W^2+B_V^2}{2\eta T}+\frac{G(B_W+B_V)}{\sqrt{T}}.
\]
The bound \eqref{gen-a} then follows directly from \eqref{decomposition-weak}.

Eq. \eqref{gen-c} in Part (b) can be proved in a similar way (e.g., by combining the stability bounds in Part (b) of Theorem \ref{lem:stab-gda} and optimization error bounds in Eq. \eqref{opt-a} together). We omit the proof for brevity.

We now turn to Part (c). According to Part (e) of Theorem \ref{lem:stab-gda} and the convexity of norm, we know
    \[
    \ebb_A\left[\left\|\begin{pmatrix}
           \bar{\bw}_T-\bar{\bw}_T' \\
           \bar{\bv}_T-\bar{\bv}_T'
         \end{pmatrix}\right\|_2\right]\leq \frac{2\sqrt{2}G}{\rho}\Big(\frac{\log^{\frac{1}{2}}(eT)}{\sqrt{T}}+\frac{1}{\sqrt{n(n-2)}}\Big).
  \]
  Analyzing analogous to Part (a), we further know
  \[
  \triangle^w(\bar{\bw}_T,\bar{\bv}_T)-\triangle^w_S(\bar{\bw}_T,\bar{\bv}_T)\leq\frac{4G^2}{\rho}\Big(\frac{\log^{\frac{1}{2}}(eT)}{\sqrt{T}}+\frac{1}{\sqrt{n(n-2)}}\Big).
  \]
  This together with the optimization error bounds in Lemma \ref{lem:opt-sc} and \eqref{decomposition-weak} gives
  \[
  \triangle^w(\bar{\bw}_T,\bar{\bv}_T)\leq \frac{4G^2}{\rho}\Big(\frac{\log^{\frac{1}{2}}(eT)}{\sqrt{T}}+\frac{1}{\sqrt{n(n-2)}}\Big)
  +  \frac{G^2\log(eT)}{\rho T}+\frac{(B_W+B_V)G}{\sqrt{T}}.
 \]
  The stated bound then follows from the choice of $T$. The proof is complete.

  Finally, we consider Part (d). Since $t_0\geq L^2/\rho^2$ we know
  $
  \eta_t=1/(\rho(t+t_0))\leq \rho/L^2.
  $
  The stability analysis in \citet{farnia2020train}\footnote{\citet{farnia2020train} considered the constant step size $\eta_t=\eta\leq \rho/L^2$. It is direct to extend the analysis there to any step size $\eta_t\leq\rho/L^2$ since an algorithm would be more stable if the step size decreases.} then shows that $A$ is $\epsilon$-argument stable with
  $
  \epsilon=O(1/(\rho n)).
  $
  This together with Part (a) of Theorem \ref{thm:stab-gen} then shows that
  \[
  \triangle^w(\bar{\bw}_T,\bar{\bv}_T)-\triangle^w_S(\bar{\bw}_T,\bar{\bv}_T)=O(1/(\rho n)).
  \]
  We can combine the above generalization bound and the optimization error bound in \eqref{opt-sc-b} together, and get
  \[
  \triangle^w(\bar{\bw}_T,\bar{\bv}_T)=O(1/(\rho n))+O\Big(\frac{\rho}{T}+\frac{\log(eT)}{\rho T}\Big).
  \]
  The stated bound then follows from $T\asymp n$. The proof is complete.
\end{proof}

We now present proofs of Theorem \ref{thm:gen-strong} on primal population risks.
\begin{proof}[Proof of Theorem \ref{thm:gen-strong}]
We have the decomposition
%\begin{align}\label{gen-strong-1}
%R(A_{\bw}(S))-\inf_{\bw}R(\bw)=R(A_{\bw}(S))-R_{S}(A_{\bw}(S))+R_{S}(A_{\bw}(S))-\inf_{\bw}R(\bw)
%%\triangle^s(\bar{\bw}_T,\bar{\bv}_T)=\triangle^s(\bar{\bw}_T,\bar{\bv}_T)-\triangle^s_S(\bar{\bw}_T,\bar{\bv}_T)+\triangle^s_S(\bar{\bw}_T,\bar{\bv}_T).
%\end{align}
\begin{multline*}
R(\bar{\bw}_T)-R(\bw^*)=\big(R(\bar{\bw}_T)-R_S(\bar{\bw}_T)\big)
+\big(R_S(\bar{\bw}_T)-F_S(\bw^*,\bar{\bv}_T)\big)\\+\big(F_S(\bw^*,\bar{\bv}_T)
-F(\bw^*,\bar{\bv}_T)\big)+\big(F(\bw^*,\bar{\bv}_T)-F(\bw^*,\bv^*)\big).
\end{multline*}
Since $F(\bw^*,\bar{\bv}_T)\leq F(\bw^*,\bv^*)$, it then follows that
\begin{equation}\label{decomposition-hp-primal}
R(\bar{\bw}_T)-R(\bw^*)\leq \big(R(\bar{\bw}_T)-R_S(\bar{\bw}_T)\big)
+\big(R_S(\bar{\bw}_T)-F_S(\bw^*,\bar{\bv}_T)\big)+\big(F_S(\bw^*,\bar{\bv}_T)
-F(\bw^*,\bar{\bv}_T)\big).
\end{equation}
Taking an expectation on both sides gives
\begin{equation}\label{decomposition-hp-primal-1}
\ebb\big[R(\bar{\bw}_T)-R(\bw^*)\big]\leq \ebb\big[R(\bar{\bw}_T)-R_S(\bar{\bw}_T)\big]
+\ebb\big[R_S(\bar{\bw}_T)-F_S(\bw^*,\bar{\bv}_T)\big]+\ebb\big[F_S(\bw^*,\bar{\bv}_T)
-F(\bw^*,\bar{\bv}_T)\big].
\end{equation}
Note that the first and the third term on the right-hand side is related to generalization, while the second term $R_S(\bar{\bw}_T)-F_S(\bw^*,\bar{\bv}_T)$ is related to optimization.
According to Part (b) of Theorem \ref{lem:stab-gda} we know the following inequality for all $t$
  \[
    \ebb_A\left[\left\|\begin{pmatrix}
           \bw_{t+1}-\bw_{t+1}' \\
           \bv_{t+1}-\bv_{t+1}'
         \end{pmatrix}\right\|_2\right]\leq \frac{G\sqrt{8e(t+t^2/n)}}{\sqrt{n}}\exp\Big(L^2t\eta^2/2\Big)\eta.
  \]
  It then follows from the convexity of a norm that
    \begin{equation}\label{primal-01}
    \ebb_A\left[\left\|\begin{pmatrix}
           \bar{\bw}_T-\bar{\bw}_T' \\
           \bar{\bv}_T-\bar{\bv}_T'
         \end{pmatrix}\right\|_2\right]\leq \frac{G\sqrt{8e(T+T^2/n)}}{\sqrt{n}}\exp\Big(L^2T\eta^2/2\Big)\eta.
  \end{equation}
  This together with Part (b) of Theorem \ref{thm:stab-gen} implies that
  \[
  \ebb_{S,A}\Big[R(\bar{\bw}_T)-R_S(\bar{\bw}_T)\Big]\leq
  \frac{\big(1+L/\rho\big)G^2\eta\sqrt{8e(T+T^2/n)}\exp\big(L^2T\eta^2/2\big)}{\sqrt{n}}.
  \]
  Similarly, the stability bound \eqref{primal-01} also implies the following bound on the gap between the population and empirical risk
  \[
  \ebb_{S,A}\Big[F_S(\bw^*,\bar{\bv}_T)-F(\bw^*,\bar{\bv}_T)\Big]\leq
  \frac{\big(1+L/\rho\big)G^2\eta\sqrt{8e(T+T^2/n)}\exp\big(L^2T\eta^2/2\big)}{\sqrt{n}}.
  \]
  According to Lemma \ref{lem:opt}, we know
  \[
  \ebb_A\big[R_S(\bar{\bw}_T)-F_S(\bw^*,\bar{\bv}_T)\big]\leq \ebb_A\big[\sup_{\bv\in\vcal}F_S(\bar{\bw}_T,\bv)-\inf_{\bw\in\wcal}F_S(\bw,\bar{\bv}_T)\big]\leq \eta G^2 + \frac{B_W^2+B_V^2}{2\eta T}+ \frac{G(B_W+B_V)}{\sqrt{T}}.
  \]
  We can plug the above three inequalities back  into \eqref{decomposition-hp-primal-1}, and derive the stated bound on the excess primal population risk in expectation.

  We now turn to the high-probability bounds. According to Assumption \ref{ass:lipschitz} and Part (d) of Theorem \ref{lem:stab-gda}, we know that with probability at least $1-\delta/4$ that SGDA is $\epsilon$-uniformly stable, where $\epsilon$ satisfies
\begin{equation}\label{epsilon-1}
\epsilon=O\Big(\eta\exp(L^2T\eta^2/2)\big(Tn^{-1}+\log(1/\delta)\big)\Big).
\end{equation}
This together with Part (d) of Theorem \ref{thm:stab-gen} implies the following inequality with probability at least $1-\delta/2$
\[
  R(\bar{\bw}_T)-R_S(\bar{\bw}_T)=O\Big(L\rho^{-1}\epsilon\log n\log(1/\delta)+n^{-\frac{1}2}\sqrt{\log(1/\delta)}\Big),
\]
where $\epsilon$ satisfies \eqref{epsilon-1}.
In a similar way, one can use Part (d) of Theorem \ref{thm:stab-gen} and stability bounds in Part (d) of Theorem \ref{lem:stab-gda} to show the following inequality with probability at least $1-\delta/4$
\begin{equation}\label{gen-hp-3}
F_S(\bw^*,\bar{\bv}_T)-F(\bw^*,\bar{\bv}_T)=
O\Big(\log n\log(1/\delta)\epsilon\Big)+O(n^{-\frac{1}{2}}\log^{\frac{1}2}(1/\delta)).
\end{equation}
According to \eqref{opt-b}, we derive the following inequality with probability at least $1-\delta/4$
\begin{align*}
    R_S(\bar{\bw}_T)-F_S(\bw^*,\bar{\bv}_T) %&= \sup_{\bv\in\vcal}F_S(\bar{\bw}_T,\bv)-\inf_{\bw\in\wcal}F_S(\bw,\bar{\bv}_T)+\inf_{\bw\in\wcal}F_S(\bw,\bar{\bv}_T)-F_S(\bw^*,\bar{\bv}_T)\\
    & = \sup_{\bv\in\vcal}F_S(\bar{\bw}_T,\bv)-F_S(\bw^*,\bar{\bv}_T)
     = O\Big(\eta + (T\eta)^{-1}+T^{-\frac{1}{2}}\log(1/\delta)\Big).
\end{align*}
We can plug the above three inequalities back into \eqref{decomposition-hp-primal} and derive the following inequality with probability at least $1-\delta$
\begin{multline}\label{hp-primal}
R(\bar{\bw}_T)-R(\bw^*)=O\Big(L\rho^{-1}\eta\exp(L^2T\eta^2/2)\log n\log(1/\delta)\big(Tn^{-1}+\log(1/\delta)\big)\Big)+O(n^{-\frac{1}{2}}\sqrt{\log(1/\delta)})\\
+O\Big(\eta + (T\eta)^{-1}+T^{-\frac{1}{2}}\log(1/\delta)\Big).
\end{multline}
The high-probability bound \eqref{primal-hp} then follows from the choice of $T$ and $\eta$.
The proof is complete.
%  It then follows from Assumption \ref{ass:lipschitz} that
%  \begin{align}
%    & \ebb_A\Big(\sup_{\bv'\in\vcal}\big[f(A_{\bw}(S),\bv';z)-f(A_{\bw}(S'),\bv';z)\big]+\sup_{\bw'\in\wcal}\big[f(\bw',A_{\bv}(S);z)-f(\bw',A_{\bv}(S');z)\big]\Big)\\
%    & \leq G\ebb_A\Big[\|A_{\bw}(S)-A_{\bw}(S')\|_2+\|A_{\bv}(S)-A_{\bv}(S')\|_2\Big]\leq 4\eta\sqrt{2}G^2\Big(\sqrt{T}+\frac{T}{n}\Big).\label{stab-gda-3}
%  \end{align}
%\end{proof}
\end{proof}

Finally, we present high-probability bounds of plain generalization errors for SGDA.
\begin{theorem}[High-probability bounds]\label{thm:gen-hp}
  Let $\{\bw_t,\bv_t\}$ be the sequence produced by \eqref{gda} with $\eta_t=\eta$. Assume for all $z$, the function $(\bw,\bv)\mapsto f(\bw,\bv;z)$ is convex-concave. Let $A$ be defined by $A_{\bw}(S)=\bar{\bw}_T$ and $A_{\bv}(S)=\bar{\bv}_T$ for $(\bar{\bw}_T,\bar{\bv}_T)$ in \eqref{bar-w-v}. Let $\sup_{\bw\in\wcal}\|\bw\|_2\leq B_W,\sup_{\bv\in\vcal}\|\bv\|_2\leq B_V$ and $\delta\in(0,1)$. Let $\widetilde{\triangle}_T=\big|F(\bar{\bw}_T,\bar{\bv}_T) - F(\bw^*,\bv^*)\big|$.
  \begin{enumerate}[label=(\alph*)]
    \item If Assumption \ref{ass:lipschitz} holds, then with probability at least $1-\delta$
    %\begin{multline*}%\label{gen-hp-a}
    %\hspace*{-0.6cm}
    \[\widetilde{\triangle}_T=
    O\Big(\eta\log n\log(1/\delta)\big(\sqrt{T}+Tn^{-1}+\log(1/\delta)\big)\Big)+
    O(n^{-\frac{1}{2}}\log^{\frac{1}2}(1/\delta))
    +O\Big((T\eta)^{-1}+T^{-\frac{1}{2}}\log(1/\delta)\Big).\]
    %\end{multline*}
    If we choose $T\asymp n^2$ and $\eta\asymp T^{-3/4}$ then we get the following inequality with probability at least $1-\delta$
    \begin{equation}\label{gen-hp-b}
        \widetilde{\triangle}_T=O(n^{-1/2}\log n\log^2(1/\delta)).
   \end{equation}
    \item If Assumptions \ref{ass:lipschitz}, \ref{ass:smooth} hold, then the following inequality holds with probability at least $1-\delta$
    %\begin{multline*}%\label{gen-hp-c}
    %\hspace*{-1cm}
    \[\widetilde{\triangle}_T=O\Big(\eta\log n\log(1/\delta)
    \exp\big(L^2T\eta^2/2\big)\big(Tn^{-1}+\log(1/\delta)\big)
    +n^{-\frac{1}{2}}\log^{\frac{1}2}(1/\delta)
    +(T\eta)^{-1}+T^{-\frac{1}{2}}\log(1/\delta)\Big).\]
    %\end{multline*}
    In particular, we can choose $T\asymp n$ and $\eta\asymp T^{-1/2}$ to derive \eqref{gen-hp-b} with probability at least $1-\delta$.
  \end{enumerate}
\end{theorem}
\begin{proof}%[Proof of Theorem \ref{thm:gen-hp}]
We use the error decomposition
\begin{multline}\label{gen-hp-1}
    F(\bar{\bw}_T,\bar{\bv}_T) - F(\bw^*,\bv^*)= F(\bar{\bw}_T,\bar{\bv}_T) - F_S(\bar{\bw}_T,\bar{\bv}_T) + F_S(\bar{\bw}_T,\bar{\bv}_T) -F_S(\bw^*,\bar{\bv}_T)\\
    +F_S(\bw^*,\bar{\bv}_T)-F(\bw^*,\bar{\bv}_T)
    +F(\bw^*,\bar{\bv}_T)-F(\bw^*,\bv^*).
\end{multline}
We first prove Part (a).
According to Assumption \ref{ass:lipschitz} and Part (c) of Theorem \ref{lem:stab-gda}, we know that SGDA is $\epsilon$-uniformly stable with probability at least $1-\delta/4$, where
\[
\epsilon=O\Big(\eta\big(\sqrt{T}+Tn^{-1}+\log(1/\delta)\big)\Big).%(Tn^{-1}\log(1/\delta))^{\frac{1}{2}}\big)
\]
This together with Part (e) of Theorem \ref{thm:stab-gen} implies the following inequality with probability at least $1-\delta/2$
\begin{equation}\label{gen-hp-2}
F(\bar{\bw}_T,\bar{\bv}_T)-F_S(\bar{\bw}_T,\bar{\bv}_T)=
O\Big(\eta\log n\log(1/\delta)\big(\sqrt{T}+Tn^{-1}+\log(1/\delta)\big)\Big)+O(n^{-\frac{1}{2}}\log^{\frac{1}2}(1/\delta)).
\end{equation}
Similarly, the following inequality holds with probability at least $1-\delta/4$
\begin{equation}\label{gen-hp-33}
    F_S(\bw^*,\bar{\bv}_T)-F(\bw^*,\bar{\bv}_T)=O\Big(\eta\log n\log(1/\delta)\big(\sqrt{T}+Tn^{-1}+\log(1/\delta)\big)\Big)+O(n^{-\frac{1}{2}}\log^{\frac{1}2}(1/\delta)).
\end{equation}
%In a similar way, one can show the following inequality with probability at least $1-\delta/4$
%\begin{equation}\label{gen-hp-3}
%F_S(\bw^*,\bar{\bv}_T)-F(\bw^*,\bar{\bv}_T)=
%O\Big(\eta\log n\log(1/\delta)\big(\sqrt{T}+Tn^{-1}+(Tn^{-1}\log(1/\delta))^{\frac{1}{2}}\big)\Big)+O(n^{-\frac{1}{2}}\log^{\frac{1}2}(1/\delta)).
%\end{equation}
According to Lemma \ref{lem:opt}, the following inequality holds with probability at least $1-\delta/4$
\begin{align}
F_S(\bar{\bw}_T,\bar{\bv}_T) -F_S(\bw^*,\bar{\bv}_T) &\leq
\sup_{\bv}F_S(\bar{\bw}_T,\bv) -\inf_{\bw} F_S(\bw,\bar{\bv}_T) = O\Big(\eta +(T\eta)^{-1}+T^{-\frac{1}{2}}\log(1/\delta)\Big).\label{gen-hp-4}
\end{align}
According to the definition of $(\bw^*,\bv^*)$, we know $F(\bw^*,\bar{\bv}_T)\leq F(\bw^*,\bv^*)$. We can plug this inequality and \eqref{gen-hp-2}, \eqref{gen-hp-33}, \eqref{gen-hp-4} back into \eqref{gen-hp-1}, and derive the following inequality with probability at least $1-\delta/2$
\begin{multline*}
F(\bar{\bw}_T,\bar{\bv}_T) - F(\bw^*,\bv^*)=O\Big(\eta\log n\log(1/\delta)\big(\sqrt{T}+Tn^{-1}+\log(1/\delta)\big)\Big)\\
    +O(n^{-\frac{1}{2}}\log^{\frac{1}2}(1/\delta))
    +O\Big((T\eta)^{-1}+T^{-\frac{1}{2}}\log(1/\delta)\Big).
\end{multline*}
Analyzing in a similar way but using the error decomposition
\begin{multline*}
    F(\bw^*,\bv^*)-F(\bar{\bw}_T,\bar{\bv}_T)= F(\bw^*,\bv^*)-F(\bar{\bw}_T,\bv^*)+F(\bar{\bw}_T,\bv^*)-F_S(\bar{\bw}_T,\bv^*)\\
    +F_S(\bar{\bw}_T,\bv^*)-F_S(\bar{\bw}_T,\bar{\bv}_T)+F_S(\bar{\bw}_T,\bar{\bv}_T)-F(\bar{\bw}_T,\bar{\bv}_T),
\end{multline*}
one can derive the following inequality with probability at least $1-\delta/2$
\begin{multline*}
 F(\bw^*,\bv^*)-F(\bar{\bw}_T,\bar{\bv}_T)=O\Big(\eta\log n\log(1/\delta)\big(\sqrt{T}+Tn^{-1}+\log(1/\delta)\big)\Big)\\
    +O(n^{-\frac{1}{2}}\log^{\frac{1}2}(1/\delta))
    +O\Big((T\eta)^{-1}+T^{-\frac{1}{2}}\log(1/\delta)\Big).
\end{multline*}
The stated bound then follows as a combination of the above two inequalities.

Part (b) can be derived similarly excepting using the stability bounds in Part (d) of Theorem \ref{lem:stab-gda}. We omit the proof for brevity. The proof is complete.
\end{proof}

%Furthermore, the standard Hoeffding inequality implies that with probability at least $1-\delta/4$ there holds
%\[
%F_S(\bw^*,\bv^*)-F(\bw^*,\bv^*)=O(\sqrt{\log(1/\delta)}/\sqrt{n}).
%\]

\section{Stability and Generalization Bounds of SGDA on Non-Convex Objectives}

\subsection{Proof of Theorem \ref{thm:stability-wcwc}\label{sec:sgda-wcwc}}

%%%%%%%%%%%%%%%%%%%%%%%%%%%%%%%%%% WCWC %%%%%%%%%%%%%%%%%%%%%%%%%%%%%%%%%

In this section, we show the stability and generalization bounds of SGDA for weakly-convex-weakly-concave objectives. We first introduce some lemmas. As an extension of a lemma in \citet{hardt2016train}, the next lemma is motivated by the fact that SGDA typically runs several iterations before encountering the different example between $S$ and $S'$.

\begin{lemma}\label{lem:stab-gen-nonconvex}
Assume $|f(\cdot,\cdot, z)| \leq 1$ for any $z$ and let Assumption \ref{ass:lipschitz} hold. Let $S=\{z_1,\ldots,z_n\}$ and $S'=\{z_1,\ldots,z_{n-1},z_n'\}$. Let $\{\bw_t,\bv_t\}$ and $\{\bw_t',\bv_t'\}$ be the sequence produced by \eqref{gda} w.r.t. $S$ and $S'$, respectively. Denote
\begin{align}\label{eq:event-E}
\Delta_t = 	\left\|\begin{pmatrix}
    \bw_t-\bw'_t \\
    \bv_t-\bv'_t
  \end{pmatrix}\right\|_2.
\end{align}
Then for any $t_0\in\nbb$ and any $\bw', \bv'$ we have
\begin{align*}
\ebb[f(\bw_T,\bv';z) - f(\bw'_T,\bv';z) + f(\bw',\bv_T;z) - f(\bw',\bv'_T;z)] \leq \frac{4t_0}{n} + \sqrt{2}G\ebb[\Delta_T | \Delta_{t_0}=0] .
\end{align*}
\end{lemma}

\begin{proof}
According to Assumption \ref{ass:lipschitz}, we know
\begin{equation}\label{stab-gen-non-convex-1}
  f(\bw_T,\bv';z) - f(\bw'_T,\bv';z) + f(\bw',\bv_T;z) - f(\bw',\bv'_T;z)\leq G\sqrt{2}\Delta_T.
\end{equation}
Let $\ecal$ denote the event that $\Delta_{t_0} = 0$. Then we have
\begin{align*}
& \ebb[f(\bw_T,\bv';z) - f(\bw'_T,\bv';z) + f(\bw',\bv_T;z) - f(\bw',\bv'_T;z)]\\
= & \pbb[\ecal]\ebb[f(\bw_T,\bv';z) - f(\bw'_T,\bv';z) + f(\bw',\bv_T;z) - f(\bw',\bv'_T;z) | \ecal] \\
& + \pbb[\ecal^c]\ebb[f(\bw_T,\bv';z) - f(\bw'_T,\bv';z) + f(\bw',\bv_T;z) - f(\bw',\bv'_T;z) | \ecal^c]\\
\leq & \sqrt{2}G\ebb[\Delta_T | \ecal]  + 4\pbb[\ecal^c],
\end{align*}
where in the last step we have used \eqref{stab-gen-non-convex-1} and the condition $|f(\cdot,\cdot, z)| \leq 1$.
Using the union bound on the outcome $i_t = n$ we obtain that
\begin{align*}
\pbb[\ecal^c] \leq \sum_{t=1}^{t_0}\pbb[i_t = n] = \frac{t_0}{n}.
\end{align*}
The proof is complete by combining the above two inequalities together.
\end{proof}

%We require the following inequality. The first one characterizes the monotonicity of the gradient of the objective function $f$ under weak-convexity-weak-concavity.
Lemma \ref{lem:monotone-wcwc} shows the monotonity of the gradient for weakly-convex-weakly-concave functions. Its proof is well known in the literature~\citep{rockafellar1976monotone,liu2020firstorder}.
\begin{lemma}\label{lem:monotone-wcwc}
Let $f$ be a $\rho$-weakly-convex-weakly-concave function. Then
\begin{equation}\label{eq:monotone-wcwc}
  \bigg\langle\begin{pmatrix}
    \bw-\bw' \\
    \bv-\bv'
  \end{pmatrix},\begin{pmatrix}
                  \nabla_{\bw}f(\bw,\bv)-\nabla_{\bw}f(\bw',\bv') \\
                  \nabla_{\bv}f(\bw',\bv')-\nabla_{\bv}f(\bw,\bv)
                \end{pmatrix}\bigg\rangle \geq -\rho\bigg\|\begin{pmatrix}
  \bw-\bw' \\
  \bv-\bv'
  \end{pmatrix}\bigg\|_2^2.
  \end{equation}
\end{lemma}

We are now ready to prove Theorem \ref{thm:stability-wcwc}.
\begin{proof}[Proof of Theorem \ref{thm:stability-wcwc}]
Note that the projection step is nonexpansive. We consider two cases at the $t$-th iteration. If $i_t\neq n$, then it follows from Lemma \ref{lem:monotone-wcwc} and the Lipschitz continuity of $f$ that
\begin{align*}\label{nonconvex-sgda-1}
& \left\|\begin{pmatrix}
           \bw_{t+1}-\bw_{t+1}' \\
           \bv_{t+1}-\bv_{t+1}'
         \end{pmatrix}\right\|_2^2
         \leq \left\|\begin{pmatrix}
                   \bw_t-\eta_t\nabla_{\bw}f(\bw_t,\bv_t;z_{i_t})-\bw_t'+\eta_t\nabla_{\bw}f(\bw_t',\bv'_t;z_{i_t}) \\
                   \bv_t+\eta_t\nabla_{\bv}f(\bw_t,\bv_t;z_{i_t})-\bv_t'-\eta_t\nabla_{\bv}f(\bw_t',\bv'_t;z_{i_t})
                 \end{pmatrix}\right\|_2^2
        =  \bigg\|\begin{pmatrix}
                                                           \bw_t-\bw'_t \\
                                                           \bv_t-\bv'_t
                                                         \end{pmatrix}\bigg\|_2^2\\
  &+ \eta_t^2\bigg\|\begin{pmatrix}
  \nabla_{\bw}f(\bw'_t,\bv'_t;z_{i_t}) - \nabla_{\bw}f(\bw_t,\bv_t;z_{i_t})\\
  \nabla_{\bv}f(\bw_t,\bv_t;z_{i_t}) - \nabla_{\bv}f(\bw'_t,\bv'_t;z_{i_t})
  \end{pmatrix}\bigg\|_2^2 -2\eta_t\bigg\langle\begin{pmatrix}
                      \bw_t-\bw'_t \\
                      \bv_t-\bv'_t
                    \end{pmatrix},\begin{pmatrix}
                                    \nabla_{\bw}f(\bw_t,\bv_t;z_{i_t})-\nabla_{\bw}f(\bw'_t,\bv'_t;z_{i_t}) \\
                                    \nabla_{\bv}f(\bw'_t,\bv'_t;z_{i_t})-\nabla_{\bv}f(\bw_t,\bv_t;z_{i_t})
                                  \end{pmatrix}\bigg\rangle \\
         \leq & (1+2\eta_t\rho)\left\|\begin{pmatrix}
           \bw_{t}-\bw_{t}' \\
           \bv_{t}-\bv_{t}'
         \end{pmatrix}\right\|_2^2+8G^2\eta_t^2. \numberthis
\end{align*}
If $i_t=n$, then it follows from the elementary inequality $(a+b)^2\leq(1+p)a^2+(1+1/p)b^2$ that
\begin{align}
  \left\|\begin{pmatrix}
           \bw_{t+1}-\bw_{t+1}' \\
           \bv_{t+1}-\bv_{t+1}'
         \end{pmatrix}\right\|_2^2
         &\leq\left\|\begin{pmatrix}
                   \bw_t-\eta_t\nabla_{\bw}f(\bw_t,\bv_t;z_{n})-\bw_t'+\eta_t\nabla_{\bw}f(\bw_t',\bv'_t;z'_{n}) \\
                   \bv_t+\eta_t\nabla_{\bv}f(\bw_t,\bv_t;z_{n})-\bv_t'-\eta_t\nabla_{\bv}f(\bw_t',\bv'_t;z'_{n})
                 \end{pmatrix}\right\|_2^2\notag\\
         & \leq (1+p)\left\|\begin{pmatrix}
           \bw_{t}-\bw_{t}' \\
           \bv_{t}-\bv_{t}'
         \end{pmatrix}\right\|_2^2
        +(1+1/p)\eta_t^2\left\|\begin{pmatrix}
                                 \nabla_{\bw}f(\bw_t,\bv_t;z_n)-\nabla_{\bw}f(\bw_t',\bv_t';z'_n) \\
                                 \nabla_{\bv}f(\bw_t,\bv_t;z_n)-\nabla_{\bv}f(\bw_t',\bv_t';z'_n)
                               \end{pmatrix}\right\|_2^2.\label{nonconvex-sgda-2}
\end{align}
Note that the event $i_t\neq n$ happens with probability $1-1/n$ and the event $i_t=n$ happens with probability $1/n$. Therefore,
we know
\begin{align*}
 \ebb_{i_t}\left[\left\|\begin{pmatrix}
           \bw_{t+1}-\bw_{t+1}' \\
           \bv_{t+1}-\bv_{t+1}'
         \end{pmatrix}\right\|_2^2\right] & \leq \frac{n-1}{n}\left((1+2\eta_t\rho)\left\|\begin{pmatrix}
           \bw_{t}-\bw_{t}' \\
           \bv_{t}-\bv_{t}'
         \end{pmatrix}\right\|_2^2+8G^2\eta_t^2\right) + \frac{1+p}{n}\left\|\begin{pmatrix}
           \bw_{t}-\bw_{t}' \\
           \bv_{t}-\bv_{t}'
         \end{pmatrix}\right\|_2^2+\frac{8(1+1/p)}{n}\eta_t^2G^2 \\
         & \leq (1+2\eta_t\rho + p/n)\left\|\begin{pmatrix}
           \bw_{t}-\bw_{t}' \\
           \bv_{t}-\bv_{t}'
         \end{pmatrix}\right\|_2^2+8\eta_t^2G^2(1+1/(np)).
\end{align*}
%{\color{red}: note we can not directly say $\bw_{t_0} = \bw_{t_0}'$. Instead we should state the following inequality in terms of condition. }
%Using the fact that $\bw_{t_0} = \bw_{t_0}'$ and $\bv_{t_0} = \bv_{t_0}'$,
Let $t_0\in\nbb$ and $\ecal$ be defined as in the proof of Lemma \ref{lem:stab-gen-nonconvex}.
We apply the above equation recursively from $t=t_0+1$ to $T$, then
\begin{align*}
  \ebb_A\left[\left\|\begin{pmatrix}
           \bw_{T}-\bw_{T}' \\
           \bv_{T}-\bv_{T}'
         \end{pmatrix}\right\|_2^2\bigg|\ \ecal\right] \leq 8G^2\big(1+1/(np)\big)\sum_{t=t_0 +1}^T\eta_t^2\prod_{k=t+1}^T\big(1+2\eta_k\rho + p/n\big).
\end{align*}
By the elementary inequality $1+a\leq \exp(a)$ and $\eta_t = \frac{c}{t}$, we further derive
\begin{align*}
  \ebb_A\left[\left\|\begin{pmatrix}
           \bw_{t+1}-\bw_{t+1}' \\
           \bv_{t+1}-\bv_{t+1}'
         \end{pmatrix}\right\|_2^2\bigg|\ \ecal\right] \leq & 8G^2\big(1+1/(np)\big)\sum_{t=t_0 + 1}^T\frac{c^2}{t^2}\prod_{k=t+1}^T\exp\Big(\frac{2c\rho}{k} + \frac{p}{n}\Big)\\
%         = & 8G^2\big(1+1/(np)\big)\sum_{t=t_0 + 1}^T\frac{c^2}{t^2}\exp\Big(\sum_{k=t+1}^T\frac{2c\rho}{k} + \frac{p(T-t)}{n}\Big)\\
         \leq & 8G^2\big(1+1/(np)\big)\sum_{t=t_0 + 1}^T\frac{c^2}{t^2}\exp\Big(\sum_{k=t+1}^T\frac{2c\rho}{k} + \frac{pT}{n}\Big).
\end{align*}
By taking $p=n/T$ in the above inequality, we further derive
\begin{align*}
  \ebb_A\left[\left\|\begin{pmatrix}
           \bw_{t+1}-\bw_{t+1}' \\
           \bv_{t+1}-\bv_{t+1}'
         \end{pmatrix}\right\|_2^2\bigg|\ \ecal\right] \leq & 8eG^2\Big(1+\frac{T}{n^2}\Big)\sum_{t=t_0 + 1}^T\frac{c^2}{t^2}\exp\Big(\sum_{k=t+1}^T\frac{2c\rho}{k} \Big)\\
         \leq & 8eG^2\Big(1+\frac{T}{n^2}\Big)\sum_{t=t_0 + 1}^T\frac{c^2}{t^2}\exp\Big(2c\rho\log\Big(\frac{T}{t}\Big) \Big)\\
         \leq & 8c^2eG^2\Big(1+\frac{T}{n^2}\Big)T^{2c\rho}\sum_{t=t_0 + 1}^T\frac{1}{t^{2c\rho + 2}}\\
         \leq & \frac{8c^2eG^2}{2c\rho + 1}\Big(1+\frac{T}{n^2}\Big)\Big(\frac{T}{t_0}\Big)^{2c\rho}\frac{1}{t_0}.
\end{align*}
Combining the above inequality and Lemma \ref{lem:stab-gen-nonconvex} together, we obtain
\begin{align}\label{new-eq-bound-b}
  \ebb_A[f(\bw_T,\bv';z) - f(\bw'_T,\bv';z) + f(\bw',\bv_T;z) - f(\bw',\bv'_T;z)] \leq &\frac{4t_0}{n} + \frac{4\sqrt{e}cG^2}{\sqrt{2c\rho + 1}}\Big(1+\frac{\sqrt{T}}{n}\Big)\Big(\frac{T}{t_0}\Big)^{c\rho}\frac{1}{\sqrt{t_0}}.
\end{align}
The right hand side is approximately minimized when
\[
t_0=\bigg(\frac{\sqrt{e}cG^2}{\sqrt{2c\rho + 1}}\Big(1+\frac{\sqrt{T}}{n}\Big)T^{c\rho}n\bigg)^{\frac{2}{2c\rho+3}}.
\]
%\begin{align*}
%t_0 = \Big(\frac{2\sqrt{2e}cG}{\sqrt{2c\rho+1}}\Big)^{\frac{1}{2c\rho+2}}n^{\frac{2}{2c\rho+3}}T^{\frac{2c\rho}{2c\rho + 3}}.
%\end{align*}
Plugging it into the Eq. \eqref{new-eq-bound-b} we have (for simplicity we assume the above $t_0$ is an integer)
\begin{align*}
\ebb_A[f(\bw_T,\bv';z) - f(\bw'_T,\bv';z) + f(\bw',\bv_T;z) - f(\bw',\bv'_T;z)] \leq & 8\bigg(\frac{\sqrt{e}cG^2}{\sqrt{2c\rho + 1}}\Big(1+\frac{\sqrt{T}}{n}\Big)T^{c\rho}\bigg)^{\frac{2}{2c\rho+3}}
         \bigg(\frac{1}{n}\bigg)^{\frac{2c\rho+1}{2c\rho+3}}.
         %(B_W+B_V)\Big(\frac{2\sqrt{2e}cG}{\sqrt{2c\rho+1}}\Big)^{\frac{1}{2c\rho+2}}n^{-\frac{2c\rho+1}{2c\rho+3}}T^{\frac{2c\rho}{2c\rho + 3}}.
\end{align*}
Since the above bound holds for all $z, S, S'$ and $\bw',\bv'$, we immediately get the same upper bound on the weak stability. Finally the theorem holds by calling Theorem \ref{thm:stab-gen}, Part (a).
\end{proof}

\subsection{High-Probability Stability and Generalization Bounds\label{sec:sgda-hp-nonconvex}}

In this section, we give stability and generalization bounds of SGDA with nonconvex-nonconcave smooth objectives with high probability. The analysis requires a tail bound for a linear combination of independent Bernoulli random variables \cite{raghavan1988probabilistic}.
%{\color{red}: what is the meaning of $B(1,p)$? Do not assume that the reader has already known this definition}
\begin{lemma}\label{lem:concentration}
Let $c_t \in (0,1]$ and let $X_1, \cdots, X_T$ be independent Bernoulli random variables with the success rate of $X_t$ being $p_t\in[0,1]$. Denote $s = \sum_{t=1}^Tc_tp_t$. Then, for all $a > 0$,
\begin{align*}
\pbb\Big[\sum_{t=1}^T c_tX_t \geq (1+a)s\Big] \leq \Big(\frac{e^a}{(1+a)^{(1+a)}}\Big)^s.
\end{align*}
In particular, for all $\delta \in (0,1)$ such that $\log(1/\delta) < s$ with probability at least $1 - \delta$ we have
\begin{align*}
    \sum_{t=1}^T c_tX_t \leq s + (e-1)\sqrt{\log(1/\delta)s}.
\end{align*}
\end{lemma}

%We also require another version of the Chernoff's bound. Note that Lemma \ref{lem:chernoff} holds for any $\tilde{\delta} > 0$. Here we need to constrain $\tilde{\delta} \in (0,1)$.

%\begin{lemma}[Chernoff's Bound\label{lem:chernoff-2}]
%  Let $X_1,\ldots,X_t$ be independent random variables taking values in $\{0,1\}$. Let $X=\sum_{j=1}^{t}X_j$ and $\mu=\ebb[X]$. Then for any $\tilde{\delta}\in (0,1)$ with probability at least $1-\exp\big(-\mu\tilde{\delta}^2/3\big)$ we have $X\leq (1+\tilde{\delta})\mu$. Furthermore, for any $\delta\in(0,1)$ such that $\log(1/\delta) < \mu/3$, with probability at least $1-\delta$ we have
%  \[
%  X\leq \mu+\sqrt{3\mu\log(1/\delta)}.
  %X\leq \Big(\frac{\log(1/\delta)}{\mu}+\frac{\sqrt{2\log(1/\delta)}}{\sqrt{\mu}}\Big)\mu.
%  \]
%\end{lemma}

\begin{theorem}\label{thm:stability-wcwc-prob}
Let $\{\bw_t, \bv_t\}$  be the sequence produced by \eqref{gda} with $\eta_t \leq \frac{c}{t}$ for some $c > 0$. Assume Assumption \ref{ass:lipschitz}, \ref{ass:smooth} hold and $|f(\cdot,\cdot;z)| \leq 1$. For any $\delta \in (0,1)$, if $c \leq \frac{1}{(n\log(2/\delta)-1)L}$, then with probability at least $1-\delta$ we have
\[
\big|F(\bw_T,\bv_T) - F_S(\bw_T,\bv_T)\big| = \ocal\Big(T^{cL}\log(n)\log^{3/2}(1/\delta)n^{-1/2} + n^{-1/2}\log^{1/2}(1/\delta)\Big).
\]
%    \begin{align*}
%\left\|\begin{pmatrix}
%           \bw_T-\bw_T' \\
%           \bv_T-\bv_T'
%         \end{pmatrix}\right\|_2 \leq & 4cGT^{cL}\Big(\frac{cL+1}{cLn} + (e-1)\sqrt{\frac{(cL+1)\log(1/\delta)}{cLn}}\Big).
%\end{align*}
\end{theorem}
\begin{proof}
Let $S'=\{z_1,\ldots,z_{n-1},z_n'\}$ and $\{\bw_t',\bv_t'\}$ be the sequence produced by \eqref{gda} w.r.t. $S'$. %Let $\eta_t = \max\{\eta_{\bw,t},\eta_{\bv,t}\}$.
If $i_{t}\neq n$, it follows from the $L$-smoothness of $f$ that
\begin{align*}
  \left\|\begin{pmatrix}
           \bw_{t+1}-\bw_{t+1}' \\
           \bv_{t+1}-\bv_{t+1}'
         \end{pmatrix}\right\|_2\leq & \left\|\begin{pmatrix}
           \bw_t-\bw_t' \\
           \bv_t-\bv_t'
         \end{pmatrix}\right\|_2 +
         \eta_t\left\|\begin{pmatrix}
                   \nabla_{\bw}f(\bw_t,\bv_t;z_{i_t})-\nabla_{\bw}f(\bw_t',\bv'_t;z_{i_t}) \\
                  \nabla_{\bv}f(\bw_t,\bv_t;z_{i_t})-\nabla_{\bv}f(\bw_t',\bv'_t;z_{i_t})
                 \end{pmatrix}\right\|_2
\leq  (1+ L\eta_t)\left\|\begin{pmatrix}
           \bw_t-\bw_t' \\
           \bv_t-\bv_t'
         \end{pmatrix}\right\|_2.
\end{align*}
If $i_t = n$, we have
\begin{align*}
  \left\|\begin{pmatrix}
           \bw_{t+1}-\bw_{t+1}' \\
           \bv_{t+1}-\bv_{t+1}'
         \end{pmatrix}\right\|_2 \leq & \left\|\begin{pmatrix}
           \bw_t-\bw_t' \\
           \bv_t-\bv_t'
         \end{pmatrix}\right\|_2 + 4\eta_t G.
\end{align*}
We can combine the above two inequalities together and get
\begin{align*}
\left\|\begin{pmatrix}
           \bw_{t+1}-\bw_{t+1}' \\
           \bv_{t+1}-\bv_{t+1}'
         \end{pmatrix}\right\|_2 \leq & (1 + L\eta_t ) \left\|\begin{pmatrix}
           \bw_{t}-\bw_{t}' \\
           \bv_{t}-\bv_{t}'
         \end{pmatrix}\right\|_2 +4G \eta_t \ibb_{[i_t=n]}.
\end{align*}
We apply the above inequality recursively from $t=1$ to $T$ and get
\begin{align*}
\left\|\begin{pmatrix}
           \bw_T-\bw_T' \\
           \bv_T-\bv_T'
         \end{pmatrix}\right\|_2\leq
         4G\sum_{t=1}^{T}\eta_t \ibb_{[i_t = n]}\prod_{k=t+1}^{T}\Big(1 + L\eta_k\Big).
\end{align*}
By the elementary inequality $1+a\leq\exp(a)$ and $\eta_t \leq \frac{c}{t}$, we further derive
\begin{align*}
\left\|\begin{pmatrix}
           \bw_T-\bw_T' \\
           \bv_T-\bv_T'
         \end{pmatrix}\right\|_2 \leq & 4cG\sum_{t=1}^{T}\frac{\ibb_{[i_t = n]}}{t}\prod_{k=t+1}^{T}\exp\Big(\frac{cL}{k}\Big)
         = 4cG\sum_{t=1}^{T}\frac{\ibb_{[i_t = n]}}{t}\exp\Big(\sum_{k=t+1}^{T}\frac{cL}{k}\Big)\\
         \leq & 4cG\sum_{t=1}^{T}\frac{\ibb_{[i_t = n]}}{t}\exp\Big(cL\log\Big(\frac{T}{t}\Big)\Big)
         \leq 4cGT^{cL}\sum_{t=1}^{T}\frac{\ibb_{[i_t = n]}}{t^{cL+1}}.
\end{align*}
By Lemma \ref{lem:concentration}, for any $\delta > 0$ such that $\log(2/\delta) < \sum_{t=1}^{T}\frac{1}{t^{cL+1}n}$, with probability at least $1- \delta/2$ we have
\begin{align*}\label{eq:sum-t}
\left\|\begin{pmatrix}
           \bw_T-\bw_T' \\
           \bv_T-\bv_T'
         \end{pmatrix}\right\|_2 \leq & 4cGT^{cL}\Big(\sum_{t=1}^{T}\frac{1}{t^{cL+1}n} + (e-1)\sqrt{\log(1/\delta)\sum_{t=1}^{T}\frac{1}{t^{cL+1}n}}\Big). \numberthis
\end{align*}
Note that
\begin{align*}
\sum_{t=1}^{T}\frac{1}{t^{cL+1}} \leq & 1 + \int_{t=1}^T\frac{\mathrm{d}t}{t^{cL+1}} \leq 1 + \frac{1}{cL}.
\end{align*}
Plugging the above bound into Equation \eqref{eq:sum-t} , we know with probability at least $1-\delta/2$
\begin{align*}
\left\|\begin{pmatrix}
           \bw_T-\bw_T' \\
           \bv_T-\bv_T'
         \end{pmatrix}\right\|_2 \leq & 4cGT^{cL}\Big(\frac{cL+1}{cLn} + (e-1)\sqrt{\frac{(cL+1)\log(1/\delta)}{cLn}}\Big).
\end{align*}
By the Lipschitz continuity of $f$, the above equation implies SGDA is $\epsilon$-uniformly stable with probability at least $1-\delta/2$ and
\[
\epsilon = \ocal\Big(T^{cL}\sqrt{\log(1/\delta)}n^{-\frac{1}{2}}\Big).
\]
This together with Part (e) of Theorem \ref{thm:stab-gen} implies the following inequality with probability at least $1-\delta$
\[
\big|F(\bw,\bv) - F_S(\bw_T,\bv_T)\big| = \ocal\Big(T^{cL}\log(n)\log^{3/2}(1/\delta)n^{-1/2} + n^{-1/2}\log^{1/2}(1/\delta)\Big).
\]
The proof is complete.
\end{proof}

%\begin{theorem}\label{thm:gen-prob}
%Let $\{\bw_t, \bv_t\}$  be the sequence produced by \eqref{gda} with $\eta_t \leq \frac{c}{t}$ for some $c > 0$. Assume Assumption \ref{ass:lipschitz}, \ref{ass:smooth} hold and $|f(\cdot,\cdot;z)| \leq 1$. For any $\delta \in (0,1)$, if $c \leq \frac{1}{(n\log(2/\delta)-1)L}$, then with probability at least $1-\delta$ we have
%\[
%\big|F(\bw_T,\bv_T) - F_S(\bw_T,\bv_T)\big| = \ocal\Big(T^{cL}\log(n)\log^{3/2}(1/\delta)n^{-1/2} + n^{-1/2}\log^{1/2}(1/\delta)\Big).
%\]
%\end{theorem}
%
%\begin{proof}
%According to Assumption \ref{ass:lipschitz} and Lemma \ref{lem:stability-wcwc-prob}, we know that SGDA is $\epsilon$-uniformly stable with probability at least $1 - \delta/2$,
%\[
%\epsilon = \ocal\Big(T^{cL}\sqrt{\log^{1/2}(1/\delta)}n^{-1/2\frac{1}{2}}\Big).
%\]
%This together with Part (e) of Theorem \ref{thm:stab-gen} implies the following inequality with probability at least $1-\delta$
%\[
%F(\bw,\bv) - F_S(\bw_T,\bv_T) = \ocal\Big(T^{cL}\log(n)\log^{3/2}(1/\delta)n^{-1/2} + n^{-1/2}\log^{1/2}(1/\delta)\Big).
%\]
%The proof is complete.
%\end{proof}

\subsection{Proof of Theorem \ref{thm:wcwc-diminishing}\label{sec:wcwc-diminishing}}
In this section, we prove Theorem \ref{thm:wcwc-diminishing} on generalization bounds under a regularity condition on the decay of weak-convexity-weak-concavity parameter along the optimization process.
\begin{proof}[Proof of Theorem \ref{thm:wcwc-diminishing}]
Let $S=\{z_1,\ldots,z_n\}$ and $S'=\{z_1',\ldots,z_n'\}$ be two neighboring datasets. Without loss of generality, we assume $z_i=z_i'$ for $i\in[n-1]$.
If $i_t\neq n$, then it follows from Assumption \ref{ass:smooth} that
\[
\left\|\begin{pmatrix}
                   \nabla_{\bw}f(\bw_t,\bv_t;z_{i_t})-\nabla_{\bw}f(\bw_t',\bv'_t;z'_{i_t}) \\
                  \nabla_{\bv}f(\bw_t,\bv_t;z_{i_t})-\nabla_{\bv}f(\bw_t',\bv'_t;z'_{i_t})
                 \end{pmatrix}\right\|_2^2=\left\|\begin{pmatrix}
                   \nabla_{\bw}f(\bw_t,\bv_t;z_{i_t})-\nabla_{\bw}f(\bw_t',\bv'_t;z_{i_t}) \\
                  \nabla_{\bv}f(\bw_t,\bv_t;z_{i_t})-\nabla_{\bv}f(\bw_t',\bv'_t;z_{i_t})
                 \end{pmatrix}\right\|_2^2 \leq L^2\left\|\begin{pmatrix}
           \bw_t-\bw_t' \\
           \bv_t-\bv_t'
         \end{pmatrix}\right\|_2^2
\]
If $i_t=n$, then it follows from Assumption \ref{ass:lipschitz} that
\[
\left\|\begin{pmatrix}
                   \nabla_{\bw}f(\bw_t,\bv_t;z_{i_t})-\nabla_{\bw}f(\bw_t',\bv'_t;z'_{i_t}) \\
                  \nabla_{\bv}f(\bw_t,\bv_t;z_{i_t})-\nabla_{\bv}f(\bw_t',\bv'_t;z'_{i_t})
                 \end{pmatrix}\right\|_2^2\leq 8G^2.
\]
Therefore, we have
\begin{equation}\label{wcwc-sgda-1}
  \ebb_{i_t}\left\|\begin{pmatrix}
                   \nabla_{\bw}f(\bw_t,\bv_t;z_{i_t})-\nabla_{\bw}f(\bw_t',\bv'_t;z'_{i_t}) \\
                  \nabla_{\bv}f(\bw_t,\bv_t;z_{i_t})-\nabla_{\bv}f(\bw_t',\bv'_t;z'_{i_t})
                 \end{pmatrix}\right\|_2^2\leq \frac{(n-1)L^2}{n}\left\|\begin{pmatrix}
           \bw_t-\bw_t' \\
           \bv_t-\bv_t'
         \end{pmatrix}\right\|_2^2+\frac{8G^2}{n}.
\end{equation}
According to \eqref{gda}, we know
\begin{multline*}
\left\|\begin{pmatrix}
           \bw_{t+1}-\bw_{t+1}' \\
           \bv_{t+1}-\bv_{t+1}'
         \end{pmatrix}\right\|_2^2\leq  \left\|\begin{pmatrix}
           \bw_t-\bw_t' \\
           \bv_t-\bv_t'
         \end{pmatrix}\right\|_2^2 +
         \eta_t^2\left\|\begin{pmatrix}
                   \nabla_{\bw}f(\bw_t,\bv_t;z_{i_t})-\nabla_{\bw}f(\bw_t',\bv'_t;z'_{i_t}) \\
                  \nabla_{\bv}f(\bw_t,\bv_t;z_{i_t})-\nabla_{\bv}f(\bw_t',\bv'_t;z'_{i_t})
                 \end{pmatrix}\right\|_2^2\\-2\eta_t\Big\langle\begin{pmatrix}
           \bw_t-\bw_t' \\
           \bv_t-\bv_t'
         \end{pmatrix},\begin{pmatrix}
           \nabla_{\bw}f(\bw_t,\bv_t;z_{i_t})-\nabla_{\bw}f(\bw_t',\bv'_t;z'_{i_t}) \\
                  \nabla_{\bv}f(\bw_t',\bv'_t;z'_{i_t})-\nabla_{\bv}f(\bw_t,\bv_t;z_{i_t})
         \end{pmatrix}\Big\rangle.
\end{multline*}
Taking a conditional expectation w.r.t. $i_t$ gives
\begin{align*}
  &\ebb_{i_t}\left\|\begin{pmatrix}
           \bw_{t+1}-\bw_{t+1}' \\
           \bv_{t+1}-\bv_{t+1}'
         \end{pmatrix}\right\|_2^2\\
         &\leq  \left\|\begin{pmatrix}
           \bw_t-\bw_t' \\
           \bv_t-\bv_t'
         \end{pmatrix}\right\|_2^2 +
         L^2\eta_t^2\left\|\begin{pmatrix}
           \bw_t-\bw_t' \\
           \bv_t-\bv_t'
         \end{pmatrix}\right\|_2^2+\frac{8G^2\eta_t^2}{n}-2\eta_t\ebb_{i_t}\Big\langle\begin{pmatrix}
           \bw_t-\bw_t' \\
           \bv_t-\bv_t'
         \end{pmatrix},\begin{pmatrix}
           \nabla_{\bw}f(\bw_t,\bv_t;z_{i_t})-\nabla_{\bw}f(\bw_t',\bv'_t;z'_{i_t}) \\
                  \nabla_{\bv}f(\bw_t',\bv'_t;z'_{i_t})-\nabla_{\bv}f(\bw_t,\bv_t;z_{i_t})
         \end{pmatrix}\Big\rangle\\
         & =  \left\|\begin{pmatrix}
           \bw_t-\bw_t' \\
           \bv_t-\bv_t'
         \end{pmatrix}\right\|_2^2 +
         L^2\eta_t^2\left\|\begin{pmatrix}
           \bw_t-\bw_t' \\
           \bv_t-\bv_t'
         \end{pmatrix}\right\|_2^2+\frac{8G^2\eta_t^2}{n}-2\eta_t\Big\langle\begin{pmatrix}
           \bw_t-\bw_t'\\
           \bv_t-\bv_t'
         \end{pmatrix},\begin{pmatrix}
           \nabla_{\bw} F_S(\bw_t,\bv_t)-\nabla_{\bw} F_{S'}(\bw_t',\bv'_t) \\
                  \nabla_{\bv}F_{S'}(\bw_t',\bv'_t)-\nabla_{\bv}F_S(\bw_t,\bv_t)
         \end{pmatrix}\Big\rangle,
\end{align*}
where we have used \eqref{wcwc-sgda-1} in the first step and used the fact \[\ebb_{i_t}\nabla f(\bw,\bv,z_{i_t})=\nabla F_S(\bw,\bv),\quad \ebb_{i_t}\nabla f(\bw,\bv,z'_{i_t})=\nabla F_{S'}(\bw,\bv)\] in the second step. According to \eqref{wcwc-diminishing-ass}, we know
\begin{align*}
  &\Big\langle\begin{pmatrix}
           \bw_t-\bw_t'\\
           \bv_t-\bv_t'
         \end{pmatrix},\begin{pmatrix}
           \nabla_{\bw} F_S(\bw_t,\bv_t)-\nabla_{\bw} F_{S'}(\bw_t',\bv'_t) \\
                  \nabla_{\bv}F_{S'}(\bw_t',\bv'_t)-\nabla_{\bv}F_S(\bw_t,\bv_t)
         \end{pmatrix}\Big\rangle\\
         &=\Big\langle\begin{pmatrix}
           \bw_t-\bw_t'\\
           \bv_t-\bv_t'
         \end{pmatrix},\begin{pmatrix}
           \nabla_{\bw} F_S(\bw_t,\bv_t)-\nabla_{\bw} F_S(\bw_t',\bv'_t) \\
                  \nabla_{\bv}F_{S}(\bw_t',\bv'_t)-\nabla_{\bv}F_S(\bw_t,\bv_t)
         \end{pmatrix}\Big\rangle+\Big\langle\begin{pmatrix}
           \bw_t-\bw_t'\\
           \bv_t-\bv_t'
         \end{pmatrix},\begin{pmatrix}
           \nabla_{\bw} F_S(\bw_t',\bv'_t)-\nabla_{\bw} F_{S'}(\bw_t',\bv'_t) \\
                  \nabla_{\bv}F_{S'}(\bw_t',\bv'_t)-\nabla_{\bv}F_S(\bw_t',\bv_t')
         \end{pmatrix}\Big\rangle\\
         & \geq -\rho_t\left\|\begin{pmatrix}
           \bw_t-\bw_t' \\
           \bv_t-\bv_t'
         \end{pmatrix}\right\|_2^2+\Big\langle\begin{pmatrix}
           \bw_t-\bw_t'\\
           \bv_t-\bv_t'
         \end{pmatrix},\begin{pmatrix}
           \nabla_{\bw} F_S(\bw_t',\bv'_t)-\nabla_{\bw} F_{S'}(\bw_t',\bv'_t) \\
                  \nabla_{\bv}F_{S'}(\bw_t',\bv'_t)-\nabla_{\bv}F_S(\bw_t',\bv_t')
         \end{pmatrix}\Big\rangle.
\end{align*}
It follows from Assumption \ref{ass:lipschitz} that
\begin{align*}
&\Big\langle\begin{pmatrix}
           \bw_t-\bw_t'\\
           \bv_t-\bv_t'
         \end{pmatrix},\begin{pmatrix}
           \nabla_{\bw} F_S(\bw_t',\bv'_t)-\nabla_{\bw} F_{S'}(\bw_t',\bv'_t) \\
                  \nabla_{\bv}F_{S'}(\bw_t',\bv'_t)-\nabla_{\bv}F_S(\bw_t',\bv_t')
         \end{pmatrix}\Big\rangle=\frac{1}{n}
         \Big\langle\begin{pmatrix}
           \bw_t-\bw_t'\\
           \bv_t-\bv_t'
         \end{pmatrix},\begin{pmatrix}
           \nabla_{\bw} f(\bw_t',\bv'_t;z_n)-\nabla_{\bw} f(\bw_t',\bv'_t;z_n') \\
                  \nabla_{\bv}f(\bw_t',\bv'_t;z_n')-\nabla_{\bv}f(\bw_t',\bv_t';z_n)
         \end{pmatrix}\Big\rangle\\
         & \geq -\frac{1}{n}\Big\|\begin{pmatrix}
           \bw_t-\bw_t'\\
           \bv_t-\bv_t'
         \end{pmatrix}\Big\|_2\Big\|\begin{pmatrix}
           \nabla_{\bw} f(\bw_t',\bv'_t;z_n)-\nabla_{\bw} f(\bw_t',\bv'_t;z_n') \\
                  \nabla_{\bv}f(\bw_t',\bv'_t;z_n')-\nabla_{\bv}f(\bw_t',\bv_t';z_n)
         \end{pmatrix}\Big\|_2\geq -\frac{2\sqrt{2}G}{n}\Big\|\begin{pmatrix}
           \bw_t-\bw_t'\\
           \bv_t-\bv_t'
         \end{pmatrix}\Big\|_2.
\end{align*}
We can combine the above three inequalities together and derive
\begin{align*}
\ebb_{i_t}\left\|\begin{pmatrix}
           \bw_{t+1}-\bw_{t+1}' \\
           \bv_{t+1}-\bv_{t+1}'
         \end{pmatrix}\right\|_2^2&\leq \big(1+2\rho_t\eta_t+L^2\eta_t^2\big)\left\|\begin{pmatrix}
           \bw_t-\bw_t' \\
           \bv_t-\bv_t'
         \end{pmatrix}\right\|_2^2 +
         \frac{8\eta_t^2G^2}{n}+\frac{4\sqrt{2}G\eta_t}{n}\left\|\begin{pmatrix}
           \bw_t-\bw_t' \\
           \bv_t-\bv_t'
         \end{pmatrix}\right\|_2\\
         & \leq \big(1+2\rho_t\eta_t+L^2\eta_t^2\big)\left\|\begin{pmatrix}
           \bw_t-\bw_t' \\
           \bv_t-\bv_t'
         \end{pmatrix}\right\|_2^2 + \frac{8\eta_t^2G^2}{n}+ \eta_t^2\left\|\begin{pmatrix}
           \bw_t-\bw_t' \\
           \bv_t-\bv_t'
         \end{pmatrix}\right\|_2^2+\frac{8G^2}{n^2}.
\end{align*}
Applying the above inequality recursively, we get
\[
\ebb_A\left\|\begin{pmatrix}
           \bw_{t+1}-\bw_{t+1}' \\
           \bv_{t+1}-\bv_{t+1}'
         \end{pmatrix}\right\|_2^2\leq \frac{8G^2}{n}\sum_{j=1}^{t}\Big(\eta_t^2+\frac{1}{n}\Big)\prod_{k=j+1}^{t}\big(1+2\rho_k\eta_k+L^2\eta_k^2+\eta_k^2\big).
\]
By the elementary inequality $1+a\leq\exp(a)$ we know
\[
\ebb_A\left\|\begin{pmatrix}
           \bw_{t+1}-\bw_{t+1}' \\
           \bv_{t+1}-\bv_{t+1}'
         \end{pmatrix}\right\|_2^2\leq \frac{8G^2}{n}\sum_{j=1}^{t}\Big(\eta_t^2+\frac{1}{n}\Big)\exp\Big(\sum_{k=j+1}^{t}\big(2\rho_k\eta_k+(L^2+1)\eta_k^2\big)\Big).
\]
It then follows from the Jensen's inequality that
\[
\ebb_A\left\|\begin{pmatrix}
           \bw_{t+1}-\bw_{t+1}' \\
           \bv_{t+1}-\bv_{t+1}'
         \end{pmatrix}\right\|_2\leq \frac{2\sqrt{2}G}{\sqrt{n}}\left(\sum_{j=1}^{t}\Big(\eta_t^2+\frac{1}{n}\Big)\exp\Big(\sum_{k=j+1}^{t}\big(2\rho_k\eta_k+(L^2+1)\eta_k^2\big)\Big)\right)^{\frac{1}{2}}.
\]
The stated bound then follows from Part (a) of Theorem \ref{thm:stab-gen} and Assumption \ref{ass:lipschitz}. The proof is complete.
\end{proof}

\section{Stability and Generalization Bounds of AGDA on Nonconvex-Nonconcave Objectives\label{sec:agda-nonconvex}}

%%%%%%%%%%%%%%%%%%%%%%%%%%%%%%%%%%%%%%%%%% AGDA %%%%%%%%%%%%%%%%%%%%%%%%%%%%%%%%%%%%%%%

In this section, we give the proof on the stability and generalization bounds of AGDA for nonconvex-nonconcave functions. The next lemma is similar to Lemma \ref{lem:stab-gen-nonconvex}, which shows AGDA typically runs several iterations before encountering the different example between $S$ and $S'$. % We will only prove the result in expectation. %Note that high-probability result can be proven following the same spirit.

\begin{lemma}\label{lem:stab-gen-nonconvex-2}
Assume $|f(\cdot,\cdot, z)| \leq 1$ for any $z$ and let Assumption \ref{ass:lipschitz} hold. Let $S=\{z_1,\ldots,z_n\}$ and $S'=\{z_1,\ldots,z_{n-1},z_n'\}$. Let $\{\bw_t,\bv_t\}$ and $\{\bw_t',\bv_t'\}$ be the sequence produced by \eqref{agda} w.r.t. $S$ and $S'$, respectively. Denote
\begin{align*}\label{eq:event-E-2}
\Delta_t = \|\bw_t-\bw'_t\|_2 + \|\bv_t-\bv'_t\|_2. \numberthis
\end{align*}
Then for any $t_0\in\nbb$ and any $\bw', \bv'$ we have
\begin{align*}
\ebb[f(\bw_T,\bv';z) - f(\bw'_T,\bv';z) + f(\bw',\bv_T;z) - f(\bw',\bv'_T;z)] \leq \frac{8t_0}{n} + G\ebb[\Delta_T | \Delta_{t_0}=0] .
\end{align*}
\end{lemma}
\begin{proof}
Let $\ecal$ denote the event that $\Delta_{t_0} = 0$. Then we have
\begin{align}
& \ebb[f(\bw_T,\bv';z) - f(\bw'_T,\bv';z) + f(\bw',\bv_T;z) - f(\bw',\bv'_T;z)]\notag\\
= & \pbb[\ecal]\ebb[f(\bw_T,\bv';z) - f(\bw'_T,\bv';z) + f(\bw',\bv_T;z) - f(\bw',\bv'_T;z) | \ecal] \notag\\
& + \pbb[\ecal^c]\ebb[f(\bw_T,\bv';z) - f(\bw'_T,\bv';z) + f(\bw',\bv_T;z) - f(\bw',\bv'_T;z) | \ecal^c]\notag\\
\leq & G\ebb[\Delta_T | \ecal]  + 4\pbb[\ecal^c],\label{event-E-00}
\end{align}
where we have used \eqref{stab-gen-non-convex-1} and the assumption $|f(\cdot,\cdot, z)| \leq 1$.
Using the union bound on the outcome $i_t = n$ and $j_t = n$ we obtain that
\begin{align*}
\pbb[\ecal^c] \leq \sum_{t=1}^{t_0}\big(\pbb[i_t = n] + \pbb[j_t = n]\big) = \frac{2t_0}{n}.
\end{align*}
The proof is complete by combining the above two inequalities together.
\end{proof}

%\begin{lemma}[Lemma \ref{lem:stability-agda} restated]
%Let $S = \{z_1, \cdots, z_n\}$ and $S' = \{z_1, \cdot, z_{n-1},z'_n\}$. Let $\{\bw_t,\bv_t\}$ and $\{\bw'_t,\bv'_t\}$ be the sequence produced by Algorithm \ref{alg:agda}. If Assumptions \ref{ass:lipschitz} and \ref{ass:smooth} hold, and $\max\{\eta_{\bw,t},\eta_{\bv,t}\} = \frac{c}{t}$, then
%\begin{align*}
%\ebb[f(\bw_T,\bv';z) - f(\bw'_T,\bv';z) + f(\bw',\bv_T;z) - f(\bw',\bv'_T;z)]  \leq & \Big(\frac{G}{L}\Big)^{\frac{1}{cL+1}}n^{-1}T^{\frac{cL}{cL+1}}.
%\end{align*}
%\end{lemma}

\begin{proof}[Proof of Theorem \ref{thm:generalization-agda}]
Since $z_{i_t}$ and $z_{j_t}$ are i.i.d, we can analyze the update of $\bw$ and $\bv$ separately. Note that the projection step is nonexpansive. We consider two cases at the $t$-th iteration. If $i_t \neq n$, then it follows from Assumption \ref{ass:smooth} that
\begin{align*}
& \|\bw_{t+1} - \bw'_{t+1}\|_2 \\
\leq &	\|\bw_t - \eta_{\bw,t}\nabla_\bw f(\bw_t, \bv_t, z_{i_t}) - \bw'_t + \eta_{\bw,t}\nabla_\bw f(\bw'_t, \bv'_t, z_{i_t})\|_2\\
\leq  & 	\|\bw_t - \eta_{\bw,t}\nabla_\bw f(\bw_t, \bv_t, z_{i_t}) - \bw'_t + \eta_{\bw,t}\nabla_\bw f(\bw'_t, \bv_t, z_{i_t})\|_2 + \|\eta_{\bw,t}\nabla_\bw f(\bw'_t, \bv_t, z_{i_t}) - \eta_{\bw,t}\nabla_\bw f(\bw'_t, \bv'_t, z_{i_t})\|_2\\
\leq & (1 + L\eta_{\bw,t}) \|\bw_t - \bw'_t\|_2 + L\eta_{\bw,t}\|\bv_t - \bv'_t\|_2.
\end{align*}
If $i_t = n$, then it follows from Assumption \ref{ass:lipschitz} that
\begin{align*}
\|\bw_{t+1} - \bw'_{t+1}\|_2 \leq &	\|\bw_t - \eta_{\bw,t}\nabla_\bw f(\bw_t, \bv_t, z_{i_t}) - \bw'_t + \eta_{\bw,t}\nabla_\bw f(\bw'_t, \bv'_t, z_{i_t})\|_2\\
\leq & \|\bw_t - \bw'_t\|_2 + 2G\eta_{\bw,t}.
\end{align*}
According to the distribution of $i_t$, we have
\begin{align}
\ebb_A[\|\bw_{t+1}-\bw_{t+1}'\|_2] \leq & \frac{n-1}{n}\ebb_A\Big[(1+\eta_{\bw,t} L)\|\bw_t -\bw_t'\|_2 +  L\eta_{\bw,t}\|\bv_t - \bv'_t\|_2\Big] + \frac{1}{n}(\|\bw_t -\bw_t'\|_2 + 2 \eta_{\bw,t} G)\nonumber\\
\leq & (1+\eta_{\bw,t} L)\ebb_A[\|\bw_t -\bw_t'\|_2] + L\eta_{\bw,t}\ebb_A\big[\|\bv_t - \bv'_t\|_2\big] + \frac{2 \eta_{\bw,t} G}{n}. \label{eq:agda-w}
\end{align}
Similarly, for $\bv$ we also have
\begin{align*}\label{eq:agda-v}
\ebb_A[\|\bv_{t+1}-\bv_{t+1}'\|_2] \leq & (1+\eta_{\bv,t} L)\ebb_A[\|\bv_t -\bv_t'\|_2] + L\eta_{\bv,t}\ebb_A\big[\|\bw_t - \bw'_t\|_2\big] + \frac{2 \eta_{\bv,t} G}{n}. \numberthis
\end{align*}
Combining \eqref{eq:agda-w} and \eqref{eq:agda-v} we have
\begin{align*}
\ebb_A[\|\bw_{t+1}-\bw_{t+1}'\|_2 + \|\bv_{t+1}-\bv_{t+1}'\|_2] \leq & (1+ (\eta_{\bw,t} + \eta_{\bv,t}) L)\ebb_A\big[\|\bw_t-\bw_t'\|_2 + \|\bv_t -\bv_t'\|_2\big] + \frac{2 (\eta_{\bw,t} + \eta_{\bv,t}) G}{n}.
\end{align*}
Recalling the event $\ecal$ that $\Delta_{t_0} = 0$, we apply the above equation recursively from $t=t_0+1$ to $T$, then
\begin{align*}
\ebb_A\big[\|\bw_{t+1}-\bw_{t+1}'\|_2 + \|\bv_{t+1}-\bv_{t+1}'\|_2\big|\Delta_{t_0} = 0\big] \leq & \frac{2G}{n}\sum_{t=t_0+1}^T(\eta_{\bw,t} + \eta_{\bv,t}) \prod_{k=t+1}^T (1+ (\eta_{\bw,k} + \eta_{\bv,k}) L).
\end{align*}
By the elementary inequality $1+x \leq \exp(x)$ and $\eta_{\bw,t} + \eta_{\bv,t} \leq \frac{c}{t}$, we have
\begin{align*}
& \ebb_A\big[\|\bw_{t+1}-\bw_{t+1}'\|_2 + \|\bv_{t+1}-\bv_{t+1}'\|_2\big|\Delta_{t_0} = 0\big]\\
\leq & \frac{2cG}{n}\sum_{t=t_0+1}^T\frac{1}{t}\prod_{k=t+1}^T\exp\Big(\frac{cL}{k}\Big) = \frac{2cG}{n}\sum_{t=t_0+1}^T\frac{1}{t}\exp\Big(\sum_{k=t+1}^T\frac{cL}{k}\Big)\\
\leq & \frac{2cG}{n}\sum_{t=t_0+1}^T\frac{1}{t}\exp\Big(cL\log\Big(\frac{T}{t}\Big)\Big) \leq \frac{2cGT^{cL}}{n}\sum_{t=t_0+1}^T\frac{1}{t^{cL+1}} \leq \frac{2G}{Ln}\Big(\frac{T}{t_0}\Big)^{cL}.
\end{align*}
By Lemma \ref{lem:stab-gen-nonconvex-2} we have
\begin{align}\label{new-eq-bound}
\ebb[f(\bw_T,\bv';z) - f(\bw'_T,\bv';z) + f(\bw',\bv_T;z) - f(\bw',\bv'_T;z)]  \leq & \frac{8t_0}{n} + \frac{2G^2}{Ln}\Big(\frac{T}{t_0}\Big)^{cL}.
\end{align}
The right hand side of the above inequality is approximately minimized when
\begin{align*}
t_0 = \Big(\frac{G^2}{4L}\Big)^{\frac{1}{cL+1}}T^{\frac{cL}{cL+1}}.
\end{align*}
Plugging it into Eq. \eqref{new-eq-bound} we have (for simplicity we assume the above $t_0$ is an integer)
\begin{align*}
\ebb[f(\bw_T,\bv';z) - f(\bw'_T,\bv';z) + f(\bw',\bv_T;z) - f(\bw',\bv'_T;z)]   \leq & 16\Big(\frac{G^2}{4L}\Big)^{\frac{1}{cL+1}}n^{-1}T^{\frac{cL}{cL+1}}.
\end{align*}
Since the above bound holds for all $z, S, S'$ and $\bw',\bv'$, we immediately get the same upper bound on the weak stability. Finally the theorem holds by calling Theorem \ref{thm:stab-gen}, Part (a).
\end{proof}

We require an assumption on the existence of saddle point to address the optimization error of AGDA~\citep{yang2020global}.
\begin{assumption}[Existence of Saddle Point]\label{ass:exist}
Assume for any $S$, $F_S$ has at least one saddle point. Assume for any $\bv$, $\min_{\bw}F_S(\bw,\bv)$ has a nonempty solution set and a finite optimal value.
Assume for any $\bw$, $\max_{\bv}F_S(\bw,\bv)$ has a nonempty solution set and a finite optimal value.
\end{assumption}
The following lemma establishes the generalization bound for the empirical maximizer of a strongly concave objective. It is a direct extension of the stability analysis in \citet{shalev2010learnability} for strongly convex objectives. %The proof is based on the argument in \citet{shalev2010learnability} for strongly convex objectives.
\begin{lemma}\label{lem:sc-stab-gen}
Assume that for any $\bw$ and $S$, the function $\bv \mapsto F_S(\bw,\bv)$ is $\rho$-strongly-concave. Suppose for any $\bw$, $\bv,\bv'$ and for any $z$ we have
\begin{equation}\label{sc-lip}
    \big|f(\bw,\bv;z) - f(\bw,\bv';z)\big| \leq G \|\bv - \bv'\|_2.
\end{equation}
Fix any $\bw$. Denote $\hat{\bv}^*_S = \arg\max_{\bv\in \vcal}F_S(\bw,\bv)$. Then
\begin{align*}
\ebb[F_S(\bw,\hat{\bv}^*_S) - F(\bw,\hat{\bv}^*_S)] \leq \frac{4G^2}{\rho n}.
\end{align*}
\end{lemma}
\begin{proof}
Let $S'=\{z'_1,\ldots,z'_n\}$ be drawn independently from $\rho$. For any $i\in[n]$, define $S^{(i)}=\{z_1,\ldots,z_{i-1},z_i',z_{i+1},\ldots,z_n\}$. Denote $\hat{\bv}^*_{S^{(i)}} = \arg\max_{\bv\in \vcal}F_{S^{(i)}}(\bw,\bv)$. Then
\begin{align*}\label{eq:sc-stab}
F_S(\bw,\hat{\bv}^*_S) - F_S(\bw,\hat{\bv}^*_{S^{(i)}}) = & \frac{1}{n}\sum_{j\neq i}\Big(f(\bw,\hat{\bv}^*_S;z_j) - f(\bw,\hat{\bv}^*_{S^{(i)}};z_j) \Big)   + \frac{1}{n}\Big(f(\bw,\hat{\bv}^*_S;z_i) - f(\bw,\hat{\bv}^*_{S^{(i)}};z_i)\Big)\\
= &  \frac{1}{n}\Big( f(\bw,\hat{\bv}^*_{S^{(i)}};z'_i)-f(\bw,\hat{\bv}^*_S;z'_i) \Big) + \frac{1}{n}\Big(f(\bw,\hat{\bv}^*_S;z_i) - f(\bw,\hat{\bv}^*_{S^{(i)}};z_i)\Big)\\
& + F_{S^{(i)}}(\bw,\hat{\bv}^*_S) - F_{S^{(i)}}(\bw,\hat{\bv}^*_{S^{(i)}})\\
\leq & \frac{1}{n}\Big( f(\bw,\hat{\bv}^*_{S^{(i)}};z'_i)-f(\bw,\hat{\bv}^*_S;z'_i)\Big) + \frac{1}{n}\Big(f(\bw,\hat{\bv}^*_S;z_i) - f(\bw,\hat{\bv}^*_{S^{(i)}};z_i)\Big)\\
\leq & \frac{2G}{n}\big\|\hat{\bv}^*_S - \hat{\bv}^*_{S^{(i)}}\big\|_2, \numberthis
\end{align*}
where the first inequality follows from the fact that $\hat{\bv}^*_{S^{(i)}}$ is the maximizer of $F_{S^{(i)}}(\bw,\cdot)$ and the second inequality follows from \eqref{sc-lip}. Since $F_S$ is strongly-concave and $\hat{\bv}^*_S$ maximizes $F_S(\bw,\cdot)$, we know
\begin{align*}
\frac{\rho}{2}\big\|\hat{\bv}^*_S - \hat{\bv}^*_{S^{(i)}}\big\|_2^2 \leq F_S(\bw,\hat{\bv}^*_S) - F_S(\bw,\hat{\bv}^*_{S^{(i)}}).
\end{align*}
Combining it with \eqref{eq:sc-stab} we get $\big\|\hat{\bv}^*_S - \hat{\bv}^*_{S^{(i)}}\big\|_2 \leq 4G/(\rho n)$. By \eqref{sc-lip}, the following inequality holds for any $z$
\begin{align*}
\big|f(\bw,\hat{\bv}^*_S ;z) - f(\bw,\hat{\bv}^*_{S^{(i)}};z)\big|  \leq \frac{4G^2}{\rho n}.
\end{align*}
Since $z_i$ and $z'_i$ are i.i.d., we have
\begin{align*}
\ebb\big[F(\bw,\hat{\bv}^*_S)\big]   =   \ebb\big[F(\bw,\hat{\bv}^*_{S^{(i)}})\big] =  \frac{1}{n}\sum_{i=1}^n\ebb\big[f(\bw,\hat{\bv}^*_{S^{(i)}};z_i)\big],
\end{align*}
where the last identity holds since $z_i$ is independent of $\hat{\bv}^*_{S^{(i)}}$.
Therefore
\begin{align*}
\ebb\big[F_S(\bw,\hat{\bv}^*_S) - F(\bw,\hat{\bv}^*_S)\big] = \frac{1}{n}\sum_{i=1}^n \ebb\big[f(\bw,\hat{\bv}^*_S;z_i) - f(\bw,\hat{\bv}^*_{S^{(i)}};z_i)\big] \leq \frac{4G^2}{\rho n}.
\end{align*}
The proof is complete.
\end{proof}
\begin{corollary}\label{cor:pl-adga}
  Let $\beta_1,\rho>0$.
  Let Assumptions \ref{ass:lipschitz}, \ref{ass:smooth}, \ref{ass:pl-two} with $\beta_1(S)\geq \beta_1,\beta_2(S)\geq\rho$ and \ref{ass:exist} hold.  Assume for any $\bw$ and any $S$, the functions $\bv\mapsto F(\bw,\bv)$ and $\bv\mapsto F_S(\bw,\bv)$ are $\rho$-strongly concave.
  Let $\{\bw_t,\bv_t\}$ be the sequence produced by \eqref{agda} with $\eta_{\bw,t} \asymp1/({\beta_1t})$ and $\eta_{\bv,t} \asymp 1/{(\beta_1\rho^2t)}$. Then for $T\asymp\big(\frac{n}{\beta_1^2\rho^3}\big)^{\frac{cL+1}{2cL+1}}$, we have
    \[
  \ebb\big[R(\bw_T)-R(\bw^*)\big]=\ocal\Big(n^{-\frac{cL+1}{2cL+1}}\beta_1^{-\frac{2cL}{2cL+1}}\rho^{-\frac{5cL+1}{2cL+1}}\Big),
  \]
  where $c\asymp 1/(\beta_1\rho^2)$.
\end{corollary}
\begin{proof}
We have the error decomposition
\begin{equation}\label{pl-adga-0}
R(\bw_T) - R(\bw^*) = \big(R(\bw_T) - R_S(\bw_T)\big) + \big(R_S(\bw_T) - R_S(\bw^*)\big) + \big(R_S(\bw^*) - R(\bw^*)\big).
\end{equation}
First we consider the term $R(\bw_T) - R_S(\bw_T)$.  Analogous to the proof of Theorem \ref{thm:generalization-agda} (i.e., the only difference is to replace the conditional expectation of function values in \eqref{event-E-00} with the conditional expectation of $\ebb[\|\bw_T-\bw_T'\|_2+\|\bv_T-\bv_T'\|_2]$), one can show that AGDA is $\ocal\big(n^{-1}T^{\frac{cL}{cL+1}}\big)$-argument stable (note the step sizes satisfy $\eta_{\bw,t}+\eta_{\bv,t}\leq c/t$). This together with Part (b) of Theorem \ref{thm:stab-gen} implies that
  \begin{equation}\label{pl-adga-1}
  \ebb\big[R(\bw_T)-R_S(\bw_T)\big]=\ocal\big((\rho n)^{-1}T^{\frac{cL}{cL+1}}\big).
  \end{equation}
For the term $R_S(\bw_T) - R_S(\bw^*)$, the optimization error bounds in \citet{yang2020global} show that
  \begin{equation}\label{pl-adga-2}
  \ebb\big[R_S(\bw_T)-R_S(\bw^*)\big]=O\Big(\frac{1}{\beta_1^2\rho^4T}\Big).
  \end{equation}
Finally, for the term $R_S(\bw^*) - R(\bw^*)$,  we further decompose it as
\[
\ebb\big[R_S(\bw^*) - R(\bw^*)\big] = \ebb\big[F_S(\bw^*,\hat{\bv}^*_S) - F(\bw^*,\bv^*)\big] = \ebb\big[F_S(\bw^*,\hat{\bv}^*_S) - F(\bw^*,\hat{\bv}^*_S)\big] + \ebb\big[F(\bw^*,\hat{\bv}^*_S) - F(\bw^*,\bv^*)\big],
\]
where $\hat{\bv}^*_S=\arg\max_{\bv}F_S(\bw^*,\bv)$.
The second term $\ebb\big[F(\bw^*,\hat{\bv}^*_S) - F(\bw^*,\bv^*)\big] \leq 0$ since $(\bw^*,\bv^*)$ is a saddle point of $F$. Therefore by Lemma \ref{lem:sc-stab-gen} we have
\[
\ebb\big[R_S(\bw^*) - R(\bw^*)\big] \leq \ebb\big[F_S(\bw^*,\hat{\bv}^*_S) - F(\bw^*,\hat{\bv}^*_S)\big] = \ocal\Big(\frac{1}{\rho n}\Big).
\]
  We can plug the above inequality, \eqref{pl-adga-1}, \eqref{pl-adga-2} into \eqref{pl-adga-0}, and get
  \[
  \ebb\big[R(\bw_T)-R(\bw^*)\big]=\ocal\big((\rho n)^{-1}T^{\frac{cL}{cL+1}}\big)+O\Big(\frac{1}{\beta_1^2\rho^4T}\Big) + \ocal\Big(\frac{1}{\rho n}\Big).
  \]
  We can choose $T\asymp\big(\frac{n}{\beta_1^2\rho^3}\big)^{\frac{cL+1}{2cL+1}}$ to get the stated excess primal population risk bounds.
%  \[
%  \ebb\big[R(\bw_T)-R(\bw^*)\big]=O\Big(n^{-\frac{cL}{2cL+1}}\beta_1^{-\frac{2cL}{2cL+1}}\rho^{-\frac{5cL+1}{2cL+1}}\Big).
%  \]
  The proof is complete.
\end{proof}

\section{Proof of Theorem \ref{thm:gen-pl}\label{sec:proof-pl}}
To prove Theorem \ref{thm:gen-pl}, we first introduce a lemma on relating the difference of function values to gradients.
\begin{lemma}\label{lem:grad-domin}
  Let Assumption \ref{ass:pl-two} hold. For any $\bu=(\bw,\bv)$ and any stationary point $\bu_{(S)}=(\bw_{(S)},\bv_{(S)})$ of $F_S$, we have
  \[
  -\frac{\|\nabla_{\bv} F_S(\bw,\bv)\|_2^2}{2\beta_2(S)}\leq F_S(\bu)-F_S(\bu_{(S)})\leq \frac{\|\nabla_{\bw}F_S(\bw,\bv)\|_2^2}{2\beta_1(S)}.
  \]
\end{lemma}
\begin{proof}
Since $\bu_{(S)}$ is a stationary point, it is also a saddle point under the PL condition~\citep{yang2020global} which means that
\[
F_S(\bw_{(S)},\bv')\leq F_S(\bw_{(S)},\bv_{(S)})\leq F_S(\bw',\bv_{(S)}),\quad\forall \bw'\in\wcal,\bv'\in\vcal.
\]
It then follows that
\begin{align*}
  F_S(\bu)-F_S(\bu_{(S)}) & = F_S(\bw,\bv)-F_S(\bw_{(S)},\bv)+F_S(\bw_{(S)},\bv)-F_S(\bw_{(S)},\bv_{(S)})\\
  & \leq F_S(\bw,\bv)-F_S(\bw_{(S)},\bv)\leq F_S(\bw,\bv)-\inf_{\bw'\in\wcal}F_S(\bw',\bv)\leq \frac{1}{2\beta_1(S)}\|\nabla_{\bw} F_S(\bw,\bv)\|_2^2,
\end{align*}
where in the last inequality we have used Assumption \ref{ass:pl-two}.
In a similar way, we know
\begin{align*}
  F_S(\bu)-F_S(\bu_{(S)}) & = F_S(\bw,\bv)-F_S(\bw,\bv_{(S)})+F_S(\bw,\bv_{(S)})-F_S(\bw_{(S)},\bv_{(S)})\\
  & \geq F_S(\bw,\bv)-F_S(\bw,\bv_{(S)})\geq F_S(\bw,\bv)-\sup_{\bv'}F_S(\bw,\bv')\geq -\frac{1}{2\beta_2(S)}\|\nabla_{\bv} F_S(\bw,\bv)\|_2^2.
\end{align*}
The proof is complete.
\end{proof}
\begin{proof}[Proof of Theorem \ref{thm:gen-pl}]
Let $S'=\{z'_1,\ldots,z'_n\}$ be drawn independently from $\rho$. For any $i\in[n]$, define $S^{(i)}=\{z_1,\ldots,z_{i-1},z_i',z_{i+1},\ldots,z_n\}$.
Let $\bu_S=(A_{\bw}(S),A_{\bv}(S))$ and $\bu_S^{(S)}$ be the projection of $\bu_S$ onto the set of stationary points of $F_S$.
For each $i\in[n]$, we denote $\bu_i=(A_{\bw}(S^{(i)}),A_{\bv}(S^{(i)}))$ and $\bu_i^{(i)}$ the projection of $\bu_i$ onto the set of stationary points of $F_{S^{(i)}}$. Then $\nabla F_{S^{(i)}}(\bu_i^{(i)})=0$.

%and $\bu^{(S^{(i)})}_{S^{(i)}}$ the projection of $\bw_{S^{(i)}}$ onto the set of global minimizer of $F_{S^{(i)}}$.
We decompose $f(\bu_i;z_i)-f(\bu_S;z_i)$ as follows
\begin{multline}\label{pl-1}
  f(\bu_i;z_i)-f(\bu_S;z_i) = \big(f(\bu_i;z_i)-f(\bu_i^{(i)};z_i)\big)
  + \big(f(\bu_i^{(i)};z_i)-f(\bu_S^{(S)};z_i)\big)+\big(f(\bu_S^{(S)};z_i)-f(\bu_S;z_i)\big).
\end{multline}
We now address the above three terms separately.

We {first} address $f(\bu_i^{(i)};z_i)-f(\bu_S^{(S)};z_i)$.
According to the definition of $F_S,S,S^{(i)}$, we know
\[
f(\bu_i^{(i)};z_i)=nF_S(\bu_i^{(i)})-nF_{S^{(i)}}(\bu_i^{(i)})+f(\bu_i^{(i)};z'_i).
\]
Since $z_i$ and $z'_i$ follow from the same distribution, we know
$\ebb[f(\bu_i^{(i)};z'_i)]=\ebb[f(\bu_S^{(S)};z_i)]$ and further get
\[
\ebb\big[f(\bu_i^{(i)};z_i)\big]=n\ebb\big[F_S(\bu_i^{(i)})\big]-n\ebb\big[F_{S^{(i)}}(\bu_i^{(i)})\big]+\ebb\big[f(\bu_S^{(S)};z_i)\big].
\]
It then follows that
\begin{equation}
  \ebb\big[f(\bu_i^{(i)};z_i)-f(\bu_S^{(S)};z_i)\big] = n\ebb\big[F_S(\bu_i^{(i)})-F_{S^{(i)}}(\bu_i^{(i)})\big]
   %& = n \ebb\Big[F_S(\bw^{(S^{(i)})}_{S^{(i)}})-\inf_{\bw\in\wcal}F_S(\bw)+\inf_{\bw\in\wcal}F_S(\bw)-F_{S^{(i)}}(\bw^{(S^{(i)})}_{S^{(i)}})\Big] \notag\\
   = n \ebb\Big[F_S(\bu_i^{(i)})-F_{S}(\bu_S^{(S)})\Big],\label{pl-00}
\end{equation}
where we have used the following identity due to the symmetry between $z_i$ and $z'_i$:
$
\ebb[F_{S^{(i)}}(\bu_i^{(i)})]=\ebb\big[F_{S}(\bu_S^{(S)})\big].
$
By the PL condition of $F_S$, it then follows from \eqref{pl-00} and Lemma \ref{lem:grad-domin} that
\begin{equation}\label{pl-0}
\ebb\big[f(\bu^{(i)};z_i)-f(\bu_S^{(S)};z_i)\big]\leq \frac{n}{2}\ebb\big[\frac{1}{\beta_1(S)}\|\nabla_{\bw} F_S(\bu_i^{(i)})\|_2^2\big].
\end{equation}
According to the definition of $\bu_i^{(i)}$ we know $\nabla_{\bw} F_{S^{(i)}}(\bu_i^{(i)})=0$ and therefore ($(a+b)^2\leq2a^2+2b^2$)
\begin{align}
  \|\nabla_{\bw} F_S(\bu_i^{(i)})\|_2^2 & = \Big\|\nabla_{\bw} F_{S^{(i)}}(\bu_i^{(i)})-\frac{1}{n}\nabla_{\bw} f(\bu_i^{(i)};z'_i)+\frac{1}{n}\nabla_{\bw} f(\bu_i^{(i)};z_i)\Big\|_2^2 \notag\\
   & \leq \frac{2}{n^2}\|\nabla_{\bw} f(\bu_i^{(i)};z'_i)\|_2^2 + \frac{2}{n^2}\|\nabla_{\bw} f(\bu_i^{(i)};z_i)\|_2^2
   \leq \frac{4G^2}{n^2},\label{pl-01}
\end{align}
where we have used Assumption \ref{ass:lipschitz}.
This together with \eqref{pl-0} gives
\begin{equation}\label{pl-2}
\ebb\big[f(\bu^{(i)};z_i)-f(\bu_S^{(S)};z_i)\big]\leq \frac{2G^2}{n}\ebb\Big[\frac{1}{\beta_1(S)}\Big].
\end{equation}

We {then} address $f(\bu_i;z_i)-f(\bu_i^{(i)};z_i)$.
Since $\bu_i$ and $\bu_i^{(i)}$ are independent of $z_i$, we know
\begin{equation}\label{pl-3}
  \ebb\big[f(\bu_i;z_i)-f(\bu_i^{(i)};z_i)\big]=\ebb\big[F(\bu_i)-F(\bu_i^{(i)})\big] = \ebb\big[F(\bu_S)-F(\bu_S^{(S)})\big],
\end{equation}
where we have used the symmetry between $z_i$ and $z'_i$.

Finally, we address $f(\bu_S^{(S)};z_i)-f(\bu_S;z_i)$. By the definition of $\bu_S^{(S)}$ we know
\begin{equation}\label{pl-9}
\sum_{i=1}^n\big(f(\bu_S^{(S)};z_i)-f(\bu_S;z_i)\big)=n\big(F_S(\bu_S^{(S)})-F_S(\bu_S)\big).
\end{equation}
Plugging \eqref{pl-2}, \eqref{pl-3} and the above inequality back into \eqref{pl-1}, we derive
\[
\sum_{i=1}^{n}\ebb\big[f(\bu_i;z_i)-f(\bu_S;z_i)\big]\leq
\ebb\Big[\frac{2G^2}{\beta_1(S)}\Big]+
n\ebb\big[F(\bu_S)-F(\bu_S^{(S)})\big]+n\ebb\big[F_S(\bu_S^{(S)})-F_S(\bu_S)\big].
\]
Since $z_i$ and $z'_i$ are drawn from the same distribution, we know
  \begin{align}
  \ebb\big[F(\bu_S)-F_S(\bu_S)\big] & = \frac{1}{n}\sum_{i=1}^{n}\ebb\big[F(\bu_i)-F_S(\bu_S)\big]
  = \frac{1}{n}\sum_{i=1}^{n}\ebb\big[f(\bu_i;z_i)-f(\bu_S;z_i)\big]\notag\\
  & \leq \frac{2G^2}{n}\ebb\Big[\frac{1}{\beta_1(S)}\Big]+
\ebb\big[F(\bu_S)-F(\bu_S^{(S)})\big]+\ebb\big[F_S(\bu_S^{(S)})-F_S(\bu_S)\big],\label{pl-4}
  \end{align}
  where the second identity holds since $z_i$ is independent of $\bu_i$. It then follows that
\begin{equation}\label{pl-5}
\ebb\big[F(\bu_S^{(S)})-F_S(\bu_S^{(S)})\big]\leq \frac{2G^2}{n}\ebb\Big[\frac{1}{\beta_1(S)}\Big].
\end{equation}
According to the Lipschitz continuity we know
\[
  \big|F(\bu_S) - F(\bu_S^{(S)})\big|\leq G\|\bu_S-\bu_S^{(S)}\|_2\quad\text{and}\quad
  \big|F_S(\bu_S) - F_S(\bu_S^{(S)})\big|\leq G\|\bu_S-\bu_S^{(S)}\|_2.
\]
Plugging the above inequality back into \eqref{pl-4}, we derive the following inequality
\begin{equation}\label{pl-6}
\ebb\big[F(\bu_S)-F_S(\bu_S)\big]\leq \frac{2G^2}{n}\ebb\Big[\frac{1}{\beta_1(S)}\Big]+2G\ebb\big[\|\bu_S-\bu_S^{(S)}\|_2\big].
\end{equation}
By Lemma \ref{lem:grad-domin} and \eqref{pl-00}, we can also have
\[
\ebb\big[f(\bu_i^{(i)};z_i)-f(\bu_S^{(S)};z_i)\big]\geq -\frac{n}{2}\ebb\big[\frac{1}{\beta_2(S)}\|\nabla_{\bv} F_S(\bu_i^{(i)})\|_2^2\big].
\]
Using this inequality, one can analyze analogously to \eqref{pl-6} and derive the following inequality
\[
\ebb\big[F(\bu_S)-F_S(\bu_S)\big]\geq -\frac{2G^2}{n}\ebb\Big[\frac{1}{\beta_2(S)}\Big]-2G\ebb\big[\|\bu_S-\bu_S^{(S)}\|_2\big].
\]
The stated inequality follows from the above inequality and \eqref{pl-6}. The proof is complete.
\end{proof}

\section{Additional Experiments\label{sec:exp-more}}

In this section, we investigate the stability of SGDA on a nonconvex-nonconcave problem. We consider the vanilla GAN structure proposed in \citet{goodfellow2014generative}. The generator and the discriminator consist of 4 fully connected layers, and use the leaky rectified linear activation before the output layer. The generator uses the hyperbolic tangent activation at the output layer. The discriminator uses the sigmoid activation at the output layer.
In order to make experiments more interpretable in terms of stability, we remove all forms of regularization such as the weight decay or dropout in the original paper. In order to truly implement SGDA, we generate only one noise for updating both the discriminator and the generator at each iteration. This differs from the common GAN training strategy, which uses different noises for updating the discriminator and the generator. We employ the \texttt{mnist} dataset \citep{lecun1998gradient} and build neighboring datasets $S$ and $S'$ by removing a randomly chosen datum indexed by $i$ from $S$ and $i+1$ from $S'$. The algorithm is run based on the same trajectory for $S$ and $S'$ by fixing the random seed. We randomly pick 5 different $i$'s and 5 different random seeds (total 25 runs). The step sizes for the discriminator and the generator are chosen as constants, i.e. $\eta = 0.0002$. We compute the Euclidean distance, i.e., Frobenius norm, between the parameters trained on the neighboring datasets. Note that we do not target at optimizing the test accuracy, but give an interpretable visualization to validate our theoretical findings. The results are given in Figure \ref{fig:gan}.

\begin{figure}[ht!]
\centering
\includegraphics[width=0.8\textwidth]{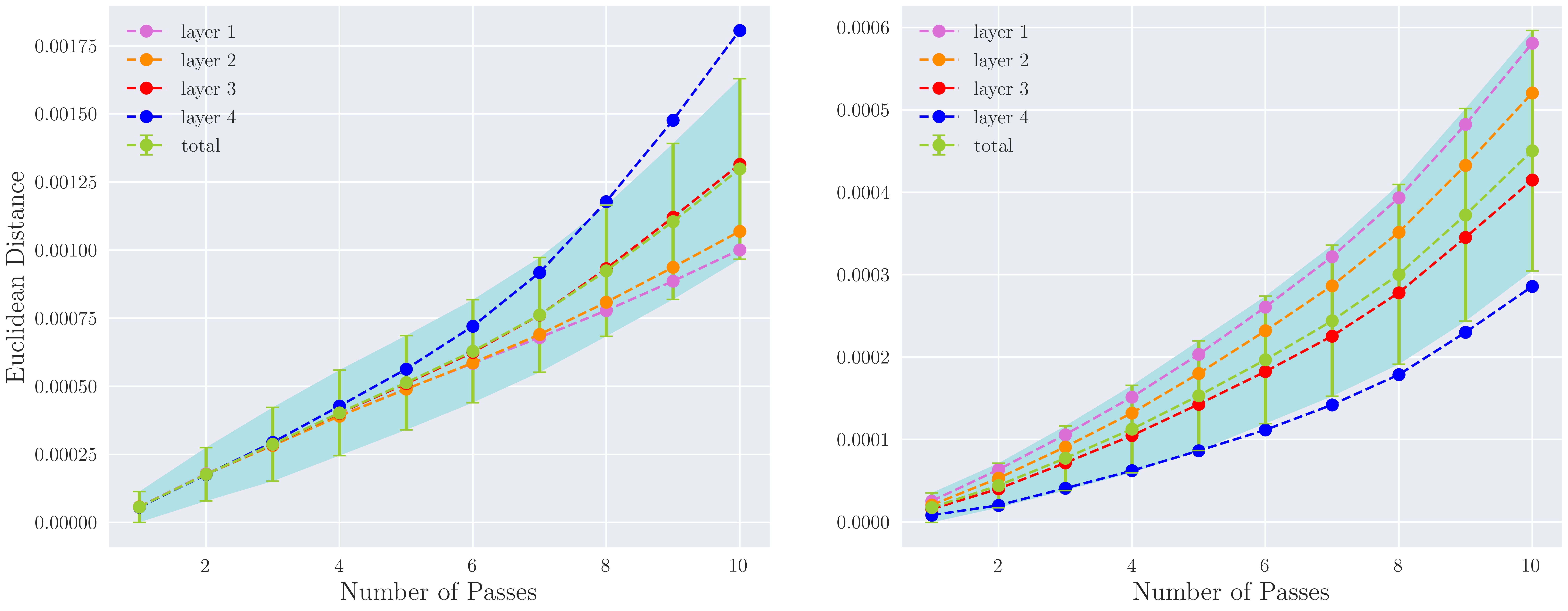}
\caption{\label{fig:gan}The parameter distance versus the number of passes. Left: generator, right: discriminator. 'total' is the mean normalized Euclidean distance across all layers and the shaded area is the standard deviation.} % In the left, we show the result for svmguide3 and in the right, we show the result for w5a.}
\end{figure}

It is clear that the parameter distances for both the generator and the discriminator continue to increase during the training process of SGDA, which is consistent with our analysis in Section \ref{sec:sgda-wcwc} and \ref{sec:wcwc-diminishing}.

%Next we turn to the proof of the optimization error of AGDA. First of all, we prove the connection between the weak PD empirical risk and the potential function $P_t = R_S(\bw_t) - R_S(\bw_*(\bv_*)) + \lambda(R_S(\bw_t) - F_S(\bw_t, \bv_t))$. In fact, we can prove a stronger bound between the strong PD empirical risk and the potential function.

%\begin{lemma}
%The following statement is true
%\begin{align*}
%\triangle_S^s(\bw_t,\bv_t) \leq \frac{1}{\lambda} P_t
%\end{align*}
%\end{lemma}

%\input{appendix}
%%%%%%%%%%%%%%%%%%%%%%%%%%%%%%%%%%%%%%%%%%%%%%%%%%%%%%%%%%%%%%%%%%%%%%%%%%%%%%%
%%%%%%%%%%%%%%%%%%%%%%%%%%%%%%%%%%%%%%%%%%%%%%%%%%%%%%%%%%%%%%%%%%%%%%%%%%%%%%%

\end{document}